\documentclass[11pt]{article}

\RequirePackage[OT1]{fontenc}
\usepackage{fullpage}

\usepackage[caption=false]{subfig}
\usepackage[usenames]{color}
\usepackage[figuresright]{rotating}
\usepackage{boxedminipage}
\usepackage[colorlinks,citecolor=blue,urlcolor=blue]{hyperref}
\usepackage{epstopdf}
\usepackage{natbib}
\usepackage{multirow}
\usepackage{enumerate}
\usepackage{enumitem}
\usepackage{verbatim}
\usepackage{bbm}
\usepackage{xcolor}
\usepackage{bm}
\usepackage{mathrsfs,color,dsfont}
\usepackage{algorithm,setspace}
\usepackage{algorithmic}

\usepackage{tikz}
\usetikzlibrary{arrows}

\makeatletter
\pgfdeclareshape{datastore}{
	\inheritsavedanchors[from=rectangle]
	\inheritanchorborder[from=rectangle]
	\inheritanchor[from=rectangle]{center}
	\inheritanchor[from=rectangle]{base}
	\inheritanchor[from=rectangle]{north}
	\inheritanchor[from=rectangle]{north east}
	\inheritanchor[from=rectangle]{east}
	\inheritanchor[from=rectangle]{south east}
	\inheritanchor[from=rectangle]{south}
	\inheritanchor[from=rectangle]{south west}
	\inheritanchor[from=rectangle]{west}
	\inheritanchor[from=rectangle]{north west}
	\backgroundpath{
	\southwest \pgf@xa=\pgf@x \pgf@ya=\pgf@y
	\northeast \pgf@xb=\pgf@x \pgf@yb=\pgf@y
	\pgfpathmoveto{\pgfpoint{\pgf@xa}{\pgf@ya}}
	\pgfpathlineto{\pgfpoint{\pgf@xb}{\pgf@ya}}
	\pgfpathmoveto{\pgfpoint{\pgf@xa}{\pgf@yb}}
	\pgfpathlineto{\pgfpoint{\pgf@xb}{\pgf@yb}}
	}
}
\makeatother

\usepackage{amsthm,amsmath,amssymb,amsfonts, amsbsy, mathtools, epsfig}

\newcommand{\Norm}[1]{\left\Vert #1\right\Vert}			
\newcommand{\vect}[1]{\boldsymbol{#1}}				

\newcommand{\tp}{{\mathrm T}}						
\newcommand{\cov}{{\mathrm Cov}}						
\newcommand{\var}{\mathrm{Var}}						
\newcommand{\med}{\mathrm{med}}						
\newcommand{\card}{{\mathrm Card}}					
\newcommand{\diff}{{\mathrm d}}						
\newcommand{\argmin}[1]{\mathop{\mathrm{argmin}}_{#1}}	
											
\newcommand*\widebar[1]{\,							
	\hbox{%
   		\kern0.01em%
		\vbox{%
			\hrule height 0.5pt		
			\kern0.33ex			
			\hbox{%
				\kern-0.1em		
				\ensuremath{#1}%
				\kern-0.1em		
			}%
		}%
	}%
\,}%
\let\hat\widehat
\let\tilde\widetilde
\let\bar\widebar

\newcommand{\eg}{\mbox{\sl e.g.\;}}					

\newcommand{\bX}{{\mathbf X}}

\newcommand{\bSigma}{\boldsymbol{\Sigma}}

\newcommand{\btheta} {\boldsymbol{\theta}}

\newcommand{\bmu} {\boldsymbol{\mu}}

\newcommand{\mcB}{{\mathcal B}}					
\newcommand{\mcC}{{\mathcal C}}

\newcommand{\mcH}{{\mathcal H}}					

\newcommand{\mcL}{{\mathcal L}}					

\newcommand{\mcN}{{\mathcal N}}					

\newcommand{\mcR}{{\mathcal R}}

\newcommand{\mbB}{{\mathbb B}}					

\newcommand{\mbE}{{\mathbb E}}					

\newcommand{\mbI}{{\mathbb I}}					

\newcommand{\mbP}{{\mathbb P}}					
\newcommand{\mbR}{{\mathbb R}}					
\newcommand{\mbS}{{\mathbb S}}					

\newcommand{\mfE}{{\mathfrak E}}					

\newcommand{\mfN}{{\mathfrak N}}					

\newcommand{\mfX}{{\mathfrak X}}					

\DeclareMathAlphabet\EuScriptBF{U}{eus}{b}{n}

\newcommand{\mfp}{\mathfrak{p}}		

\definecolor{DSgray}{cmyk}{0,1,0,0}

\newtheorem{definition}{Definition}
\newtheorem{lemma}{Lemma}
\newtheorem{theorem}{Theorem}

\newtheorem{proposition}{Proposition}
\newtheorem{remark}{Remark}
\newtheorem{example}{Example}
\newtheorem{assumption}{Assumption}

\author{
	Jiyuan Tu
	\footnote{Shanghai Jiao Tong University, Shanghai, China, Email: tujy.19@gmail.com} ~
	Weidong Liu
	\footnote{Shanghai Jiao Tong University, Shanghai, China, Email: weidongl@sjtu.edu.cn} ~
	Xiaojun Mao
	\footnote{Fundan University,
		Shanghai,
		China,
		Email: maoxj@fudan.edu.cn} ~
	Xi Chen
	\footnote{New York University, New York, USA, Email: xchen3@stern.nyu.edu} 
	}

\begin{document}
	\title{Variance Reduced Median-of-Means Estimator for Byzantine-Robust Distributed Inference}
	\author{Jiyuan Tu
	\footnote{Shanghai Jiao Tong University, Shanghai, China, Email: tujy.19@gmail.com} ~
	Weidong Liu
	\footnote{Shanghai Jiao Tong University, Shanghai, China, Email: weidongl@sjtu.edu.cn} ~
	Xiaojun Mao
	\footnote{Fundan University,
		Shanghai,
		China,
		Email: maoxj@fudan.edu.cn} ~
	Xi Chen
	\footnote{New York University, New York, USA, Email: xchen3@stern.nyu.edu} 
	}
	
	\date{}
	\maketitle
\begin{abstract}
	
	This paper develops an efficient distributed inference algorithm,  which is robust against a moderate fraction of Byzantine nodes, namely arbitrary and possibly adversarial machines in a distributed learning system. In robust statistics, the median-of-means (MOM) has been a popular approach to hedge against Byzantine failures due to its ease of implementation and computational efficiency. However, the MOM estimator has the shortcoming in terms of statistical efficiency. The first main contribution of the paper is to propose a variance reduced median-of-means (VRMOM) estimator, which improves the statistical efficiency over the vanilla MOM estimator and is computationally as efficient as the MOM. Based on the proposed VRMOM estimator, we develop a general distributed inference algorithm that is robust against Byzantine failures.  Theoretically, our distributed algorithm achieves a fast convergence rate with only a constant number of rounds of communications. We also provide the asymptotic normality result for the purpose of statistical inference. To the best of our knowledge, this is the first normality result in the setting of Byzantine-robust distributed learning.	The simulation results are also presented to illustrate the effectiveness of our method.

\end{abstract}
\noindent
\textbf{keywords:}
Byzantine robustness, distributed inference, median-of-means, statistical efficiency

\section{Introduction}

Due to the rapid increase of the scale of data,  modern datasets are usually too large to fit in a single device, and thus have to be stored and processed in a distributed manner.  In a common distributed computing environment, data are stored across multiple machines/nodes. A single master node is in charge of maintaining and updating target parameters, and a large number of worker machines perform local computations and communicate the computed information with the master node (see Figure \ref{fig:distribute} in Section \ref{sec:robust-csl} for an illustration). As compared to the traditional single machine setting, where the entire data can be loaded into the memory for the centralized computation, the distributed setting poses two major challenges.  

The first challenge comes from the tradeoff between communication cost and statistical accuracy. For example, one-shot communication (e.g., taking average of local estimators), though incurs low communication cost, has a poor performance for nonlinear estimation when the number of machines is large (see, e.g., \cite{li2013statistical, zhang_etal.2013, zhang2015divide, zhao2016partially,  rosenblatt_etal.2016, Shang17Computation, lee2017communication}). Therefore, iterative approaches are adopted in literature (see, e.g., \cite{shamir_etal.2013,jordan_etal.2019,Chen:19qr,fan_guo_etal.2019,wang2019distributed, chen_liu_etal.2019}). For iterative algorithms, since each iteration of communication requires synchronization, a communicationally efficient algorithm should run with a small number of iterations. Our goal is to develop algorithms that achieve communication efficiency without losing statistical accuracy.

The second challenge comes from the vulnerability of worker machines and communication channels.  In particular, the information sent from a worker machine can be arbitrarily erroneous due to hardware or software breakdowns, data crashes, or communication failures. Such an error is usually referred to as Byzantine failures \citep{lamport_etal.1982}. In other words, a subset of workers called Byzantine machines, may send arbitrary and even adversarial messages to the master. Distributed learning under the Byzantine setting has attracted a lot of research attentions in recent years (see, e.g., \cite{feng_etal.2014, chen_su_etal.2017,blanchard_etal.2017,xie_koyejo_gupta.2018,alistarh_zhu_li.2018,yin_etal.2018,yin_chen_etal.2018,su_xu.2018}). However, as we will survey later, some of these methods suffer from a larger number of iterations of communications and existing analysis only focuses on the convergence rate.  The statistical inference with Byzantine failures, which plays an important role in uncertainty quantification, is still largely open. 

The goal of this paper is to propose a communication-efficient statistical inference method, which is robustly against Byzantine failures. We consider a general risk minimization problem, 
\begin{equation}	\label{eq:true_para}
\vect{\theta}^*= \argmin{\vect{\theta}\in\mbR^p}\mbE_{X \sim \mfX}\left\{ f(X,\vect{\theta})\right\},
\end{equation}
where $f$ is the convex loss function,  $X$ denotes the random sample from a probability distribution $\mfX$, and $\vect{\theta}^* \in \mathbb{R}^p$ is the target parameter vector of interest. To infer the underlying parameter $\vect{\theta}^*$, assume that $N$ i.i.d. observations $\{X_1,...,X_N\}$ are collected and evenly distributed over $(m+1)$-machines $\{\mcH_0,\dots\mcH_m\}$, where each machine contains $n$ observations. We allow diverging $N$, $n$, and $p$ under certain rate constraints. 

{In this paper, we consider the Byzantine distributed framework, which allows for Byzantine failures described as follows. In particular, we assume there exists an $\alpha_n$ fraction of worker machines (a.k.a. Byzantine machines), whose indices form a subset $\mcB\subseteq\{1,\dots,m\}$ with $\card(\mcB)= \lfloor\alpha_n m\rfloor$. The Byzantine machines are subject to the following Byzantine failures when communicating information:
\begin{definition}[Byzantine Failures]
	In each round of communication, assume the information produced by each machine is $\vect{v}_j$, then the actual information $\bar{\vect{v}}_j$ received from each worker machine $\mcH_j$ is as follows
	\begin{equation*}
		\bar{\vect{v}}_j=
		\begin{cases}
			\vect{v}_j\quad &j\notin \mcB,\\
			*\quad&j\in\mcB,
		\end{cases}
	\end{equation*}
	where $*$ denotes an arbitrary value.
\end{definition} 
}

Let us start with the most fundamental setting where $\vect{\theta}^*$ is the population mean and the goal is to infer the population mean with the presence of Byzantine failures. A widely used robust estimator is the median-of-means (MOM) estimator \citep{nemirovsky_etal.1983, jerrum_etal.1986, Alon:99:Space}, which first computes the local sample mean on each machine and then aggregate them by taking the median. Due its ease of implementation and computational efficiency,  the MOM estimator has attracted a lot of attentions \citep{minsker.2015, hsu_sabato.2016,lecue_lerasle.2017, lugosi_mendelson_a.2019,minsker_strawn.2019} and served as an important building block in distributed learning with Byzantine failures (\cite{yin_etal.2018}).  However, despite its popularity, the MOM estimator suffers from low asymptotic statistical efficiency. More precisely, the asymptotic efficiency of the MOM estimator is only $2/\pi\approx0.637$, which is far from 1 for normal mean problem.

{The main contribution of the paper is to propose a computationally efficient robust mean estimator, which greatly improves the statistical efficiency of the MOM. Our estimator is called \emph{variance reduced median-of-means (VRMOM)} estimator.} Instead of using the median in MOM, we use multiple quantile levels to improve the statistical efficiency. By formulating a carefully designed stochastic optimization problem and leveraging the idea of one-step Newton iteration, our VRMOM estimator achieves the same order of computational complexity as the MOM estimator, but improves the asymptotic efficiency from $2/\pi\approx0.637$ of  MOM to $3/\pi\approx0.955$ (see Theorem \ref{thm:normality_vrmom}). The proposed VRMOM estimator naturally serves as a more efficient substitute of MOM in all robust statistical applications that benefit from the MOM estimator.

{As an application of our VRMOM estimator, we describe a communication-efficient algorithm for the general risk minimization problem in \eqref{eq:true_para} based on the VRMOM estimator.} In a standard distributed gradient descent (GD) approach, each local machine computes the gradient information, which takes the form of the \emph{mean of gradients of each local data point}. Then, the master receives the transmitted gradient information and aggregates the local gradients by taking the average. However, the averaged gradient is highly sensitive to Byzantine failure, whose value can be completely skewed by a single Byzantine worker. To hedge against Byzantine failures, the work by \cite{yin_etal.2018} proposed to take the coordinate-wise \emph{median} of the transmitted gradients, which is essentially an MOM estimator based on gradients of local data. Our method improves this result from two aspects. First, instead of using the median, our VRMOM serves as a new \emph{gradient aggregator}, which is statistically more efficient in a large class of  distributed robust inference problems. Second, the distributed gradient method would take a large number of iterations (i.e., $O(\log (N/p))$) to converge, which is communicationally expensive. To address this issue, we leverage the  surrogate loss function in the {Communication-efficient Surrogate Likelihood (CSL)} framework \citep{jordan_etal.2019} and develop the \emph{robust CSL (RCSL)} method. In a wide range of choices of $N$, $m$, and $p$,  our RCSL only requires a constant order of iterations to achieve a fast convergence rate, which greatly saves the total communication cost. Theoretically, we establish the convergence rates of our estimator (see Theorem \ref{thm:iter_rdane_byzantine_conv} and Theorem \ref{thm:rdane_byzantine_conv} in Appendix \ref{sec:theory:RCSL}) and provide the asymptotic normality result (see Theorems \ref{cor:VRMOM_limit}).


\subsection{Contributions and Related Works}
\label{sec:related}

The median-of-means (MOM), which was introduced by \cite{nemirovsky_etal.1983}, has been a popular estimator in robust statistics due to its ease of implementation and convergence guarantees \citep{minsker.2015, hsu_sabato.2016, lecue_lerasle.2017, lugosi_mendelson_a.2019,minsker_strawn.2019}. The MOM estimator finds a wide range of applications, including robust PCA \citep{minsker.2015}, linear regression \citep{hsu_sabato.2016}, sparse linear regression \citep{minsker.2015,lecue_lerasle.2017}, robust empirical risk minimization \citep{lecue_lerasle.2017,lugosi_mendelson_a.2019,minsker_strawn.2019}. This paper improves the MOM estimator by proposing a variance reduction scheme, which significantly boosts the statistical efficiency. Our VRMOM estimator is motivated by the idea that composite quantiles can improve the efficiency \citep{zou_yuan.2008}. However, directly taking multiple sample quantiles would incur a higher computational cost and our VRMOM is carefully designed to be computationally efficient and admit a simple closed-form for the ease of theoretical analysis.  
The proposed VRMOM estimator can be a natural substitute for the classical MOM estimator for all aforementioned applications.

In recent years, statistical learning and optimization with the presence of Byzantine failures have attracted a lot of attentions \citep{feng_etal.2014, chen_su_etal.2017,blanchard_etal.2017,xie_koyejo_gupta.2018,alistarh_zhu_li.2018, yin_etal.2018,yin_chen_etal.2018,su_xu.2018}. The key idea behind these work is to let each worker machine compute the gradient (or stochastic gradient) information, and the gradients from workers are aggregated using some robust mean estimators instead of the vanilla gradient mean. There are many applicable estimators like median, trimmed mean  \citep{yin_etal.2018,yin_chen_etal.2018}, geometric median \citep{feng_etal.2014,chen_su_etal.2017}, Krum \citep{blanchard_etal.2017}, marginal median, mean-around-median \citep{xie_koyejo_gupta.2018}, and iterative filtering \citep{su_xu.2018,yin_chen_etal.2018}. However, most existing methods are only based on gradient information, without utilizing any second order properties. In this paper, we propose the robust CSL method, which combines the new gradient aggregator --- VRMOM estimator, and the approximate-Newton framework (See, e.g. \cite{shamir_etal.2013, jordan_etal.2019}, and \cite{fan_guo_etal.2019}). The combination of the VRMOM and approximate-Newton greatly facilitates communication efficient estimation by reducing the total number of communication rounds. From a theoretical perspective, we establish the asymptotic normality result, which has not been well explored in previous robust distributed learning literature. A more detailed comparison of the convergence rates with the existing approaches is presented after Theorem \ref{thm:iter_rdane_byzantine_conv}, after the formal description of our convergence result.


\subsection{Paper Organization and Notations}

The rest of the paper is organized as follows. {Section \ref{sec:model} describes the proposed VRMOM estimator and its theoretical results. In Section \ref{sec:robust-csl}, we introduce the Robust CSL (RCSL) method for Byzantine robust machine learning problem as an important application of the VRMOM estimator. Simulation experiments are provided in Section \ref{sec:sim}, which demonstrate the superiority of our method over some existing methods. Finally, we conclude our work in Section \ref{sec:conclude}. The proofs of the theories of the VRMOM estimator and the theories of the RCSL method are relegated to Appendices.}

For every vector $\vect{v}=(v_1,...,v_p)^{\rm T}$, denote $|\vect{v}|_2=\sqrt{\sum_{l=1}^pv_l^2}$. For every matrix $\vect{A}$, define $\Norm{\vect{A}}=\sup_{|\vect{v}|_2=1}|\vect{A}\vect{v}|_2$ as the operator norm, $\Lambda_{\max}(\vect{A})$ and $\Lambda_{\min}(\vect{A})$ as the largest and smallest eigenvalues of $\vect{A}$ respectively. Suppose there is another matrix $\vect{B}$, and we denote $\vect{B}\preceq\vect{A}$ if and only if $\vect{A}-\vect{B}$ is positive definite. Let $\mcN(0,1)$ be the standard normal distribution. We denote $\Phi(x)= \mbP(\mcN(0,1)\leq x)$ and $\psi(x)= e^{-x^2/2}/\sqrt{2\pi}$ to be its cumulative distribution function and probability density function, respectively. Denote $\mbS^{p-1}(\vect{\theta})$ and $\mbB^p(\vect{\theta})$ as the unit sphere and the unit ball centered at $\vect{\theta} \in\mbR^p$ respectively. For simplicity, we denote $\mbS^{p-1}$ and $\mbB^p$ as unit sphere and unit ball centered at $\vect{0}$. We will use $\mbI(\cdot)$ as the indicator function. The symbols $\lfloor x\rfloor$ ($\lceil x\rceil$) denotes the greatest integer (the smallest integer) not larger than (not less than) $x$. Summation symbol will be heavily used throughout this article. For the convenience of reading, in each summand, we will use the subscripts $i  (1\leq i\leq n)$ for each data point, $j (0\leq j\leq m)$ for each machine, $k (1\leq k\leq K)$ for each quantile level and $l (1\leq l\leq p)$ for each entry of a vector, respectively. Lastly, the generic constants are assumed to be independent of $m,n,$ and $p$.


\section{Proposed Methods}\label{sec:model}

{In this section, we will firstly introduce the construction  of our VRMOM estimator. Then we provide theoretical guarantees for it.} 

\subsection{Variance Reduced Median-of-Means Estimator}
\label{sec:vrmom-estimator}

To motivate our estimator, let us provide a brief review of the standard MOM estimator. Let $X_1,...,X_N$ be \emph{i.i.d.} {copies of $X$} with $\mbE(X)=\mu$ and $\var(X)=\sigma^2$. For the ease of presentation, we assume that $N$ observations are evenly partitioned into $(m+1)$-batches $\{\mcH_0,\dots\mcH_m\}$, where each $\mcH_j$ denotes the indices of the samples within the $j$-th batch. Let $n=N/(m+1)$ be the sample size of each batch and $\bar{X}_j=\sum_{i\in \mcH_j}X_i/n$ be the sample mean of the observations in the $j$-th batch. To estimate the population mean $\mu$, the MOM estimator is defined as
	\begin{equation}	\label{eq:mom_estimator}
\hat{\mu}=\med(\bar{X}_0,...,\bar{X}_{m}),
\end{equation}
where $\med(\cdot)$ denotes the sample median. The MOM estimator is computationally efficient and robust against Byzantine failures. Moreover, as shown in \cite{minsker_strawn.2019},  when $m \rightarrow \infty$ and $m=o(\sqrt{N})$,  under some mild moment conditions (e.g., $\mbE|X-\mu|^3<\infty$), the MOM estimator admits the following limiting distribution as $N \rightarrow \infty$,
\begin{equation*}
	\sqrt{N}(\hat{\mu}-\mu)\xrightarrow{d}\mcN(0,\frac{\pi}{2}\sigma^2).
\end{equation*}
In addition to the robustness, the statistical efficiency is another important issue.   In a classical statistical estimation setting without Byzantine failures, we can see the relative efficiency of $\hat{\mu}$ with respect to the vanilla sample mean is $(\sigma^2)/(\frac{\pi}{2}\sigma^2) =  2/\pi\approx0.637$, which is far from the optimal efficiency 1. Therefore, a natural question is:

\bigskip
\addtolength{\fboxsep}{5pt}
\begin{boxedminipage}{.9\textwidth}
\emph{Is it possible to construct a computationally efficient robust estimator that achieves a nearly-optimal efficiency?}
\end{boxedminipage}
\bigskip

\paragraph{The key idea behind our VRMOM estimator}
To address this challenge, we first note that by the central limit theorem, for each sample mean $\bar{X}_j$, $\sqrt{n}\bar{X}_j$ asymptotically obeys the normal distribution $\mcN(\mu,\sigma^2)$. Moreover, for the ease of notation, for a fixed $n$, we define 
\begin{equation}	\label{eq:uni_model}
 \bar{X} = \mu+ \epsilon, \qquad \epsilon \sim \mcN(0,\sigma^2/n),
\end{equation}
where $\mu$ and $\sigma$ are unknown.
{Note that for every quantile level $\tau$, the $\tau$-th population quantile of the normal distribution $\mcN(\mu,\sigma^2/n)$ is exactly $\mu^{\tau}:=\mu+\sigma\Phi^{-1}(\tau)/\sqrt{n}$. To see this, 
\[
\mbP(\bar{X} \leq \mu^{\tau}) = \mbP((\bar{X}-\mu)/(\sigma/\sqrt{n}) \leq \Phi^{-1}(\tau))=\mbP(N(0,1)\leq \Phi^{-1}(\tau))=\tau.
\]	
Additionally, by symmetry of normal distribution, we have 
\begin{equation*}
	\frac{1}{2}(\mu^{\tau}+\mu^{1-\tau})=\frac{1}{2}\Big[2\mu+\frac{\sigma}{\sqrt{n}}\left\{\Phi^{-1}(\tau)+\Phi^{-1}(1-\tau)\right\}\Big]=\mu.
\end{equation*}} 
Therefore, to improve the statistical efficiency, a natural idea is to approximate $\mu$ by averaging many pairs of estimators for the $(\tau,1-\tau)$-th quantiles of $\mcN(\mu,\sigma^2/n)$, instead of using a single quantity (i.e., median). More precisely, let $K$ be a pre-fixed integer. For any $1 \leq k \leq K$, let $\tau_k:=k/(K+1)$ and $\bar{\mu}^{\tau_k}$ be the $\tau_k$-th sample quantile of $\{\bar{X}_0,\dots,\bar{X}_m\}$. Since $\tau_{K+1-k}=1- \tau_k$, $\bar{\mu}^{\tau_k}$ and $\bar{\mu}^{\tau_{K+1-k}}$ are symmetrical about $\mu$ and their average is a natural estimator of $\mu$. Based on this idea, we can take weighted average of $\{\bar{\mu}^{\tau_k}\}_{k=1}^K$ as an estimator of $\mu$, which improves the statistical efficiency. We also illustrate the main idea of the weighted averaged estimator in Figure \ref{fig:quantile_wmean}. Next, we introduce a computationally more efficient  estimator for  implementing this idea. Moreover, since it is a closed-form estimator, which also facilitates the theoretical analysis.

\begin{figure}[!t]
	\begin{center}
		\includegraphics[width=0.8\textwidth]{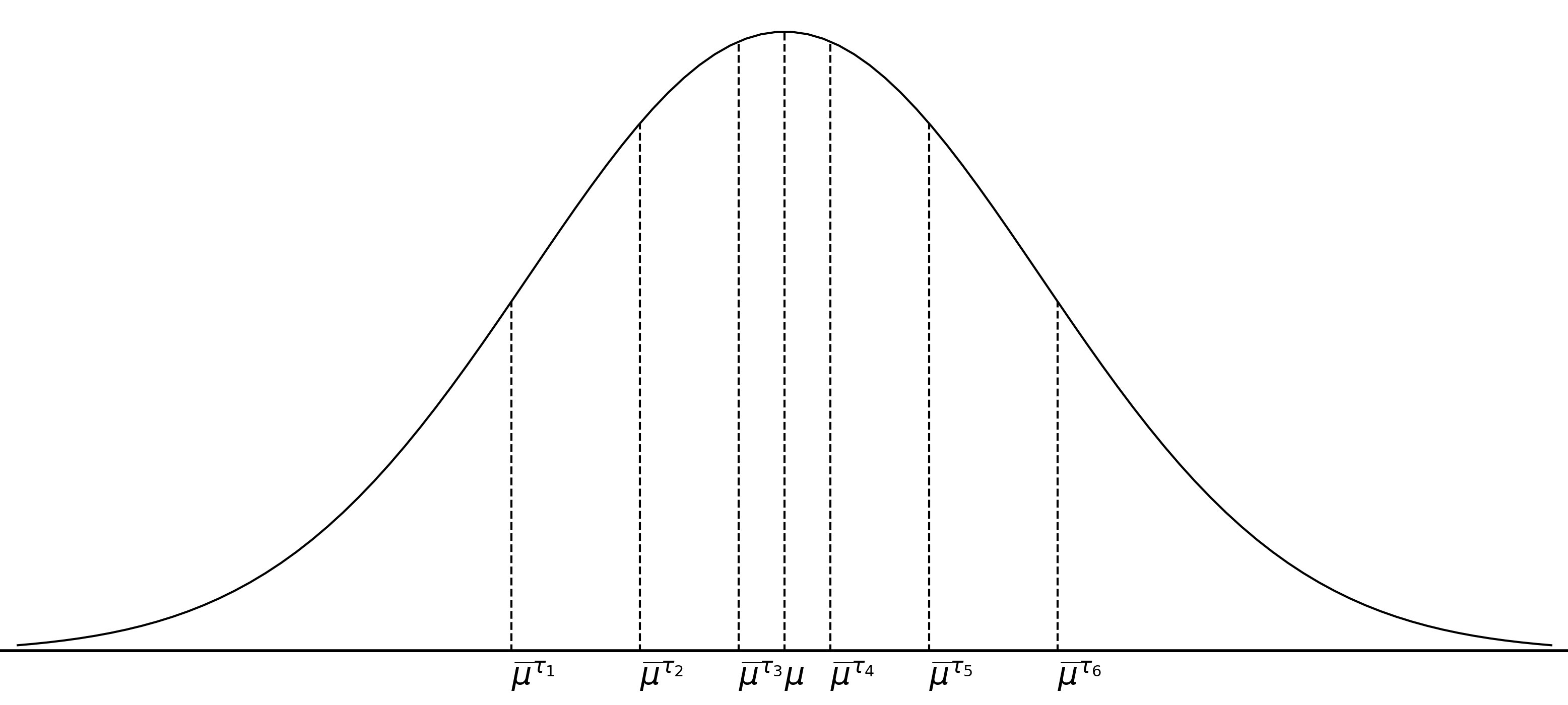}
	\end{center}
	\caption{Let $K=6$. For $1\leq k\leq K$,  $\bar{\mu}^{\tau_k}$ is defined as the $k/7$-th sample quantile of $\{\bar{X}_0,\dots,\bar{X}_m\}$. We can see that the pairs $(\bar{\mu}^{\tau_1},\bar{\mu}^{\tau_6}), (\bar{\mu}^{\tau_2},\bar{\mu}^{\tau_{5}}), (\bar{\mu}^{\tau_3},\bar{\mu}^{\tau_{4}})$ are nearly symmetrical about the targeting parameter $\mu$. Thus the weighted average of $\{\bar{\mu}^{\tau_k}\}_{k=1}^K$ serves as an estimator of $\mu$.}
	\label{fig:quantile_wmean}
\end{figure}

\paragraph{Computationally efficient VRMOM estimator}
Denote the quantile loss function as $\rho_{\tau}(z)=z(\tau-\mbI(z\leq0))$, we consider the following stochastic optimization problem:
\begin{equation}	\label{eq:compo_model}
	\argmin{x\in\mbR}\left[\mbE\left\{G(\bar{X},x)\right\}\right]:=
	\argmin{x\in\mbR}\left[\mbE\left\{\sum_{k=1}^K\rho_{\tau_k}\left(\bar{X}-\frac{\sigma\Delta_k}{\sqrt{n}}-x\right)\right\}\right],
\end{equation}
where $\Delta_k:=\Phi^{-1}(\tau_k)$, and the expectation is taken over $\bar{X}$ in \eqref{eq:uni_model}. We can easily see that $\mu$ is the solution of \eqref{eq:compo_model}. To approximate $\mu$ from \eqref{eq:compo_model}, we adopt the idea of one-step estimator as follows. Define 
\begin{align*}
	g(x):=&\frac{\diff}{\diff x}\mbE\left\{G(\bar{X},x)\right\}=\mbE\left\{\sum_{k=1}^K\mbI\left(\bar{X}\leq x+\frac{\sigma\Delta_k}{\sqrt{n}}\right)-\tau_k\right\},\\
	H(x):=&\frac{\diff}{\diff x}g(x)=\sum_{k=1}^K\mfp\left(x-\mu+\frac{\sigma\Delta_k}{\sqrt{n}}\right),
\end{align*}
as the gradient and Hessian of the loss function in \eqref{eq:compo_model} respectively. Here, $\mfp(x)=\sqrt{n}\psi(\sqrt{n}x/\sigma)/\sigma$ denotes the probability density function of the noise $\epsilon \sim N(0, \sigma^2/n)$ in \eqref{eq:uni_model}, and $\psi(\cdot)$ denotes the density function of $\mcN(0,1)$. Given an initial crude estimator $\mu_0$ of $\mu$, the one-step estimator essentially takes the following Newton-Raphson step:
\begin{equation}	\label{eq:newton_iteration}
	\tilde{\mu}_1:=\mu_0-g(\mu_0)/H(\mu_0)=\mu_0-\frac{\mbE\{\sum_{k=1}^K\mbI\left(\bar{X}\leq \mu_0+\sigma\Delta_k/\sqrt{n}\right)-\tau_k\}}{\sum_{k=1}^K\mfp(\mu_0-\mu+\sigma\Delta_k/\sqrt{n})}.
\end{equation}
Next, we use the MOM estimator $\hat{\mu}$ in \eqref{eq:mom_estimator} as the initial estimator of \eqref{eq:compo_model}. The unknown parameter $\sigma^2$ can be estimated by $\hat{\sigma}^2:=\sum_{i\in\mcH_0}(X_i-\bar{X}_0)^2/n$, the sample variance of the first batch of observations $\mcH_0$. With the initial estimator $\mu_0$ in place, replacing $g(\mu_0)$ in \eqref{eq:newton_iteration} with its empirical counterpart and approximating $H(\mu_0)$ by $\sum_{k=1}^K\mfp(\hat{\sigma}\Delta_k/\sqrt{n})\approx\sqrt{n}\sum_{k=1}^K\psi(\Delta_k)/\hat{\sigma}$, we derive the following one-step estimator of \eqref{eq:compo_model} from \eqref{eq:newton_iteration}
\begin{equation}	\label{eq:vrmom_def}
	\bar{\mu}=\hat{\mu}-\frac{\hat{\sigma}}{(m+1)\sqrt{n}\sum_{k=1}^K\psi(\Delta_k)}\sum_{j=0}^m\sum_{k=1}^K\left\{\mbI\left(\bar{X}_j\leq \hat{\mu}+\frac{\hat{\sigma}\Delta_k}{\sqrt{n}}\right)-\frac{k}{K+1}\right\}.
\end{equation}
To further alleviate the burden of computation, we choose one summand in \eqref{eq:vrmom_def} and simplify it as follows
\begin{eqnarray*}
		&&\sum_{k=1}^K\left\{\mbI\left(\bar{X}_j\leq\hat{\mu}+\frac{\hat{\sigma}\Delta_k}{\sqrt{n}}\right)-\frac{k}{K+1}\right\}\cr
		&=&\sum_{k=1}^K\mbI\left(\bar{X}_j\leq\hat{\mu}+\frac{\hat{\sigma}\Delta_k}{\sqrt{n}}\right)-\frac{K}{2}\cr
		&=&\card\left\{k\;:\;\frac{\sqrt{n}(\bar{X}_j-\hat{\mu})}{\hat{\sigma}}\leq\Delta_k,\; 1\leq k\leq K \right\}-\frac{K}{2}\cr
		&=&\card\left\{k\;:\;\Phi\left(\frac{\sqrt{n}(\bar{X}_j-\hat{\mu})}{\hat{\sigma}}\right)\leq \frac{k}{K+1},\; 1\leq k\leq K \right\}-\frac{K}{2}\cr
		&=&\frac{K}{2}+1-\left\lceil(K+1)\Phi\left(\frac{\sqrt{n}(\bar{X}_j-\hat{\mu})}{\hat{\sigma}}\right)\right\rceil.
	\end{eqnarray*}
In fact, our theoretical result will show that a larger $K$ leads to a better statistical efficiency. This derivation shows that it is possible to enhance the efficiency by taking a larger $K$ without incurring additional computational cost. Our VRMOM estimator in \eqref{eq:vrmom_def} can be rewritten as
	\begin{equation}	\label{eq:vrmom_def2}
		\bar{\mu}=\hat{\mu}-\frac{\hat{\sigma}}{(m+1)\sqrt{n}\sum_{k=1}^K\psi(\Delta_k)}\sum_{j=0}^{m}\left\{\frac{K}{2}+1-\left\lceil(K+1)\Phi\left(\frac{\sqrt{n}(\bar{X}_j-\hat{\mu})}{\hat{\sigma}}\right)\right\rceil\right\}.
	\end{equation} 

Although the expression of VRMOM $\bar{\mu}$ in \eqref{eq:vrmom_def2} seems more complicated than the MOM estimator $\hat{\mu}$,  the time complexity of $\bar{\mu}$ is at the same order as the MOM estimator. In particular, at each iteration, each worker machine computes a local sample mean in parallel, which takes $O(n)$ time complexity. In MOM estimator, it takes another $O(m)$ operations to find the median $\hat{\mu}$ (see \cite{paterson_1997}). While in VRMOM estimator, we only need an extra $O(m+n+K)$ time complexity ($O(n)$ for sample variance computed in $\mcH_0$, $O(m+K)$ for the variance reduction term in \eqref{eq:vrmom_def2}). Therefore the time complexity of both methods is $O(m+n)$ (here $K$ is fixed). While keeping the same order of computational complexity, the VRMOM greatly improves the statistical efficiency.  As we can see from Theorem \ref{thm:normality_vrmom} in the following section, the asymptotic efficiency of $\bar{\mu}$ approaches $3/\pi\approx0.955$ as $K$ grows to infinity, which is nearly optimal. In fact, by taking $K=5$, the efficiency has already been more than $0.9$ as compared to $0.637$ of $\hat{\mu}$. 

\begin{remark}
	As illustrated in Figure \ref{fig:quantile_wmean}, we can also find the $\tau_k$-th sample quantile $\bar{\mu}^{\tau_k}$ and take a weighted average of $\{\bar{\mu}^{\tau_k}\}_{k=1}^K$ as an estimator of $\mu$.  However, to give the sample quantiles at $K$ different levels, we need to perform a sorting algorithm among the set $\{\bar{X}_j\}_{j=0}^m$, which takes $O(m\log m)$ operations. Therefore the total complexity would be $O(n+m\log m)$, which it is more costly compared with $O(m+n)$ complexity of our VRMOM estimator in \eqref{eq:vrmom_def2}. The inferior in complexity is exacerbated in multivariate case. When the dimension $p$ is very large, our coordinate-wise VRMOM estimator has complexity $O(p(m+n))$, while the average of sample quantiles would take complexity $O(p(m\log m + n))$. 
\end{remark}

\begin{remark}
	Although we utilize averaging sample quantiles to motivate our method, the direct average among sample quantiles cannot tolerate even a small fraction of Byzantine machines. For example, we assume $K=6$, and there are $1/6$ of machines are Byzantine. In this case, the $1/7$-th and $6/7$-th sample quantiles can be completely ruined by the Byzantine machines, and further foil the weighted average. In contrast, our VRMOM estimator tolerates $1/2-\delta$ (where $\delta\in(0,1/2)$ can be arbitrarily small) fraction of Byzantine machines, which is better than the direct weighted average of sample quantiles (see Theorem \ref{thm:concentration_vrmom_byzantine} for more details). To see this in a more intuitive way, we can take a closer look at  \eqref{eq:vrmom_def}. Each summand of the correction term is bounded in the interval $[-1,1]$. Noticing that there is a factor of order $O(1/(m\sqrt{n}))$ multiplying the summation, the overall magnitude of the correction term is only of the order $O(1/\sqrt{n})$. In consequence, as long as the initial estimator (e.g., the MOM estimator) is robust, our proposed VRMOM estimator is Byzantine robust.
\end{remark}

\begin{remark}
	To approximate $\mu$ from \eqref{eq:compo_model}, we can also directly solve the following optimization problem
	\begin{equation*}
		\argmin{x\in\mbR}\left[\sum_{k=1}^K\sum_{j=0}^m\rho_{\tau_k}\left(\bar{X}_j-\frac{\hat{\sigma}\Delta_k}{\sqrt{n}}-x\right)\right],
	\end{equation*}
	which is the empirical version of \eqref{eq:compo_model}. This formula is similar as the univariate composite quantile regression in \cite{zou_yuan.2008}. However, it is much costly to solve such non-smooth optimization problem in a direct way. Instead, we leverage the idea of Newton-Raphson step, which greatly improves computation efficiency.
\end{remark}


\subsection{Theories for VRMOM Estimator}
\label{sec:theory:VRMOM}
{Now, we present several theoretical results for the proposed VRMOM estimator. We firstly provide the asymptotic normality and convergence result of the VRMOM estimator in one-dimensional case. Then we extend these results to its multi-dimensional variant. The proofs of results in this section are all relegated in Appendix \ref{sec:proof}.}

\begin{theorem}[Asymptotic normality of VRMOM]	\label{thm:normality_vrmom}
	Let $N=(m+1)n$ i.i.d. random variables $X_1,...,X_{N}$ be evenly distributed in $m+1$ subsets $\mcH_0,...,\mcH_m$. There is a subset of Byzantine machine indices $\mcB\subseteq\{1,\dots,m\}$ with $\card(\mcB)=\lfloor\alpha_nm\rfloor$, where  $\alpha_n=o(m^{-1/2})$. Let
	\begin{equation}\label{eq:Byzantine}
		\bar{X}_j=
		\begin{cases}
			 \frac{1}{n}\sum_{i\in\mcH_j}X_i\quad&j\notin\mcB,\\
			 \quad*\quad&j\in\mcB,
		\end{cases}
	\end{equation}
	and $\bar{\mu}$ be defined as in \eqref{eq:vrmom_def}. Suppose $X$ satisfies $\mbE(X)=\mu,\var(X)=\sigma^2$. Assume there exists some $\kappa>0$ such that $\mbE[|X-\mu|^{2+\kappa}]<\infty$, and $m=o(\min\{n,n^{2\kappa/(2+\kappa)}\}),\log^3 n=o(m)$. Then we have
	\begin{equation*}
		\sqrt{N}(\bar{\mu}-\mu)\xrightarrow{d}\mcN(0,\sigma^2_K),
	\end{equation*}
	where 
	\begin{equation}	\label{eq:vrmom_variance}
		\sigma_K^2=\frac{\sum_{k_1,k_2=1}^K\min(\tau_{k_1},\tau_{k_2})\{1-\max(\tau_{k_1},\tau_{k_2})\}}{\{\sum_{k=1}^K\psi(\Delta_k)\}^2}\sigma^2,
	\end{equation}
	with $\tau_k=k/(K+1)$. Moreover, $\lim_{K\rightarrow\infty}\sigma_K^2=\pi\sigma^2/3$.
\end{theorem}

{ This theorem provides the asymptotic normality result of our VRMOM estimator $\bar{\mu}$ and characterizes the asymptotic variance.}
In particular, it shows that $\bar{\mu}$ is a consistent estimator of $\mu$. Comparing with the MOM estimator, we \emph{improve the efficiency} of the estimator by reducing the variance from $\pi\sigma^2/2$ \citep{minsker_strawn.2019} to $\pi\sigma^2/3$ when $K$ goes to infinity. It should be noted that, we impose the rate constraints $\alpha_n=o(m^{-1/2}), m=o(\min\{n,n^{2\kappa/(2+\kappa)}\})$, and $\log^3 n=o(m)$ in order to obtain asymptotic normality. In the following theorem, we drop out these conditions and investigate the convergence rate of the VRMOM estimator.


\begin{theorem}[Convergence rate of VRMOM]	\label{thm:concentration_vrmom_byzantine}
	Let $N=(m+1)n$ i.i.d. random variables $X_1,...,X_{N}$ be evenly distributed in $m+1$ subsets $\mcH_0,...,\mcH_m$. There is a subset of Byzantine machine indices $\mcB\subseteq\{1,\dots,m\}$ with $\card(\mcB)=\lfloor\alpha_nm\rfloor$, where  $\alpha_n\leq1/2-\delta$ for some fixed $\delta\in(0,1/2)$. Let
	\begin{equation*}
		\bar{X}_j=
		\begin{cases}
			 \frac{1}{n}\sum_{i\in\mcH_j}X_i\quad&j\notin\mcB,\\
			 \quad*\quad&j\in\mcB,
		\end{cases}
	\end{equation*}
	and $\bar{\mu}$ be defined as in \eqref{eq:vrmom_def}. Suppose $X$ satisfies $\mbE(X)=\mu,\text{Var}(X)=\sigma^2$. Assume there exists some $\kappa>0$ such that $\mbE[|X-\mu|^{2+\kappa}]<\infty$. Then  we have
	\begin{equation}	\label{concentration_mmm_byzantine.ineq}
		|\,\bar{\mu}-\mu|= O_{\mbP}\left(\frac{\alpha_n}{\sqrt{n}}+\frac{1}{\sqrt{mn}}+\frac{1}{n^{(3\kappa_2+2)/(2\kappa_2+4)}}+\frac{\log^{3/4}n}{n^{1/2}m^{3/4}}\right),
	\end{equation}
	where $\kappa_2=\min(\kappa,2)$.
\end{theorem}

{This convergence result shows that our VRMOM estimator is consistent as long as $\alpha_n$ is strictly smaller than $1/2$.} The condition on $\alpha_n$ is also necessary because clearly the sample median can be ruined when there are more than $\lceil m/2\rceil$ corruptions. From \eqref{concentration_mmm_byzantine.ineq}, when $m=O(\min\{n,n^{2\kappa/(2+\kappa)}\}), \log^3 n=O(m)$, the rate matches the optimal rate $O(\alpha_n/\sqrt{n}+1/\sqrt{mn})$ (See Observation 1 in \cite{yin_etal.2018}). Further assume that $\alpha_n=O(1/\sqrt{m})$, the VRMOM achieves square root-$N$ consistency.

{Next we extend our VRMOM estimator to  the multi-dimensional extension setting. Let $\vect{X}_1,\dots,\vect{X}_N$ be i.i.d. copies of the $p$-dimensional random vectors $\vect{X}=(X^{(1)},\dots,X^{(p)})^{\tp}$ with $\mbE(\vect{X})=\vect{\mu}=(\mu^{(1)},\dots,\mu^{(p)})^{\tp}$ and $\cov(\vect{X})=\vect{\Sigma}=(\sigma_{l_1,l_2})_{l_1,l_2=1}^p$. Then the multi-dimensional VRMOM estimator $\bar{\vect{\mu}}$ is defined by applying \eqref{eq:vrmom_def2} on each coordinate $l$, where $1\leq l\leq p$. We first obtain the convergence rate of the multi-dimensional VRMOM estimator in terms of $\ell_2$-norm.}

{
\begin{theorem}[Convergence rate of multi-dimensional VRMOM]	\label{thm:concen_multi_vrmom}
	Let $N=(m+1)n$ i.i.d. random vectors $\vect{X}_1,...,\vect{X}_{N}$ be evenly distributed in $m+1$ subsets $\mcH_0,...,\mcH_m$. There is a subset of Byzantine machine indices $\mcB\subseteq\{1,\dots,m\}$ with $\card(\mcB)=\lfloor\alpha_nm\rfloor$, where  $\alpha_n\leq1/2-\delta$ for some fixed $\delta\in(0,1/2)$. 
	Let $\bar{\vect{\mu}}$ be the multi-dimensional VRMOM estimator defined in \eqref{eq:vrmom_def2}. Suppose $\vect{X}$ satisfies $\mbE(\vect{X})=\vect{\mu},\cov(\vect{X})=\vect{\Sigma}$. Moreover, for each coordinate $l\in\{1,\dots,p\}$, we assume there exists some $\kappa>0$ such that $\mbE[|X^{(l)}-\mu^{(l)}|^{2+\kappa}]<\infty$. Then we have
	\begin{equation}	\label{concen_multi_vrmom.ineq}
		|\,\bar{\vect{\mu}}-\vect{\mu}|_2= O_{\mbP}\left(\frac{\alpha_n\sqrt{p}}{\sqrt{n}}+\sqrt{\frac{p}{mn}}+\frac{\sqrt{p}}{n^{(3\kappa_2+2)/(2\kappa_2+4)}}+\frac{p^{1/2}\log^{3/4}n}{n^{1/2}m^{3/4}}\right),
	\end{equation}
	where $\kappa_2=\min(\kappa, 2)$.
\end{theorem}
}

{As we can see from the theorem, the convergence rate of multi-dimensional VRMOM estimator is simply the rate in Theorem \ref{thm:concentration_vrmom_byzantine} multiplied with $\sqrt{p}$, which is not surprising because the VRMOM estimator is applied coordinate-wisely. Moreover, to guarantee consistency of the proposed estimator, we require the rate on the right hand side of \eqref{concen_multi_vrmom.ineq} to be the order of $o_{\mbP}(1)$, which implies that
\begin{equation}	\label{eq:multi_vrmom_conv_rate}
	p=o\Big(\min\Big\{\frac{n}{\alpha^2_n},mn,n^{\frac{3\kappa_2+2}{\kappa_2+2}},\frac{nm^{3/2}}{\log^{3/2}n}\Big\}\Big).
\end{equation} 
In particular, we are more interested in the asymptotic normality of the multi-dimensional VRMOM estimator, which will  be presented in the next theorem.}

{By definition we know $\sigma_{l_1,l_2}=\cov\{X^{(l_1)},X^{(l_2)}\}$ is the $(l_1,l_2)$-entry of covariance matrix of $\vect{X}$. Let $(Z_{l_1},Z_{l_2})$ admit the following bivariate normal distribution 
	\begin{equation}	\label{eq:bivariate_normal}
		\mcN\left(\vect{0},\bm{\Sigma}_{l_1,l_2} \right),\quad\text{where}\quad\bm{\Sigma}_{l_1,l_2}=
			\begin{pmatrix}
				1	&	\frac{\sigma_{l_1,l_2}}{\sqrt{\sigma_{l_1,l_1}\sigma_{l_2,l_2}}}	\\
				\frac{\sigma_{l_1,l_2}}{\sqrt{\sigma_{l_1,l_1}\sigma_{l_2,l_2}}}	&	1	
			\end{pmatrix}.
	\end{equation}
	Next we define the $p \times p$ matrix $\vect{\mathcal{C}}$ with its $(l_1,l_2)$-entry given by the following formula:
	\begin{equation}	\label{eq:normal_cov_entry}
		\mathcal{C}_{l_1,l_2}=\frac{\sum_{k_1,k_2=1}^K(\tau^{l_1,l_2}_{k_1,k_2}-\tau_{k_1}\tau_{k_2})}{\{\sum_{k=1}^K\psi(\Delta_k)\}^2}\sqrt{\sigma_{l_1,l_1}\sigma_{l_2,l_2}},
	\end{equation}
	where $\tau_k=k/(K+1)$,  and $\tau^{l_1,l_2}_{k_1,k_2}=\mbP(Z_{l_1}\leq \Delta_{k_1},Z_{l_2} \leq \Delta_{k_2})$. Then we can prove the following asymptotic normality result:
\begin{theorem}[Asymptotic normality of multi-dimensional VRMOM estimator]	\label{cor:VRMOM_limit}
	Under the same assumption as in Theorem \ref{thm:concen_multi_vrmom}, and additionally, we assume the rate constraints $p=o(\min\{\frac{m^{1/2}}{\log^{3/2}n},\frac{n^{2\kappa_2/(\kappa_2+2)}}{m}\})$, and $\alpha_n=o(1/\sqrt{mp})$. Then for any vector $\vect{v}\in\mbR^p$ with $|\vect{v}|_2 = 1$, we have that
	\begin{equation}	\label{eq:multi_vrmom_normality}
		\frac{\sqrt{N}}{\sigma_{\vect{v}}}\left\langle \vect{v},\,\bar{\vect{\mu}}-\vect{\mu}\right\rangle\xrightarrow{d}\mcN(0,1),
	\end{equation}	
	as $n\to\infty$, where $\sigma^2_{\vect{v}}=\vect{v}^{\tp}\vect{\mathcal{C}}\vect{v}$.
\end{theorem}
}

{To prove asymptotic normality result for the multi-dimensional VRMOM estimator, we need a more restrictive constraint on the dimension $p$ than the one in \eqref{eq:multi_vrmom_conv_rate}. 
Moreover, we require the number of Byzantine machines is $o(\sqrt{m/p})$, i.e., the fraction $\alpha_n=o(1/\sqrt{mp})$. As compared to the condition $\alpha_n=o(1/ \sqrt{m})$ in Theorem \ref{thm:normality_vrmom}, there is an extra $1/\sqrt{p}$ in the condition since we are dealing with a $p$-dimensional multivariate inference problem. }

{In order to illustrate the efficiency of our VRMOM estimator in multi-dimensional case, let us consider the multi-dimensional median-of-means (MOM) estimator $\hat{\vect{\mu}}_{\mathrm{MOM}}$. More specifically, we also apply the MOM estimator at each coordinate and establish the following parallel asymptotic normality result for the multi-dimensional MOM estimator.}

{
\begin{proposition}[Asymptotic normality of multi-dimensional MOM estimator]	\label{cor:MOM_limit}
	Under the same assumption as in Theorem \ref{thm:concen_multi_vrmom}, and additionally, we assume the rate constraints $p=o(\min\{\frac{m^{1/2}}{\log^{3/2}n},\frac{n^{2\kappa_2/(\kappa_2+2)}}{m}\})$, and $\alpha_n=o(1/\sqrt{mp})$. Then for any vector $\vect{v}\in\mbR^p$ with $|\vect{v}|_2 = 1$, we have that
	\begin{equation}	\label{eq:normality_cor_med}
		\frac{\sqrt{N}}{\sigma_{\mathrm{MOM},\vect{v}}}\left\langle \vect{v},\,\hat{\vect{\mu}}_{\rm MOM}-\vect{\mu}\right\rangle\xrightarrow{d}\mcN(0,1),
	\end{equation}
	as $n\to\infty$, where $\sigma^2_{\mathrm{MOM},\vect{v}}=\vect{v}^{\tp}\vect{\mathcal{C}}_{\rm MOM}\vect{v}$, and $\vect{\mathcal{C}}_{\rm MOM}$ is a $p\times p$ matrix with each $(l_1,l_2)$-entry taking the following form,
	\begin{equation}	\label{eq:med_normal_entry}
		\mathcal{C}_{\mathrm{MOM},l_1,l_2}=\left(2\pi\tau^{l_1,l_2}_{(K+1)/2,(K+1)/2}-\frac{\pi}{2}\right)\sqrt{\sigma_{l_1,l_1}\sigma_{l_2,l_2}}.
	\end{equation}
\end{proposition}
}

{When each coordinate of random vector $\vect{X}$ is independent, the off-diagonal entries of the matrices $\vect{\mathcal{C}}$ and $\vect{\mathcal{C}}_{\rm MOM}$ are  all  zero (in this case $\tau_{k_1,k_2}^{l_1,l_2}=\tau_{k_1}\tau_{k_2}$ for $l_1\neq l_2$). For diagonal entries, we can readily compute that
	\begin{equation*}
	\mathcal{C}_{l,l}=\frac{\sum_{k_1,k_2=1}^K\min(\tau_{k_1},\tau_{k_2})\{1-\max(\tau_{k_1},\tau_{k_2})\}}{\{\sum_{k=1}^K\psi(\Delta_k)\}^2}\sigma_{l,l},\quad\mathcal{C}_{\mathrm{MOM},l,l}=\frac{\pi}{2}\sigma_{l,l}.
	\end{equation*}
According to Theorem \ref{thm:normality_vrmom}, when $K \rightarrow \infty$, we have $\mathcal{C}_{l,l} \rightarrow \frac{\pi}{3}\sigma_{l,l}$, which suggests that our multi-dimensional VRMOM estimator has a higher statistical efficiency than the corresponding MOM estimator.}
{
\begin{remark}	\label{rem:positive_def}
	In two-dimensional case, the covariance matrix of $\vect{X}$ can be written as the following form
	\begin{equation*}
		\vect{\Sigma}=
		\begin{pmatrix}
			\sigma_{1,1}	&	\sin\phi\sqrt{\sigma_{1,1}\sigma_{2,2}}\\
			\sin\phi\sqrt{\sigma_{1,1}\sigma_{2,2}}	&	\sigma_{2,2}
		\end{pmatrix},
	\end{equation*} 
	where $\phi\in[-\frac{\pi}{2},\frac{\pi}{2}]$. Then we have that our 2-dimensional $\mathrm{VRMOM}$ estimator $\bar{\vect{\mu}}$ has higher statistical efficiency than the 2-dimensional $\mathrm{MOM}$ estimator $\hat{\vect{\mu}}$ as $K$ tends to infinity. The detailed argument is relegated to Appendix \ref{sec:posit_def}. In the  higher dimension case when $p>2$, we believe that the superiority in efficiency of our $\mathrm{VRMOM}$ estimator still holds. We leave the theoretical investigation as a future work.
\end{remark}
}


\section{{Application for Byzantine Distributed Statistical Optimization}}
\label{sec:robust-csl}

{As an important application of the proposed VRMOM estimator, in this section, we consider the general distributed statistical optimization problem in \eqref{eq:true_para} under Byzantine setup. In particular, we propose a Byzantine robust distributed approximate newton method, called Robust Communication-efficient Surrogate Likelihood (RCSL) Method.}


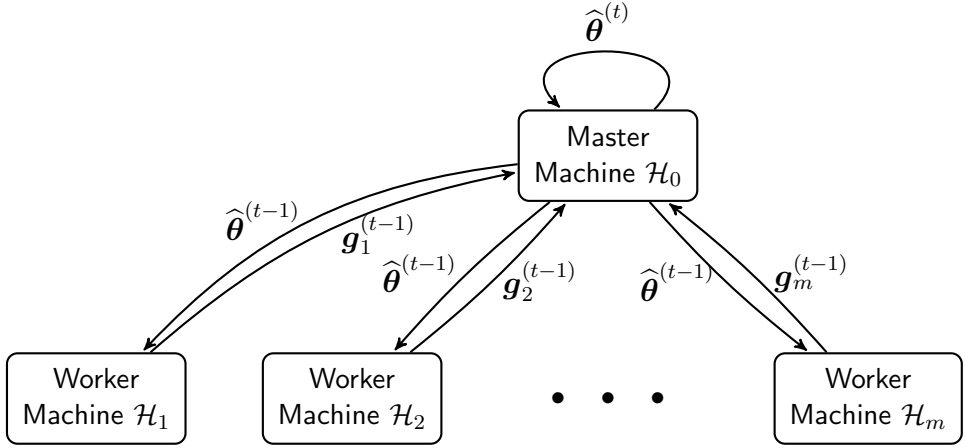
\begin{figure}[!t]
	\begin{center}
		\begin{tikzpicture}[
		font=\sffamily,
		every matrix/.style={ampersand replacement=\&,column sep=1cm,row sep=2cm},
		process/.style={draw,thick,rounded corners, inner sep=.2cm},
		dots/.style={gray,scale=2},
		to/.style={->,>=stealth',shorten >=1pt,thick,font=\sffamily\large},
		every node/.style={align=center}]
		
		\matrix{
			\&	\&	\node[process] (master) {Master\\ Machine $\mathcal{H}_0$}; \& \& \\
			
			\node[process] (worker1) {Worker\\ Machine $\mathcal{H}_1$};	\&\node[process] (worker2) {Worker\\ Machine $\mathcal{H}_2$};	\& \node[datastore] (buffer) {$\bullet\quad\bullet\quad\bullet$};\&
			\node[process] (workerm) {Worker\\ Machine $\mathcal{H}_m$};\\
		};
		
		\draw[to] (master)  to[bend right=20]
		node[midway,left] {$\hat{\vect{\theta}}^{(t-1)}$} (worker1);
		\draw[to] (master) to[bend right=5]
		node[midway,left] {$\hat{\vect{\theta}}^{(t-1)}$} (worker2);
		\draw[to] (master) to[bend right=5]
		node[midway,left] {$\hat{\vect{\theta}}^{(t-1)}$} (workerm);
		\draw[to] (worker1) to[bend left=15]
		node[midway,right=4] {$\vect{g}^{(t-1)}_1$} (master);
		\draw[to] (worker2) to[bend right=5]
		node[midway,right] {$\vect{g}^{(t-1)}_2$} (master);
		\draw[to] (workerm) to[bend right=5]
		node[midway,right=5] {$\vect{g}^{(t-1)}_m$} (master);
		\draw[to] (master) to[loop, above, looseness=3]
		node[above] {$\hat{\vect{\theta}}^{(t)}$} (master);
		\end{tikzpicture}
	\end{center}
	\caption{Communication protocol of the robust CSL (RCSL) method. In the $t$-th iteration, the master machine $\mcH_0$ distributes the parameter $\hat{\vect{\theta}}^{(t-1)}$ to each worker machine. The $j$-th worker machine computes the local gradient $\vect{g}_j^{(t-1)}$ and sends it back to the master machine. Then $\mcH_0$ updates the new parameter $\hat{\vect{\theta}}^{(t)}$ and repeats the procedure.}
	\label{fig:distribute}
\end{figure}

For the ease of presentation, we adopt the master/worker setting in \cite{jordan_etal.2019}, where $\mcH_0$ denotes the master machine and the rest are worker machines. We assume that the master machine $\mcH_0$ stores $n$ observations as each local worker and the data on the master machine will not be corrupted. In practice, it is easier to use one powerful machine as the master machine that is robust. 
Let $\hat{\vect{\theta}}^{(0)}$ be an initial estimator of $\vect{\theta}^*$. At the beginning, the master machine $\mcH_0$ broadcasts the parameter $\hat{\vect{\theta}}^{(0)}$ to all worker machines $\mcH_1,\dots,\mcH_{m}$. The $j$-th worker computes a local gradient $\vect{g}_j^{(0)}=n^{-1}\sum_{i\in \mcH_j}\nabla|_{\vect{\theta}} f(X_i,\hat{\vect{\theta}}^{(0)})$ and sends it back to master. Then the master machine applies the VRMOM estimator to every coordinate. More precisely, for each of the $l$-th coordinate (we will use the subscript $l$ to represent the entry of a vector), master machine computes the VRMOM estimator
\begin{equation}	\label{eq:gbar_def}
\begin{aligned}
	\bar{g}_l^{(0)}=&\hat{g}_l^{(0)}-\frac{\hat{\sigma}_l^{(0)}}{(m+1)\sqrt{n}\sum_{k=1}^K\psi(\Delta_k)}\times\\
	&\sum_{j=0}^{m}\left\{\frac{K}{2}+1-\left\lceil(K+1)\Phi\left(\frac{\sqrt{n}(g_{j,l}^{(0)}-\hat{g}_l^{(0)})}{\hat{\sigma}^{(0)}_l}\right)\right\rceil\right\},
\end{aligned}
\end{equation}
where $\hat{g}_l^{(0)}=\text{med}\{g_{0,l}^{(0)},...,g_{m,l}^{(0)}\}$ is the median and  
\begin{equation*}
	(\hat{\sigma}_l^{(0)})^{2}=\frac{1}{n}\sum_{i\in \mcH_0}\left\{\nabla|_{\vect{\theta}} f_l(X_i,\hat{\vect{\theta}}^{(0)})-g_{0,l}^{(0)}\right\}^2,
\end{equation*} 
is the local sample variance. Here, $\nabla|_{\vect{\theta}} f_l(X_i,\hat{\vect{\theta}}^{(0)})$ denotes the $l$-th coordinate of $\nabla|_{\vect{\theta}} f(X_i,\hat{\vect{\theta}}^{(0)})$. Let  $\bar{\vect{g}}^{(0)}=(\bar{g}_1^{(0)},\dots\bar{g}_p^{(0)})^{\rm T}$ denote the VRMOM aggregated gradient. The master machine solves the following surrogate loss introduced by \cite{jordan_etal.2019} to update the parameter,
\begin{equation}	\label{eq:rcsl_estimator}
	\hat{\vect{\theta}}^{(1)}=\argmin{\vect{\theta}\in\mbR^p}\left\{\frac{1}{n}\sum_{i\in \mcH_0}f(X_i,\vect{\theta})-\left\langle \vect{g}_0^{(0)}-\bar{\vect{g}}^{(0)},\vect{\theta}\right\rangle\right\},
\end{equation} 
where $\vect{g}_0^{(0)}=n^{-1}\sum_{i\in \mcH_0}\nabla|_{\vect{\theta}} f(X_i,\hat{\vect{\theta}}^{(0)})$ is the local gradient computed on the master machine. As shown in \cite{jordan_etal.2019} and our experiments, for a wide range of statistical learning problems, the surrogate loss in \eqref{eq:rcsl_estimator} can be easily minimized by existing optimization solvers. Moreover, this surrogate loss minimization is done on the master machine, and thus does not involve any communication.

\begin{algorithm}[!t]
	\caption{{\small Robust CSL (RCSL) Method}}
	\label{alg:RCSL}
	\hspace*{\algorithmicindent} \hspace{-0.7cm}   {\textbf{Input:}  The   data $\{X_1,...,X_N\}$  is evenly distributed on  $m+1$ machines $\{\mcH_0,\dots,\mcH_{m}\}$. Let $\mcH_0$ be the master and the rest be workers.} 	
	\begin{algorithmic}[1]
		\STATE Compute an initial estimator $\hat{\vect{\theta}}^{(0)}$ on the master machine $\mcH_0$. 
		\FOR{$t=1,\dots, T$}
		\STATE Distribute $\hat{\vect{\theta}}^{(t-1)}$ to each local machine $j=1,2,\dots,m$.
		\FOR{$j=0,\dots, m$}
		\STATE The $j$-th worker machine computes the local gradient
		\begin{equation*}
		\vect{g}_j^{(t-1)}=
		\begin{cases}
		n^{-1}\sum_{i\in\mcH_j}\nabla|_{\vect{\theta}} f(X_i,\hat{\vect{\theta}}^{(t-1)})\quad&\text{if $\mcH_j$ is normal},\\
		*	\quad&\text{if $\mcH_j$ is Byzantine},
		\end{cases}
		\end{equation*}
		where $*$ denotes arbitrary values. Then the $j$-th worker sends $\vect{g}_j^{(t-1)}$ back to master machine.
		\ENDFOR
		\STATE Master machine constructs the VRMOM aggregated gradient $\bar{\vect{g}}^{(t-1)}=(\bar{g}^{(t-1)}_1,\dots\bar{g}^{(t-1)}_p)^{\rm T}$, where each $l$-th coordinate takes the following form,
 
 		\begin{equation}\label{eq:gbar_iterate}
		\begin{aligned}
		\bar{g}_l^{(t-1)}=&\hat{g}_l^{(t-1)}-\frac{\hat{\sigma}_l^{(t-1)}}{(m+1)\sqrt{n}\sum_{k=1}^K\psi(\Delta_k)}\times\\
		&\sum_{j=0}^{m}\left\{\frac{K}{2}+1-\left\lceil(K+1)\Phi\left(\frac{\sqrt{n}(g_{j,l}^{(t-1)}-\hat{g}_l^{(t-1)})}{\hat{\sigma}^{(t-1)}_l}\right)\right\rceil\right\},
		\end{aligned}
		\end{equation}
		where $\hat{g}_l^{(t-1)}=\text{med}\{g_{0,l}^{(t-1)},...,g_{m,l}^{(t-1)}\}$, and
		\begin{equation*}
		\left(\hat{\sigma}^{(t-1)}_l\right)^2=\frac{1}{n}\sum_{i\in \mcH_0}\left\{\nabla|_{\vect{\theta}} f_l(X_i,\hat{\vect{\theta}}^{(t-1)})-g_{0,l}^{(t-1)}\right\}^2.
		\end{equation*}
 		\STATE Master machine solves the following surrogate loss minimization problem
		\begin{equation}\label{eq:rdane_roundt}	
		\hat{\vect{\theta}}^{(t)}=\argmin{\vect{\theta}\in\mbR^{p}}\left\{\frac{1}{n}\sum_{i\in\mcH_0}f(X_i,\vect{\theta})-\left\langle\,\vect{g}_0^{(t-1)}-\bar{\vect{g}}^{(t-1)},\,\vect{\theta}\,\right\rangle\right\}.
		\end{equation}
		\ENDFOR
	\end{algorithmic}
	 \textbf{Output:}  The final estimator $\hat{\vect{\theta}}^{(T)}$.
\end{algorithm}

Repeating the above procedure, we develop a multi-round algorithm named Robust CSL (RCSL), which is presented in Algorithm \ref{alg:RCSL}.  We note that in the Byzantine setting, there is a subset of workers $\mcB\subseteq\{1,\dots,m\}$ , which return arbitrary values in each iteration. For $j\in\mcB$, we will use $\vect{g}_j^{(t-1)}=*$ to represent these nuisance values. To guarantee the consistency of the initial estimator, in Step 1 of Algorithm \ref{alg:RCSL}, we can compute $\hat{\vect{\theta}}^{(0)}$ by the local empirical risk minimization on the master machine $\mcH_0$, i.e.,
\begin{equation}\label{eq:init}
\hat{\vect{\theta}}^{(0)}=\argmin{\vect{\theta}\in\mbR^p}\left\{\frac{1}{n}\sum_{i\in \mcH_0}f(X_i,\vect{\theta})\right\}.
\end{equation}
We note that our theoretical result only requires the consistency of the initial estimator, and thus other consistent estimators could also be used as the initial estimator. 

Now we briefly comment on the communication cost of Algorithm \ref{alg:RCSL}. In each round, the communication cost is $O(mp)$, which is at the same order as other gradient descent algorithms in the literature (e.g., \cite{yin_etal.2018,yin_chen_etal.2018}). {Moreover, from Theorem \ref{thm:iter_rdane_byzantine_conv} below, our RCSL only takes a constant number of rounds of communication, as opposed to the order of $O(\log (N/p))$ rounds in other gradient-based methods. Therefore, the total communication cost of RCSL is only $O(mp)$. For the sake of clarity, we only present the result of multi-round convergence rate of the RCSL method in the following. More detailed technical conditions and theoretical results are relegated to Appendix \ref{sec:rcsl_theory}.}

\begin{theorem}[Multi-round convergence rate of RCSL method]	\label{thm:iter_rdane_byzantine_conv}
	Suppose Assumption \ref{assump:convexity}-\ref{assump:mnp_dependence} (see Appendix \ref{sec:tech_assump}) hold and the initial estimator $\hat{\vect{\theta}}^{(0)}$ satisfies $|\hat{\vect{\theta}}^{(0)}-\vect{\theta}^*|_2=O_{\mbP}(r_n)$. Further assume the fraction $\alpha_n$ of Byzantine machines satisfies $\alpha_n\leq 1/2-\delta$ for some fixed $\delta\in(0,1/2)$. The RCSL estimator in the $t$-th iteration $\hat{\vect{\theta}}^{(t)}$ defined in (\ref{eq:rdane_roundt}) satisfies
	\begin{equation}	\label{eq:iter_rcsl_conv}
		|\hat{\vect{\theta}}^{(t)}-\vect{\theta}^*|_2= O_{\mbP}\left(\frac{\alpha_n\sqrt{p}}{\sqrt{n}}+\sqrt{\frac{p\log n}{mn}}+\frac{p^{1/2}\log^{3/4}n}{n^{1/2}m^{3/4}}+r_np^t\left(\frac{\log n}{n}\right)^{t/2}\right).
	\end{equation}
\end{theorem}
The proof of this theorem can be found in Appendix \ref{sec:theory:RCSL}. We note that the second term $\sqrt{p\log n/(mn)}$ in the right hand side of \eqref{eq:iter_rcsl_conv} matches the optimal rate $\sqrt{p/(mn)}$ up to a logarithmic factor. The third term is inherited from the last term in \eqref{concentration_mmm_byzantine.ineq} of Theorem \ref{thm:concentration_vrmom_byzantine}. It becomes $O(\sqrt{p\log n/(mn)})$ when $\log n=O(m)$. The first term is the price paid for the Byzantine failures. As we can see, after each iteration, the fourth term in \eqref{eq:iter_rcsl_conv} is improved by a factor $p\sqrt{\log n/n}$, which is of the order $o(1)$ by the rate constraints in Assumption \ref{assump:mnp_dependence} (see Appendix \ref{sec:tech_assump}). In stark contrast, for the Byzantine robust gradient descent, the convergence rate is only improved by a constant factor $c<1$ after each round of communication (See, e.g. Theorem 3 in \cite{yin_chen_etal.2018} and Theorem 3.4 in \cite{alistarh_zhu_li.2018}). Therefore, Theorem \ref{thm:iter_rdane_byzantine_conv} suggests that the proposed RCSL method enjoys faster convergence rate than vanilla Byzantine robust gradient descent, therefore is more communication-efficient. {We can also demonstrate the communication-efficiency of our RCSL method in another way. From the expression of \eqref{eq:iter_rcsl_conv}, we can see that when the number of $t$ is sufficiently large, i.e., 
\begin{equation}	\label{eq:iteration}
	t\geq\frac{\log m+\log p+2\log r_n}{\log n-2\log p-\log\log n}+1,
\end{equation}
the last term will become the order of $O(\sqrt{p\log n/(mn)})$. Moreover, with $p=O(n^{1/3}/\sqrt{\log n})$ and $r_n=o(1)$ in Assumption \ref{assump:mnp_dependence}, we have
\begin{equation*}
	\frac{\log m+\log p+2\log r_n}{\log n-2\log p-\log\log n}+1\leq \frac{3\log m}{\log n}+c_0,\quad \text{for some constant } c_0>0.
\end{equation*}
Therefore, when $t\geq c_0+3\log m/\log n$, the rate of $t$-th iteration is dominated by the first three terms. It would save a lot of communication cost as compared to gradient-based algorithms, where at least {$O(\log (N/p))$} iterations of communications are necessary.}

	Assume $\log n=O(m)$ and the number of iterations $t$ is sufficiently large, our rate in \eqref{eq:iter_rcsl_conv} will become $O_{\mbP}(\alpha_n\sqrt{p}/\sqrt{n}+\sqrt{p\log n/(mn)})$. It is interesting to compare this rate with contemporary results of the Byzantine perturbed gradient method with different aggregators (see Theorem 2 in \cite{yin_chen_etal.2018}). For example, median aggregator leads to the rate $\alpha_n\sqrt{p}/\sqrt{n}+p/\sqrt{mn}$ (we save a $\sqrt{p}$ in the second term), the trimmed-mean aggregator to the rate $\alpha_np/\sqrt{n}+p/\sqrt{mn}$ (saving a $\sqrt{p}$ in both terms), and the iterative filtering to the rate  $\sqrt{\alpha_n /n}+\sqrt{p/mn}$ (saving a  $\sqrt{\alpha_n }$ but losing a $\sqrt{p}$ in the first term). However, their filtering estimator involves solving convex programs iteratively. Moreover, in the theory of the filtering estimator, the Byzantine fraction $\alpha_n$ is required to be not larger than $1/4$ (see, \eg Theorem 1 in \cite{su_xu.2018} and Theorem 5 in \cite{yin_chen_etal.2018}), which is more restrictive than ours ($\alpha_n\leq 1/2-\delta$).
We would also like to note that the lower bound on the convergence rate is known to be $\Omega(\alpha_n/\sqrt{n}+\sqrt{p/mn})$. As compared to the lower bound, our upper bound is missing a $\sqrt{p}$ factor in the first term. When there is no Byzantine worker (i.e., $\alpha_n=0$) or the dimensionality $p$ is a constant, our rate is optimal upto logarithmic factors.  We believe this extra $\sqrt{p}$ in the first term comes out because the gradients are aggregated coordinate-wisely.
It would be interesting as a future direction to develop a new multi-variate aggregator based on VRMOM, which is both statistically and computationally efficient and achieves the optimal convergence rate. 

It is also worthwhile noting that we can extend our algorithm to a general scheme by replacing \eqref{eq:rcsl_estimator} with the following surrogate loss minimization, 
\begin{equation}	\label{eq:rdane_scheme}
	\argmin{\vect{\theta}\in\mbR^p}\left\{\frac{1}{n}\sum_{i\in \mcH_0}f(X_i,\vect{\theta})-\left\langle \vect{g}^{(0)}_0-\mathrm{Aggr}(\vect{g}^{(0)}_0,\dots,\vect{g}^{(0)}_m),\vect{\theta}\right\rangle\right\},
\end{equation}
where $\mathrm{Aggr}(\vect{g}^{(0)}_0,\dots,\vect{g}^{(0)}_m)$ can be any consistent estimator of the population gradient $\mbE\{ \nabla|_{\vect{\theta}}f(X,\hat{\vect{\theta}}^{(0)})\}$ given $\hat{\vect{\theta}}^{(0)}$. Any robust aggregator in the literature can be adopted in \eqref{eq:rdane_scheme}, e.g., median-of-means, trimmed mean \citep{yin_etal.2018}, geometric median \citep{feng_etal.2014, chen_su_etal.2017}, Krum \citep{blanchard_etal.2017}, marginal median \citep{xie_koyejo_gupta.2018}. In the original CSL framework \citep{jordan_etal.2019}, the aggregator is chosen as the coordinate-wise average, which is sensitive to corruptions. 

\begin{remark}
{It is worthwhile noting that when the target parameter admits some specific structures, e.g., sparsity structure, it is straightforward to extend the proposed framework \eqref{eq:rdane_scheme} to the following regularized problem
\begin{equation}	\label{eq:regularized_rcsl}
	\argmin{\vect{\theta}\in\mbR^p}\left\{\frac{1}{n}\sum_{i\in \mcH_0}f(X_i,\vect{\theta})-\left\langle \vect{g}^{(0)}_0-\mathrm{Aggr}(\vect{g}^{(0)}_0,\dots,\vect{g}^{(0)}_m),\vect{\theta}\right\rangle+\lambda_n\mcR(\vect{\theta})\right\},
\end{equation}
where $\mcR(\vect{\theta})$ is some regularizer, for example, the $\ell_1$-penalty \citep{tibshirani1996regression}, smooth clipped absolute deviation (SCAD) \citep{fan_li.2001}, and minimax concave penalty (MCP) \citep{zhang.2010}. With the above formulation, we are able to address the sparse learning problem in Byzantine robust setup. 
We leave more rigorous theoretical investigation of the penalized Byzantine robust estimation to future research. }
\end{remark}




\section{Simulation Studies}\label{sec:sim}

In our simulation studies, we conduct several experiments to show the effectiveness of the VRMOM and the robust CSL (RCSL) Method.

\subsection{Results for VRMOM}

In this section, we show the performance of the proposed VRMOM estimator for robust mean estimation problem. We first demonstrate how the number of quantile levels $K$ in \eqref{eq:vrmom_def2} affects the estimation accuracy of the VRMOM estimator, and then compare the statistical efficiency of our VRMOM estimator and the MOM estimator defined in \eqref{eq:mom_estimator}. We generate the random vectors $X_i$s from the normal distribution $\mathcal{N}(\bmu^{\ast},\bSigma_{X})$ where $\bmu^{\ast}=p^{-1/2}(1,(p-2)/(p-1),(p-3)/(p-1),\dots,0)$  and $\bSigma_{\bX}=\text{diag}(1,\dots,1)$. We choose $p=1$ and $p=30$ to consider both univariate and multivariate cases. The entire sample size is $N=1000 \times (100+1)$. By dividing the data into one master machine $\mcH_0$ and $100$ worker machines $\{\mcH_1,\dots,\mcH_{100}\}$. Note that the master machine $\mcH_0$ can never be corrupted in our setup. 
Therefore, each local sample size is $n=1000$. We consider the following settings: (1) $\alpha_n=0$ which denotes no Byzantine machine, (2) $\alpha_n>0$, which denotes the existence of Byzantine machine case. We vary the fraction of Byzantine machines $\alpha_n=0.05,0.1,0.15$. When $\alpha_n>0$, we replace the sample means in each Byzantine machine by a random vector whose entries are generated from $\mathcal{N}(0,200\mbI)$ independently. For each experiment, we repeat 500 independent simulations and report averaged root mean square estimation errors and standard deviations.

\subsubsection{Effect of $K$}

{In the first experiment, we vary the number of quantile levels $K$ from $\{10,20,50,100\}$ and investigate the estimation accuracy of the VRMOM estimator for different dimensions $p\in\{1,30\}$ and different fractions $\alpha_n\in\{0,0.05,0.1,0.15\}$. The results of the root mean square errors and the standard errors are presented in Table \ref{tab:choiceK}. As we can see from the table, for each fraction of Byzantine machine $\alpha_n$, the root mean square errors of the VRMOM estimator for different $K$s are almost the same. Based on this observation, in the following experiment, we fix $K$ to be $10$ for the ease of computation.}

\begin{table}[b]	
	\centering{
		\caption{\small The root mean square errors (RMSEs) and their standard errors (in parentheses) of the VRMOM under sample size $N=1000 \times  (100+1)$, local sample size $n=1000$ and number of quantile levels $K=10,20,50,100$. The sample means sent from Byzantine machines are generated from Gaussian $\mathcal{N}(0,200 \mbI)$.}\label{tab:choiceK}
		\bigskip
		\small
		\begin{tabular}{r|r|rrrr}
			\hline
			$p$ & $K$  & $\alpha_n=0$ & $\alpha_n=0.05$ & $\alpha_n=0.1$ & $\alpha_n=0.15$ \\ 
			\hline
			\multirow{2}{*}{1} & \text{10} & 0.0025 (0.0018) & 0.0027 (0.0020) & 0.0030 (0.0022) & 0.0032 (0.0024) \\ 
			& \text{20} & 0.0025 (0.0018) & 0.0027 (0.0021) & 0.0030 (0.0022) & 0.0034 (0.0026) \\ 
			& \text{50} & 0.0025 (0.0018) & 0.0028 (0.0021) & 0.0030 (0.0022) & 0.0033 (0.0024) \\ 
			& \text{100} & 0.0025 (0.0018) & 0.0028 (0.0020) & 0.0030 (0.0022) & 0.0032 (0.0024) \\
			\hline
			\multirow{2}{*}{30}  & \text{10} & 0.0175 (0.0022) & 0.0192 (0.0024) & 0.0209 (0.0026) & 0.0227 (0.0028) \\ 
			& \text{20} & 0.0174 (0.0022) & 0.0192 (0.0024) & 0.0209 (0.0026) & 0.0228 (0.0030) \\ 
			& \text{50} & 0.0174 (0.0022) & 0.0192 (0.0025) & 0.0208 (0.0026) & 0.0230 (0.0029) \\ 
			& \text{100} & 0.0174 (0.0022) & 0.0192 (0.0024) & 0.0208 (0.0027) & 0.0230 (0.0030) \\ 
			\hline
		\end{tabular}
	}
\end{table}

\subsubsection{Comparison between VRMOM and MOM}

\begin{table}[t]
	\centering{
	\caption{\small The root mean square errors (RMSEs) and their standard errors (in parentheses) of the VRMOM and MOM, and the ratios of RMSEs between VRMOM and MOM under sample size $N=1000 \times  (100+1)$, local sample size $n=1000$ and integer $K=10$. The sample means sent from Byzantine machines are generated from Gaussian $\mathcal{N}(0,200 \mbI)$.}\label{tab:vrmomratio}
	\bigskip
	\small
	\begin{tabular}{r|r|rrrr}
		\hline
		 $p$ &   & $\alpha_n=0$ & $\alpha_n=0.05$ & $\alpha_n=0.1$ & $\alpha_n=0.15$ \\ 
		\hline
		\multirow{2}{*}{1} & \text{VRMOM} & 0.0025 (0.0018) & 0.0027 (0.0020) & 0.0030 (0.0022) & 0.0032 (0.0024) \\ 
        & \text{MOM} & 0.0030 (0.0022) & 0.0031 (0.0023) & 0.0033 (0.0025) & 0.0035 (0.0026) \\
		& \text{Ratio} & 0.8613 & 0.8901 & 0.9044 & 0.9129 \\
		\hline
	    \multirow{2}{*}{30}  & \text{VRMOM} & 0.0175 (0.0022) & 0.0192 (0.0024) & 0.0209 (0.0026) & 0.0227 (0.0028) \\ 
        & \text{MOM} & 0.0211 (0.0028) & 0.0223 (0.0028) & 0.0234 (0.0030) & 0.0249 (0.0034) \\ 
		& \text{Ratio} & 0.8285 & 0.8601 & 0.8921 & 0.9108 \\ 
		\hline
	\end{tabular}
	}
\end{table}
 
In the second experiment, we compare the performance of the VRMOM estimator and the MOM estimator in terms of the root mean square errors and their standard errors. We fixed the total sample size as $N=1000\times(100+1)$, local sample size $n=1000$. We let the dimension $p\in\{1,30\}$ and the fraction of Byzantine machines varies from $\alpha_n\in\{0,0.05,0.1,0.15\}$. Throughout our experiment, we fix the number of quantiles $K$ in \eqref{eq:vrmom_def2} to be $K=10$. 

From Table \ref{tab:vrmomratio}, we observe that VRMOM has smaller root mean square errors than MOM as all the ratios are greater than 1. With the increase of the fraction of Byzantine machines, both methods has larger root mean square errors. The difference between VRMOM and MOM tends to be smaller with more Byzantine machines. {It is interesting to note that, when the dimension $p$ is $30$, the ratio of the root mean square errors between VRMOM and MOM is smaller than that when $p=1$. It suggests that the variance reduction effect of our VRMOM estimator becomes better for higher dimensions, although we have only proved the superiority when $p=1$ and $2$ in this paper (see Remark \ref{rem:positive_def}). }

\subsection{Results for Robust CSL Method}	\label{subsec:rcsl}
In this section, we consider the linear model and logistic regression model to demonstrate our robust CSL method. 
\paragraph{Settings for the linear regression model} For the linear model experiment, the data are generated as follows:
\[
Y_{i}=\vect{X}_i^{\rm T}\btheta^{\ast}+\epsilon_{i},\qquad i=1,2,\dots,n,
\]
where each $\vect{X}_i=(X_{i,1},\dots,X_{i,p})^{\rm T}$ is a $p$-dimensional covariate vector and $(X_{i,1},\dots,X_{i,p})$s are drawn $i.i.d.$ from a multivariate normal distribution $\mathcal{N}(0,\bm{\Sigma}_{X})$. The covariance  $\bm{\Sigma}_{X}$ is a symmetric Toeplitz matrix with $\Sigma_{ij}=0.5^{|i-j|}$ for $1\le i,j \le p$, which enforces correlation structure among covariates. We fix the dimension $p=30$ and generate the entries of the true coefficient vector $\btheta^{\ast}$ to be $p^{-1/2}(1,(p-2)/(p-1),(p-3)/(p-1),\dots,0)$. Similar to the previous experiment, the entire sample size is $N=1000 \times (100+1)$ and we divide the data into one master machine $\mcH_0$ and $100$ worker machines $\{\mcH_1,\dots,\mcH_{100}\}$. 
so that each local sample size $n=1000$.  We consider the standard normal noise distribution where the noise $\epsilon_{i}\sim\mathcal{N}(0,1)$. For the initial estimator $\widehat{\btheta}^{(0)}$, according to \eqref{eq:init}, we use the least square estimator with the data only on the master machine $\mcH_0$, i.e.,
$
\widehat{\btheta}^{(0)} =(\sum_{i\in\mcH_0}\vect{X}_i\vect{X}_i^{\rm T})^{-1}(\sum_{i\in\mcH_0}\vect{X}_iY_i).
$
The computation of the initial estimator is very efficient since it has closed form and does not require any communication. We also note that for the surrogate loss minimization problem \eqref{eq:rdane_roundt} at $t-$th iteration, we directly obtain the closed-form solution $\hat{\vect{\theta}}^{(t)}=(2n^{-1}\sum_{i\in\mcH_0}\vect{X}_i\vect{X}_i^{\rm T})^{-1}(2n^{-1}\sum_{i\in\mcH_0}\vect{X}_iY_i+\vect{g}_0^{(t-1)}-\bar{\vect{g}}^{(t-1)})$, which is also computationally efficient. Moreover, we generate corrupted gradients sent from Byzantine machines from the following attack models, 

\begin{itemize}
	\item[(a)] Gaussian attack: We replace the gradient vectors in the Byzantine machines by random vectors in which all the entries are generated from $\mathcal{N}(0,200\mbI)$ independently.
	
	\item[(b)] Omniscient attack:  For the Byzantine machines, we replace the true gradient vectors by the scaled negative gradients where the scale constant is extremely large ($1e10$ in our experiment).
	
	\item[(c)] Bit-flip attack:  For the Byzantine machines, we replace the true gradient vectors by flipping the sign fo the first five dimensions. 
\end{itemize}

\paragraph{Settings for the logistic regression model} 
For the logistic regression model experiment, the data are generated from the following:
\[
Y_{i}=
	\begin{cases}
		\quad1	\quad & \text{with probability }\mcL(\vect{X}_i^{\rm T}\btheta^{\ast}),\\
		\quad0	\quad & \text{with probability }1-\mcL(\vect{X}_i^{\rm T}\btheta^{\ast}),
	\end{cases}\qquad i=1,2,\dots,n
\]
where the link function $\mcL(x) = {e^x}/{(1 + e^x)}$ and each $\vect{X}_i=(X_{i,1},\dots,X_{i,p})^{\rm T}$ is a $p$-dimensional covariate vector which is drawn $i.i.d.$ from a multivariate normal distribution $\mathcal{N}(\bm{\mu}_x,\bm{\Sigma}_{X})$, which is coincident with the setting in the linear regression model. We choose $\bm{\mu}_x=(\mu_x,\dots,\mu_x)^{\rm T}$ and adopt two settings that $\mu_x=0$ and $\mu_x=0.5$. Here $\mu_x=0$ corresponds to the balanced response case that $50\%$ $Y_i$s are $1$ and $50\%$ $Y_i$s are $0$. And $\mu_x=0.5$ corresponds to the imbalanced response case where $76\%$ $Y_i$s are $1$ and $24\%$ $Y_i$s are $0$. We also fix the dimension $p=30$ and adopt the same true coefficient vector $\btheta^{\ast}$ as before. For each setting, we repeat 500 independent simulations and report averaged root mean square estimation errors and standard deviations. For the initial estimator $\widehat{\btheta}^{(0)}$, we use the logistic regression estimator with the data only on master machine $\mcH_0$, i.e.,
$
\widehat{\btheta}^{(0)}=\argmin{\btheta\in\mbR^p}\left[\frac{1}{n}\sum_{i\in \mcH_0}\left\{\log\left(1+e^{\vect{X}_i^{\rm T}\vect{\theta}}\right)-Y_i\vect{X}_i^{\rm T}\vect{\theta}\right\}\right].
$
The surrogate loss minimization problem \eqref{eq:rdane_roundt} at $t-$th iteration, i.e., 
	\begin{align*}
	\hat{\btheta}^{(t)}=&\argmin{\btheta\in\mbR^p}\Big[\frac{1}{n}\sum_{i\in \mcH_0}\Big\{\log\left(1+e^{\vect{X}_i^{\rm T}\vect{\theta}}\right)\Big\}\\
	&-\Big\{\Big(\frac{1}{n}\sum_{i\in \mcH_0}Y_i\vect{X}_i^{\rm T}\Big)+\vect{g}_0^{(t-1)}-\bar{\vect{g}}^{(t-1)}\Big\}\vect{\theta}\Big],
	\end{align*}
can be efficiently solved by standard gradient descent or quasi-Newton approaches \citep{Nocedal:06:numerical} on the center machine without any communication. 

As for the attack mode of Byzantine machines in the logistic regression model, we simulate the transmitted message in the following way. We replace every response $Y$ by $1-Y$ and compute the gradients based on these transformed local data on each Byzantine machine.

\begin{table}[t]
	\centering{
	\caption{\small The root mean square errors (RMSEs) and their standard errors (in parentheses) of the RCSL and MOM-RCSL, and the ratios of RMSEs between RCSL and MOM-RCSL under sample size $N=1000 \times  (100+1)$, local sample size $n=1000$ and integer $K=10$.  The corrupted gradients sent from Byzantine machines are generated from Gaussian, Omniscient and Bit-flip attacks.  The tolerance parameter for the stopping rule is set to $e_{r}=10^{-4}$. }\label{tab:linearratio}
	\bigskip
	\small
	\begin{tabular}{r|rrr}
		\hline
		Attack & \multicolumn{3}{c}{None} \\
		$\alpha_n$ & \multicolumn{3}{c}{$0$} \\ 
		\hline
		\text{RCSL} & \multicolumn{3}{c}{0.0231 (0.0036)}  \\ 
		\text{MOM-RCSL} & \multicolumn{3}{c}{0.0319 (0.0050)}  \\ 
		\text{Ratio} & \multicolumn{3}{c}{0.7243}  \\ 
		\hline
		Attack &\multicolumn{3}{c}{Gaussian} \\
		$\alpha_n$ & $0.05$ & $0.1$ & $0.15$ \\ 
		\hline
		\text{RCSL} & 0.0270 (0.0044) & 0.0312 (0.0049) & 0.0351 (0.0060) \\ 
		\text{MOM-RCSL} & 0.0343 (0.0054) & 0.0369 (0.0058) & 0.0398 (0.0063) \\ 
		\text{Ratio} & 0.7863 & 0.8434 & 0.8817 \\
		\hline
		Attack & \multicolumn{3}{c}{Omniscient}  \\
		$\alpha_n$ & $0.05$ & $0.1$ & $0.15$  \\
		\hline
		\text{RCSL} & 0.0276 (0.0042) & 0.0328 (0.0051) & 0.0396 (0.0060) \\
		\text{MOM-RCSL} & 0.0355 (0.0057) & 0.0395 (0.0061) & 0.0449 (0.0069) \\ 
		\text{Ratio} & 0.7774 & 0.8296 & 0.8815 \\
		\hline
		Attack & \multicolumn{3}{c}{Bit-flip} \\
		$\alpha_n$ &  $0.05$ & $0.1$ & $0.15$ \\
		\hline
		\text{RCSL} & 0.0236 (0.0037) & 0.0242 (0.0039) & 0.0250 (0.0041) \\ 
		\text{MOM-RCSL} & 0.0325 (0.0051) & 0.0334 (0.0053) & 0.0343 (0.0058) \\ 
		\text{Ratio} & 0.7276 & 0.7248 & 0.7291 \\ 
		\hline
	\end{tabular}
	}
\end{table}

\paragraph{More settings in the simulation} 
{
For the fraction of Byzantine machines, we consider the following settings: (1) $\alpha_n=0$ which denotes no Byzantine machine, (2) $\alpha_n>0$, which denotes the existence of Byzantine machine case. We vary the fraction of Byzantine machines $\alpha_n=0.05,0.1,0.15$. We also discuss the stopping criteria for Algorithm \ref{alg:RCSL}. Throughout the experiments, we use the tolerance parameter $e_{r}=10^{-4}$ as the stopping criterion. In particular,  at the $t$-th iteration of Algorithm \ref{alg:RCSL}, we compute $e=|\hat\btheta^{(t)}-\hat\btheta^{(t-1)}|_2^2/|\hat\btheta^{(t-1)}|_2^2$ and stop the algorithm once $e\le e_{r}$. In our experiments, it only requires $4$ to $8$ iterations to stop. We also provide the results with simple fixed number of iterations with $T=5$ and $T=10$.
}

{Since the main focus of this paper is the variance reduction effect of the proposed VRMOM estimator, in the simulation study, we mainly compare the performance of our Robust CSL algorithm (RCSL) with the MOM-based Robust CSL algorithm (MOM-RCSL). More specifically, let $\hat{\vect{\theta}}^{(0)}$ be the initial estimator obtained in the master machine $\mcH_0$, the refined estimator $\hat{\vect{\theta}}^{(1)}_{\rm{MOM}}$ is defined as the solution of \eqref{eq:rdane_scheme} with   $\mathrm{Aggr}(\vect{g}^{(0)}_0,\dots,\vect{g}^{(0)}_m)$ being the MOM aggregator of the gradients. Repeat the procedure, and we denote the $t$-th round MOM-RCSL estimator by $\hat{\vect{\theta}}^{(t)}_{\rm{MOM}}$.}


\subsubsection{Results for linear regression model}
\label{sec:linear}

The results for linear regression model are presented in  Table \ref{tab:linearratio} (for adaptive stopping criterion) and \ref{tab:linearfixratio} (for fixed number of iterations). 
From Table \ref{tab:linearratio},  RCSL has smaller root mean square errors than MOM-RCSL and all the ratios are greater than 1 for different kinds of attacks. With the increase of Byzantine fractions $\alpha_n$, both methods have larger root mean square errors. The difference between RCSL and MOM-RCSL tends to be smaller with more Byzantine machines, {which is coincident with the phenomenon of the mean estimation problem (See Table \ref{tab:vrmomratio}).} Among these three attack models, it seems that the omniscient attacker has the largest root mean square error in general. It is quite natural since this attacker makes the parameter vector go into the opposite direction by negative gradients with an extremely large scale factor (i.e., $1e10$), which exerts the most negative impact on the gradient aggregation step. On the other hand, even for such a strong attack model, our RCSL method still performs quite well. { It is also interesting to compare Table \ref{tab:vrmomratio} with the Gaussian attacker part in Table \ref{tab:linearratio} and \ref{tab:linearfixratio}. All simulations share the same attack mode and the same dimension. It seems that the variance reduction effect of VRMOM is more significant when it is performed as an iterative gradient aggregator than as a mean estimator.}

\begin{table}[p]
	\centering{
	\caption{\small The root mean square errors (RMSEs) and their standard errors (in parentheses) of the RCSL and MOM-RCSL, and the ratios of RMSEs between RCSL and MOM-RCSL under sample size $N=1000 \times  (100+1)$, local sample size $n=1000$ and integer $K=10$. The iteration numbers $T$ are fixed to be $5$ and $10$. The corrupted gradients sent from Byzantine machines are generated from Gaussian, omniscient and bit-flip attacks. }\label{tab:linearfixratio}
	\bigskip
	\small
	\begin{tabular}{r|r|rrr}
		\hline
		\multirow{2}{*}{$T$} & Attack & \multicolumn{3}{c}{None}  \\
		& $\alpha_n$ & \multicolumn{3}{c}{$0$}  \\  
		\hline
		\multirow{3}{*}{$5$} & \text{RCSL} & \multicolumn{3}{c}{0.0231 (0.0036)}  \\ 
		& \text{MOM-RCSL} & \multicolumn{3}{c}{0.0319 (0.0051)}  \\ 
		& \text{Ratio} & \multicolumn{3}{c}{0.7236}   \\ 
		\hline
		\multirow{3}{*}{$10$} & \text{RCSL} & \multicolumn{3}{c}{0.0231 (0.0036)}  \\
		& \text{MOM-RCSL} & \multicolumn{3}{c}{0.0319 (0.0051)}  \\ 
		& \text{Ratio} & \multicolumn{3}{c}{0.7233}  \\ 
		\hline
		\multirow{2}{*}{$T$} & Attack & \multicolumn{3}{c}{Gaussian} \\
		& $\alpha_n$ &  $0.05$ & $0.1$ & $0.15$ \\  
		\hline
		\multirow{3}{*}{$5$} & \text{RCSL}  & 0.0271 (0.0043) & 0.0310 (0.005) & 0.0354 (0.0057) \\ 
		& \text{MOM-RCSL} & 0.0343 (0.0056) & 0.0373 (0.006) & 0.0400 (0.0063) \\ 
		& \text{Ratio} & 0.7897 & 0.8305 & 0.8859 \\
		\hline
		\multirow{3}{*}{$10$} & \text{RCSL}  & 0.0272 (0.0042) & 0.0313 (0.0051) & 0.0348 (0.0058) \\ 
		& \text{MOM-RCSL}  & 0.0344 (0.0053) & 0.0368 (0.0058) & 0.0398 (0.0065) \\ 
		& \text{Ratio}  & 0.7905 & 0.8483 & 0.8750 \\
		\hline
		\multirow{2}{*}{$T$} & Attack & \multicolumn{3}{c}{Omniscient}  \\
		& $\alpha_n$ & $0.05$ & $0.1$ & $0.15$  \\
		\hline
		\multirow{3}{*}{$5$} & \text{RCSL} & 0.0276 (0.0042) & 0.0328 (0.0052) & 0.0396 (0.0061) \\ 
		& \text{MOM-RCSL} & 0.0355 (0.0057) & 0.0395 (0.0061) & 0.045 (0.0069) \\ 
		& \text{Ratio} & 0.7768 & 0.8304 & 0.8811 \\ 
		\hline
		\multirow{3}{*}{$10$} & \text{RCSL} & 0.0276 (0.0042) & 0.0329 (0.0052) & 0.0398 (0.0061) \\ 
		& \text{MOM-RCSL} & 0.0355 (0.0057) & 0.0396 (0.0061) & 0.0451 (0.0069) \\ 
		& \text{Ratio} & 0.7769 & 0.8311 & 0.8820 \\ 
		\hline
		\multirow{2}{*}{$T$} & Attack  & \multicolumn{3}{c}{Bit-flip} \\
		& $\alpha_n$ & $0.05$ & $0.1$ & $0.15$ \\
		\hline
		\multirow{3}{*}{$5$} & \text{RCSL} & 0.0236 (0.0037) & 0.0242 (0.0039) & 0.0250 (0.0041) \\ 
		& \text{MOM-RCSL}  & 0.0325 (0.0051) & 0.0334 (0.0053) & 0.0343 (0.0058) \\ 
		& \text{Ratio}  & 0.7268 & 0.7240 & 0.7281 \\ 
		\hline
		\multirow{3}{*}{$10$} & \text{RCSL} & 0.0236 (0.0037) & 0.0242 (0.0039) & 0.0250 (0.0041) \\ 
		& \text{MOM-RCSL} & 0.0325 (0.0051) & 0.0335 (0.0053) & 0.0344 (0.0058) \\ 
		& \text{Ratio} & 0.7259 & 0.7235 & 0.7260 \\ 
		\hline
	\end{tabular}
	}
\end{table}

In Table \ref{tab:linearfixratio}, we fix the iteration number to be $5$ and $10$.  
Table \ref{tab:linearfixratio} shows similar patterns as Table \ref{tab:linearratio}. There are almost no difference between $T=5$ and $T=10$ iterations. In other words, it shows that only using a very small number of iterations (i.e., $T=5$), our RCSL estimator has already converged. This experiment suggests that the RCSL is communicationally efficient.


\subsubsection{Results for logistic regression model}

\begin{table}[t!]
	\centering{
	\caption{\small The root mean square errors (RMSEs) and their standard errors (in parentheses) of the RCSL and MOM-RCSL, and the ratios of RMSEs between RCSL and MOM-RCSL under sample size $N=1000 \times  (100+1)$, local sample size $n=1000$ and integer $K=10$.  The tolerance parameter for the stopping rule is set to $e_{r}=10^{-4}$. The parameter $\mu_x$ controls the class balance.}\label{tab:logisticratio}
		\bigskip
		\small
	\begin{tabular}{r|rrrr}
		\hline
		$\mu_x$ & \multicolumn{4}{c}{0} \\
		$\alpha_n$ & $0$ & $0.05$ & $0.1$ & $0.15$ \\ 
		\hline
		\text{RCSL} & 0.0504 (0.0075) & 0.0531 (0.0075) & 0.0600 (0.0072) & 0.0701 (0.0075) \\  
		\text{MOM-RCSL} & 0.0699 (0.0109) & 0.0716 (0.0107) & 0.0765 (0.0108) & 0.0830 (0.0112) \\ 
		\text{Ratio} & 0.7215 & 0.7418 & 0.7844 & 0.8452 \\  
		\hline
		$\mu_x$ & \multicolumn{4}{c}{0.5} \\
		$\alpha_n$ & $0$ & $0.05$ & $0.1$ & $0.15$ \\ 
		\hline
		\text{RCSL} & 0.0583 (0.0087) & 0.0601 (0.0087) & 0.0632 (0.0091) & 0.0669 (0.0096) \\
		\text{MOM-RCSL} & 0.0830 (0.0135) & 0.0841 (0.0132) & 0.0868 (0.0134) & 0.0905 (0.0141) \\ 
		\text{Ratio} & 0.7024 & 0.7142 & 0.7281 & 0.7395 \\ 
		\hline
	\end{tabular}
	}
\end{table}

The results for logistic regression model are presented in Table \ref{tab:logisticratio} (for adaptive stopping criterion) and \ref{tab:logisticfixratio} (for fixed number of iterations).  From Table \ref{tab:logisticratio}, RCSL has smaller root mean square errors than MOM-RCSL since all the ratios are greater than 1. With the increase of Byzantine fractions $\alpha_n$, both methods has larger root mean square errors. Moreover, the RCSL leads to a more significant improvement over the MOM-RCSL for imbalanced class case. Similar observations can be made from  Table \ref{tab:logisticfixratio}, where the number of iterations have been pre-determined.  Table \ref{tab:logisticfixratio} also shows that our RCSL is communicationally efficient since $T=5$ iterations have been sufficient for the convergence.

\begin{table}[h]
	\centering{
	\caption{\small The root mean square errors (RMSEs) and their standard errors (in parentheses) of the RCSL and MOM-RCSL, and the ratios of RMSEs between RCSL and MOM-RCSL under sample size $N=1000 \times  (100+1)$, local sample size $n=1000$ and integer $K=10$. The iteration numbers $T$ are fixed to be $5$ and $10$.  The parameter $\mu_x$ controls the class balance. }\label{tab:logisticfixratio}
	\bigskip
	\small
	\begin{tabular}{r|r|rrrr}
		\hline
		\multirow{2}{*}{$T$} & $\mu_x$ & \multicolumn{4}{c}{$0$} \\
		& $\alpha_n$ & $0$ & $0.05$ & $0.1$ & $0.15$ \\ 
		\hline
		\multirow{3}{*}{$5$} & \text{RCSL} & 0.0505 (0.0075) & 0.0531 (0.0076) & 0.0600 (0.0072) & 0.0702 (0.0075) \\ 
	    & \text{MOM-RCSL} & 0.0699 (0.0109) & 0.0716 (0.0107) & 0.0765 (0.0108) & 0.0830 (0.0113) \\ 
		& \text{Ratio} & 0.7112 & 0.7371 & 0.7901 & 0.8646 \\
		\hline
		\multirow{3}{*}{$10$} & \text{RCSL} & 0.0504 (0.0075) & 0.0531 (0.0075) & 0.0600 (0.0072) & 0.0701 (0.0075) \\ 
		& \text{MOM-RCSL} & 0.0699 (0.0109) & 0.0716 (0.0107) & 0.0765 (0.0108) & 0.0830 (0.0112) \\ 
		& \text{Ratio} & 0.7218 & 0.7418 & 0.7845 & 0.8452 \\ 
		\hline
		\multirow{2}{*}{$T$} & $\mu_x$ & \multicolumn{4}{c}{$0.5$} \\
		& $\alpha_n$ & $0$ & $0.05$ & $0.1$ & $0.15$ \\ 
		\hline
		\multirow{3}{*}{$5$} & \text{RCSL} & 0.0583 (0.0087) & 0.0601 (0.0087) & 0.0632 (0.0091) & 0.0670 (0.0096) \\ 
	    & \text{MOM-RCSL} & 0.0829 (0.0135) & 0.0841 (0.0132) & 0.0868 (0.0134) & 0.0906 (0.0140) \\ 
		& \text{Ratio} & 0.7028 & 0.7145 & 0.7281 & 0.7395 \\ 
		\hline
		\multirow{3}{*}{$10$} & \text{RCSL} & 0.0583 (0.0087) & 0.0601 (0.0087) & 0.0632 (0.0091) & 0.0670 (0.0096) \\ 
		& \text{MOM-RCSL} & 0.0830 (0.0135) & 0.0841 (0.0132) & 0.0868 (0.0134) & 0.0905 (0.0141) \\ 
		& \text{Ratio} & 0.7025 & 0.7145 & 0.7281 & 0.7395 \\ 
		\hline
	\end{tabular}
	}
\end{table}

{
\section{Conclusions and Future Work}\label{sec:conclude}
In this paper, we design a Byzantine tolerant algorithm to address a general class of estimation and inference problems in a distributed setting. The first contribution is to improve the statistical efficiency of the widely used Median-of-Means (MOM) estimator by proposing a new  Variance Reduced MOM (VRMOM) estimator. It achieves a nearly optimal convergence rate upto a logarithmic factor and has the same order of computational complexity as the MOM estimator. Inspired by the VRMOM estimator, we further develop the Robust CSL (RCSL) method for general statistical inference problem. The convergence rate improves the previous results in \citet{yin_etal.2018} using either median or trimmed-mean as  the gradient aggregator. Moreover, we establish the asymptotic normality result for our RCSL method.}

{
To highlight our main idea behind VRMOM, we choose to focus on the one-dimensional case and extend to multi-variate case in a coordinate-wise way. Therefore, the convergence rate in Theorem \ref{thm:iter_rdane_byzantine_conv} has an extra $\sqrt{p}$ in the term related to Byzantine failures (i.e., $\alpha_n \sqrt{p}/\sqrt{n}$). How to remove this extra $\sqrt{p}$ in a computationally efficient way while still achieving the same statistical efficiency is still an open problem. As an important future direction, we will address this challenge by developing new multivariate efficient aggregators based on the VRMOM estimator. On the other hand, by adding additional regularizers as in \eqref{eq:regularized_rcsl}, we can address a wider class of problems in Byzantine setup. We leave these problems in the future study.
}

\section*{Acknowledgement}
Weidong Liu and Xiaojun Mao are the co-corresponding authors. Weidong Liu's research is supported by National Program on Key Basic Research Project (973 Program, 2018AAA0100704), NSFC Grant No. 11825104 and 11690013, Youth Talent Support Program, and a grant from Australian Research Council. Xiaojun Mao's research is supported by NSFC Grant No. 12001109, Shanghai Sailing Program 19YF1402800, Major Research Plan of NSFC Grant No. 92046021, and the Science and Technology Commission of Shanghai Municipality grant 20dz1200600. {The authors would like to thank the action editor and two anonymous referees for their constructive comments, which greatly improves the quality of the paper}. 

\newpage

\appendix

\section*{Appendix}

{The appendix consists of four parts. In Appendix \ref{sec:proof}, we provide detailed proof for the theoretical results of VRMOM estimator presented in Section \ref{sec:theory:VRMOM}. In Appendix \ref{sec:posit_def}, we prove the positive definiteness of $\vect{\mathcal{C}}_{\mathrm{MOM}}-\vect{\mathcal{C}}$ in two-dimensional case, which have been mentioned in Remark \ref{rem:positive_def} of the main paper. In Appendix \ref{sec:rcsl_theory}, we present the theoretical results and some technical assumptions for the RCSL method. Lastly, in Appendix \ref{sec:examples_verification}, we will show that a large class of generalized linear models and $M$-estimators suffice the proposed assumptions in Appendix \ref{sec:tech_assump}.}


\section{Proof of Theories for VRMOM Estimator}
\label{sec:proof}

{In this appendix, we mainly prove the theoretical results of the proposed VRMOM estimator in Section \ref{sec:theory:VRMOM}. In Appendix \ref{sec:vrmom_lemma}, we introduce several lemmas which are useful for proofing the main theorems. Next, we present the main proofs in Appendix \ref{sec:proof_VRMOM}.}

\subsection{Technical Lemmas}\label{sec:vrmom_lemma}


\begin{lemma}	\label{lemma:berry_esseen}
	(Berry-Esseen Theorem, Theorem 9.1.3 in \cite{chow_teicher.2012}) If $\{X_i, \; i\geq1\}$ are i.i.d. mean-zero random variables with $\mbE|X_1|^2=\sigma^2,\mbE|X_1|^{2+\kappa}<\infty$, where $\kappa\in(0,1]$
	. Then there exists a constant $C_{\kappa}>0$ such that
	\begin{equation*}
		\sup_{-\infty<x<\infty}\left|\mbP\left\{\sum_{i=1}^nX_i<x\sigma n^{1/2}\right\}-\Phi(x)\right|\leq C_{\kappa}\frac{\mbE|X_1|^{2+\kappa}}{\sigma^{2+\kappa}n^{\kappa/2}}.
	\end{equation*}
\end{lemma}


\begin{lemma}	\label{lemma:cai_liu}
	(Exponential Inequality, Lemma 1 in \cite{cai_liu.2011}) Let $X_1,...,X_n$ be i.i.d. random variables with zero mean. Suppose that there exist some $\eta>0$ and $C>0$ such that $\mbE (X_1^2e^{\eta|X_1|})\leq C$. Then uniformly for $0<x\leq C$ and $n\geq 1$, there is
	\begin{equation*}
		\mbP\left\{\frac{1}{n}\sum_{i=1}^nX_i\geq (\eta+\eta^{-1})x\right\}\leq\exp\left(-\frac{nx^2}{C}\right).
	\end{equation*}
\end{lemma}

This lemma will be the workhorse throughout all proofs in this article. For ease of notations, we will use
\begin{equation}	\label{cai_liu.symbol}
	\mfE(\eta,X):=\mbE\left(X^2e^{\eta|X|}\right),
\end{equation}
when the expression of $X$ is too complicated.

\begin{lemma}	\label{lemma:empirical}
	Let $X_1,...,X_m$ be i.i.d. random variables with cumulative distribution function $G(x)=\mbP(X_1\leq x)$. And there exists some constants $C>0$ and $\kappa\in(0,1]$ such that there is
	\begin{equation}	\label{empirical_condition.ineq}
		|G(x_1)-G(x_2)|\leq C|x_1-x_2|+\frac{C}{n^{\kappa/2}},
	\end{equation}
	holds for any $x_1,x_2\in\mbR$. Further define the function
	\begin{align*}
		Z_i(x)=\mbI\{X_i\leq x\}- G(x).
	\end{align*}
	Let $\delta_n=O(1)$ be some rate, then there exists $\tilde{C}$ large enough such that
	\begin{equation*}
		\mbP\left[\frac{1}{m}\sup_{|x|\leq\delta_n}\left|\sum_{i=1}^m\left\{Z_i(x)-Z_i(0)\right\}\right|\geq \tilde{C}\left(\sqrt{\frac{\delta_n\log n}{m}}+\frac{1}{n^{\kappa/2}}\right)\right]= O(n^{-\gamma}),
	\end{equation*}
	provided the rates constraints $\log n=O(m)$ and $\max\{m^{-1}\log n,n^{-\kappa/2}\}=O(\delta_n)$.
\end{lemma}

\begin{proof}
	Evenly divide the interval $[-\delta_n,\delta_n]$ into $2n$ pieces and denote the set $\mfN=\{-\delta_n,-\frac{n-1}{n}\delta_n,...,\delta_n\}$. Then
	\begin{align*}	
		\frac{1}{m}\sup_{|x|\leq\delta_n}\left|\sum_{i=1}^m\{Z_i(x)-Z_i(0)\}\right|\leq&\max_{\tilde{x}\in\mfN}\left|\frac{1}{m}\sum_{i=1}^m\{Z_i(\tilde{x})-Z_i(0)\}\right|\stepcounter{equation}\tag{\theequation}\label{eq:empirical_lemma_term1}\\
		&+\max_{\tilde{x}\in\mfN}\sup_{\{x:|\tilde{x}-x|\leq\delta_n/n\}}\left|\frac{1}{m}\sum_{i=1}^m\{Z_i(\tilde{x})-Z_i(x)\}\right|.
	\end{align*}
	For the second term, notice that
	\begin{align*}
		\sup_{\{x:|\tilde{x}-x|\leq\delta_n/n\}}\left|\frac{1}{m}\sum_{i=1}^m\{Z_i(\tilde{x})-Z_i(x)\}\right|\leq \frac{1}{m}\sum_{i=1}^m\mbI\left\{|X_i-\tilde{x}|\leq \frac{\delta_n}{n}\right\}+C\left(\frac{2\delta_n}{n}+\frac{1}{n^{\kappa/2}}\right).
	\end{align*}
	For every fixed $\tilde{x}\in\mfN$, we know
	\begin{align*}
		&\mbE\left\{\mbI\left(|X_i-\tilde{x}|\leq \frac{\delta_n}{n}\right)\right\}=\mbP\left(|X_i-\tilde{x}|\leq \frac{\delta_n}{n}\right)\leq C\left(\frac{2\delta_n}{n}+\frac{1}{n^{\kappa/2}}\right),\\
		&\mathfrak{E}\left\{1,\mbI\left(|X_i-\tilde{x}|\leq \frac{\delta_n}{n}\right)-\mbP\left(|X_i-\tilde{x}|\leq \frac{\delta_n}{n}\right)\right\}\\
			\leq&e\mbP\left(|X_i-\tilde{x}|\leq \frac{\delta_n}{n}\right)\left\{\mbP\left(|X_i-\tilde{x}|\leq \frac{\delta_n}{n}\right)+1\right\}=O\left(\frac{\log n}{m}+\frac{1}{n^{\kappa/2}}\right),
	\end{align*}
	since $\delta_n/n=O(n^{-\kappa/2})$ from the rates constraints. Then applying Lemma \ref{lemma:cai_liu} we have
	\begin{equation}	\label{ineq:empirical_term1}
		\begin{aligned}
			&\mbP\left[\max_{\tilde{x}\in\mfN}\sup_{\{x:|\tilde{x}-x|\leq\delta_n/n\}}\frac{1}{m}\left|\sum_{i=1}^m\left\{Z_i(\tilde{x})-Z_i(x)\right\}\right|\geq C_1\left(\frac{1}{n^{\kappa/2}}+\frac{\log n}{m}\right)+C\left(\frac{2\delta_n}{n}+\frac{1}{n^{\kappa/2}}\right)\right]\\
				\leq&2n\max_{\tilde{x}\in\mfN}\mbP\left\{\frac{1}{m}\sum_{i=1}^m\mbI\left(|X_i-\tilde{x}|\leq \frac{\delta_n}{n}\right)-\mbP\left(|X_1-\tilde{x}|\leq \frac{\delta_n}{n}\right)\geq C_1\left(\frac{\log n}{m}+\frac{1}{n^{\kappa/2}}\right)\right\}\\
				=&O(n^{-\gamma}),
		\end{aligned}
	\end{equation}
	by letting $C_1$ large enough. Now for the first term in \eqref{eq:empirical_lemma_term1}, again we apply Lemma \ref{lemma:cai_liu} with
	\begin{align*}
		&\mbE[Z_i(x)-Z_i(0)]=0,\\
		\sup_{-\delta_n\leq x\leq\delta_n}&\mbE\left[\{Z_i(x)-Z_i(0)\}^2e^{|Z_i(x)-Z_i(0)|}\right]=O(\delta_n),
	\end{align*}
	there is
	\begin{equation}	\label{empirical_term2.ineq}
		\begin{aligned}
			&\mbP\left[\max_{\tilde{x}\in\mfN}\left|\frac{1}{m}\sum_{i=1}^m\{Z_i(\tilde{x})-Z_i(0)\}\right|\geq C_2\sqrt{\frac{\delta_n\log n}{m}}\right]\\
				\leq&2n\max_{\tilde{x}\in\mfN}\mbP\left[\frac{1}{m}\left|\sum_{i=1}^m\{Z_i(\tilde{x})-Z_i(0)\}\right|\geq C_2\sqrt{\frac{\delta_n\log n}{m}}\right]= O(n^{-\gamma}),
		\end{aligned}
	\end{equation}
	for some $C_2$ large enough. Then the lemma is proved by combining (\ref{ineq:empirical_term1}) and (\ref{empirical_term2.ineq}).
\end{proof}


\begin{lemma}	\label{lemma:concen_mom_byzantine}
	(Concentration of median-of-means with Byzantine machines) Let $N$($=(m+1)n$) i.i.d. random variables $X_1,...,X_{N}$ evenly distributed in $m+1$ subsets $\mcH_0,...,\mcH_m$. There is a subset of index $\mcB\subseteq\{1,\dots,m\}$ with $\card(\mcB)=\lfloor\alpha_nm\rfloor$. Define
	\begin{equation*}
		\bar{X}_j=
		\begin{cases}
			 \frac{1}{n}\sum_{i\in\mcH_j}X_i\quad&j\notin\mcB,\\
			 \quad *\quad&j\in\mcB,
		\end{cases}\qquad
		\hat{X}=\text{med}\left(\bar{X}_j\mid 0\leq j\leq m\right).
	\end{equation*}
	Suppose $\mbE(X_1)=0,\var(X_1)=\sigma^2,\mbE|X_1|^{2+\kappa}<\infty$, where $\kappa\in(0,1]$. The fraction $\alpha_n$ satisfies $\alpha_n\leq1/2-\delta$
	for some $\delta>0$. Then for every $\gamma>1$, there exists $\tilde{C}>0$ large enough, such that
	\begin{equation*}
		\mbP\left\{\left|\hat{X}\right|\geq \frac{\tilde{C}}{\sqrt{n}}
		\left(\alpha_n+\frac{1}{n^{\kappa/2}}+\sqrt{\frac{\log n}{m}}\right)\right\}=O(n^{-\gamma}).
	\end{equation*}
\end{lemma}

\begin{proof}
	By definition of medians, for any $x>0$, there is
	\begin{align*}
		&\mbP\left(\hat{X}\geq x\right)=\mbP\left\{\sum_{j=0}^{m}\mbI\left(\bar{X}_j<x\right)\leq\frac{m+1}{2}\right\}\leq\mbP\left\{\sum_{j\notin\mcB}\mbI\left(\bar{X}_j<x\right)\leq\frac{m+1}{2}\right\}\\
			=&\mbP\left\{\frac{1}{(1-\alpha_n)(m+1)}\sum_{j\notin\mcB}\mbI\left(\bar{X}_j<x\right)-\mbP\left(\bar{X}_1<x\right)\leq\frac{1}{2}+\frac{\alpha_n}{2(1-\alpha_n)}-\mbP\left(\bar{X}_1<x\right)\right\}\\
			\leq&\mbP\left\{\frac{1}{(1-\alpha_n)(m+1)}\sum_{j\notin\mcB}\mbI\left(\bar{X}_j<x\right)-\mbP\left(\bar{X}_1<x\right)\leq\frac{1}{2}+\frac{\alpha_n}{2(1-\alpha_n)}+C_{\kappa}\frac{\mbE|X_1|^{2+\kappa}}{\sigma^{2+\kappa}n^{\kappa/2}}-\Phi\left(\frac{\sqrt{n}x}{\sigma}\right)\right\}\\
			=&\mbP\left\{\frac{1}{(1-\alpha_n)(m+1)}\sum_{j\notin\mcB}\mbI\left(\bar{X}_j<x\right)-\mbP\left(\bar{X}_1<x\right)\leq-\sqrt{\frac{c\log n}{(1-\alpha_n)(m+1)}}\right\},
	\end{align*}
	where the last line uses Berry-Esseen Theorem (Lemma \ref{lemma:berry_esseen}), and $x$ is given by 
	\begin{equation*}
		x=\frac{\sigma}{\sqrt{n}}\Phi^{-1}\left(\frac{1}{2}+\frac{\alpha_n}{2(1-\alpha_n)}+C_{\kappa}\frac{\mbE|X_1|^{2+\kappa}}{\sigma^{2+\kappa}n^{\kappa/2}}+\sqrt{\frac{c\log n}{(1-\alpha_n)(m+1)}}\right).
	\end{equation*}
	Now apply Lemma \ref{lemma:cai_liu} for the i.i.d. sequence $\mbI\left(\bar{X}_j<x\right)-\mbP\left(\bar{X}_1<x\right)$ with
	\begin{equation*}
		\mathfrak{E}\left\{1,\mbI\left(\bar{X}_j<x\right)-\mbP\left(\bar{X}_1<x\right)\right\}\leq e,
	\end{equation*}
	we have
	\begin{equation*}
		\mbP\left\{\frac{1}{(1-\alpha_n)(m+1)}\sum_{j\notin\mcB}\mbI\left(\bar{X}_j<x\right)-\mbP\left(\bar{X}_1<x\right)\leq-\sqrt{\frac{c\log n}{(1-\alpha_n)(m+1)}}\right\}=O( n^{-\gamma}),
	\end{equation*}
	with $c= 4e\gamma$. Moreover, we have the following elementary facts
	\begin{equation*}
		\Phi^{-1}(x_0)=\Phi^{-1}(x_0)-\Phi^{-1}\left(1/2\right)\leq\left(x_0-1/2\right)\left(\Phi^{-1}\right)'(x_0)\leq\frac{x_0-1/2}{\psi\left\{\Phi^{-1}(1-\delta/2)\right\}},
	\end{equation*}
	holds for any $1/2\leq x_0<1-\delta/2$. Also we know $1/2\leq\Phi(\sqrt{n}x/\sigma)<1-\delta/2$ holds for $m,n$ sufficiently large. Denote $C_{\delta}=1/\psi\{\Phi^{-1}(1-\delta/2)\}$, then there is
	\begin{align*}
		&\mbP\left\{\hat{X}\geq \frac{\sigma C_{\delta}}{\sqrt{n}}\left(\frac{\alpha_n}{2(1-\alpha_n)}+C_{\kappa}\frac{\mbE|X_1|^{2+\kappa}}{\sigma^{2+\kappa}n^{\kappa/2}}+\sqrt{\frac{c\log n}{(1-\alpha_n)(m+1)}}\right)\right\}\leq\mbP\left(\hat{X}\geq x\right)\\
			\leq&\mbP\left\{\frac{1}{(1-\alpha_n)(m+1)}\sum_{j\notin\mcB}\mbI\left(\bar{X}_j<x\right)-\mbP\left(\bar{X}_1<x\right)\leq-\sqrt{\frac{c\log n}{(1-\alpha_n)(m+1)}}\right\}\leq n^{-\gamma}.
	\end{align*}
	Now do the same thing for $\mbP(\hat{X}\leq-x)$, we finally get the desired result.
\end{proof}


\subsection{Proofs of the results in Section \ref{sec:theory:VRMOM}}
\label{sec:proof_VRMOM}

We firstly provide the proofs related to the univariate VRMOM estimator.


\begin{proof}[Proof of Theorem \ref{thm:normality_vrmom}, Theorem \ref{thm:concentration_vrmom_byzantine}, and Theorem \ref{thm:concen_multi_vrmom}]
	Denote $\kappa_1=\min\{1,\kappa\}$ and $\kappa_2=\min\{2,\kappa\}$. To prove this result, let's firstly give a convergence rate for sample variance $\hat{\sigma}-\sigma$. From the moment bound $\mbE|X_1-\mu|^{2+\kappa}<\infty$, 
	 by Marcinkiewicz-Zygmund theorem (Theorem 5.2.2 in \cite{chow_teicher.2012}), we have
	\begin{align*}
		\frac{1}{n}\sum_{i\in\mcH_0}(X_i-\mu)^2-\sigma^2=o_{\mbP}\left(\frac{1}{n^{\kappa_2/(2+\kappa_2)}}\right),\quad\frac{1}{n}\sum_{i\in\mcH_0}(X_i-\mu)=o_{\mbP}\left(\frac{1}{n^{\kappa_2/(2+\kappa_2)}}\right).
	\end{align*} 
	Therefore we have
	\begin{equation}	\label{eq:th1_vrmom_sigma}
		\begin{aligned}
			\left|\hat{\sigma}-\sigma\right|=&\frac{1}{\hat{\sigma}+\sigma}\left|\frac{1}{n}\sum_{i\in\mcH_0}\left(X_i-\bar{X}_0\right)^2-\sigma^2\right|\\
			\leq&\frac{1}{\sigma}\left|\frac{1}{n}\sum_{i\in\mcH_0}(X_i-\mu)^2-\sigma^2-\left(\bar{X}_0-\mu\right)^2\right|=o_{\mbP}\left(\frac{1}{n^{\kappa_2/(2+\kappa_2)}}\right).
		\end{aligned} 
	\end{equation}
	Next pick out one summand in \eqref{eq:vrmom_def}. Denote 
	\begin{align*}
		G_j(x):=\mbP\left\{\frac{\sqrt{n}(\bar{X}_j-\mu)}{\sigma}\leq x\right\},\quad I_j(x):=\mbI\left\{\frac{\sqrt{n}(\bar{X}_j-\mu)}{\sigma}\leq x\right\},
	\end{align*}
	then 
	\begin{align*}
		&\mbI\left\{\bar{X}_j\leq\hat{\mu}+\frac{\hat{\sigma}\Delta_k}{\sqrt{n}}\right\}-\frac{k}{K+1}\\
			=&I_j\left\{\frac{\sqrt{n}(\hat{\mu}-\mu)}{\sigma}+\frac{\hat{\sigma}}{\sigma}\Delta_k\right\}-G_j\left\{\frac{\sqrt{n}(\hat{\mu}-\mu)}{\sigma}+\frac{\hat{\sigma}}{\sigma}\Delta_k\right\}+\underbrace{G_j\left\{\frac{\sqrt{n}(\hat{\mu}-\mu)}{\sigma}+\frac{\hat{\sigma}}{\sigma}\Delta_k\right\}-\Phi(\Delta_k)}_{T}.
	\end{align*}
	For the term $T$,
	\begin{equation}	\label{mmm.term1}
	\begin{aligned}
		T=&G_j\left\{ \frac{\sqrt{n} (\hat{\mu}-\mu)}{\sigma}+\frac{\hat{\sigma}}{\sigma} \Delta_k \right\}-\Phi \left\{ \frac{\sqrt{n} (\hat{\mu}-\mu)}{\sigma}+\frac{\hat{\sigma}}{\sigma}\Delta_k \right\}+\Phi \left\{ \frac{\sqrt{n}(\hat{\mu}-\mu)}{\sigma}+\frac{\hat{\sigma}}{\sigma} \Delta_k \right\}-\Phi(\Delta_k)\\
			=&O(n^{-\kappa_1/2})+\psi(\Delta_k) \frac{\sqrt{n} (\hat{\mu}-\mu)}{\sigma}+\psi(\Delta_k)\frac{\hat{\sigma}-\sigma}{\sigma}\Delta_k +O\left(\frac{\sqrt{n}(\hat{\mu}-\mu)}{\sigma}+\frac{\hat{\sigma} -\sigma}{\sigma} \Delta_k \right)^2\\
			=&\psi(\Delta_k)\frac{\sqrt{n}(\hat{\mu}-\mu)}{\sigma}+O_{\mbP}\left(\alpha_n^2+\frac{\log n}{m}+\frac{1}{n^{\kappa_2/(2+\kappa_2)}}\right),
	\end{aligned}
	\end{equation}
	where line 1 to line 2 uses Berry-Esseen Inequality (Lemma \ref{lemma:berry_esseen}) and Taylor expansion, line 2 to line 3 uses \eqref{eq:th1_vrmom_sigma} and concentration inequalities for 
	$\hat{\mu}$ (Lemma \ref{lemma:concen_mom_byzantine}).
	By using Lemma \ref{lemma:concen_mom_byzantine} and \eqref{eq:th1_vrmom_sigma} again we have that
	\begin{equation*}
		\frac{\sqrt{n}(\hat{\mu}-\mu)}{\sigma}+\frac{\hat{\sigma}-\sigma}{\sigma}\Delta_k=O_{\mbP}\left(\alpha_n+\sqrt{\frac{\log n}{m}}+\frac{1}{n^{\kappa_2/(2+\kappa_2)}}\right).
	\end{equation*} 
	Also note that Berry-Esseen Inequality guarantees that $\sqrt{n}(Y_j-\mu)/\sigma$ satisfies
	\begin{equation*}
		|G_j(x_1)-G_j(x_2)|\leq \psi(0)|x_1-x_2|+2C_{\kappa_1}\frac{\mbE|X_1-\mu|^{2+\kappa_1}}{\sigma^{2+\kappa_1}n^{\kappa_1/2}}.
	\end{equation*} 
	So we can apply Lemma \ref{lemma:empirical} with $\delta_n=O(\alpha_n+\sqrt{\frac{\log n}{m}}+\frac{1}{n^{\kappa_2/(2+\kappa_2)}})$, which yields
	\begin{align*}
		&\frac{1}{m+1}\left|\sum_{j\notin\mcB}\left[I_j\left\{\frac{\sqrt{n}(\hat{\mu}-\mu)}{\sigma}+\frac{\hat{\sigma}}{\sigma}\Delta_k\right\}-G_j\left\{\frac{\sqrt{n}(\hat{\mu}-\mu)}{\sigma}+\frac{\hat{\sigma}}{\sigma}\Delta_k\right\}-I_j(\Delta_k)+G_j(\Delta_k)\right]\right|\\
		\leq&\frac{1}{m+1}\sup_{|x-\Delta_k|\leq\delta_n}\left|\sum_{j\notin\mcB}\left\{I_j(x+\Delta_k)-G_j(x+\Delta_k)-I_j(\Delta_k)+G_j(\Delta_k)\right\}\right|\\
			=&O_{\mbP}\left(\sqrt{\frac{\delta_n\log n}{m}}+\frac{1}{n^{\kappa_1/2}}\right)\\
			=&O_{\mbP}\left(\sqrt{\frac{\alpha_n\log n}{m}}+\Big(\frac{\log n}{m}\Big)^{3/4}+\frac{\log^{1/2}n}{m^{1/2}n^{\kappa_2/(4+2\kappa_2)}}+\frac{1}{n^{\kappa_1/2}}\right).
	\end{align*}
	Thus it implies that
	\begin{equation}	\label{mmm.term2}
		\begin{aligned}
			&\frac{1}{m+1}\sum_{j=0}^{m}\left[I_j\left\{\frac{\sqrt{n}(\hat{\mu}-\mu)}{\sigma}+\frac{\hat{\sigma}}{\sigma}\Delta_k\right\}-G_j\left\{\frac{\sqrt{n}(\hat{\mu}-\mu)}{\sigma}+\frac{\hat{\sigma}}{\sigma}\Delta_k\right\}\right]\\
				=&\frac{1}{m+1}\sum_{j=0}^{m}\left\{I_j(\Delta_k)-G_j(\Delta_k)\right\}\\
				&+O_{\mbP}\left(\sqrt{\frac{\alpha_n\log n}{m}}+\Big(\frac{\log n}{m}\Big)^{3/4}+\frac{\log^{1/2}n}{m^{1/2}n^{\kappa_2/(4+2\kappa_2)}}+\frac{1}{n^{\kappa_1/2}}\right).
		\end{aligned}
	\end{equation}
	Again from \eqref{eq:th1_vrmom_sigma} we have 
	\begin{equation}	\label{mmm.term3}
		\begin{aligned}
			&\left|\frac{\hat{\sigma}-\sigma}{(m+1)\sqrt{n}\sum_{k=1}^K\psi(\Delta_k)}\sum_{k=1}^K\sum_{j=0}^{m}\left[\mbI\left(\bar{X}_j\leq\hat{\mu}+\frac{\hat{\sigma}\Delta_k}{\sqrt{n}}\right)-\frac{k}{K+1}\right]\right|\\
				\leq&\frac{K|\hat{\sigma}-\sigma|}{\sqrt{n}\sum_{k=1}^K\psi(\Delta_k)}=O_{\mbP}\left(\frac{1}{n^{(3\kappa_2+2)/(2\kappa_2+4)}}\right).
		\end{aligned}
	\end{equation}
	Combining equations (\ref{mmm.term1}) (\ref{mmm.term2}) (\ref{mmm.term3}), we have
	\begin{equation}	\label{eq:simplified_diff}
	\begin{aligned}
		\bar{\mu}-\mu=&\hat{\mu}-\mu-\frac{\hat{\sigma}}{(m+1)\sqrt{n}\sum_{k=1}^K\psi(\Delta_k)}\sum_{k=1}^K\sum_{j=0}^{m}\left\{\mbI\left(\bar{X}_j\leq\hat{\mu}+\frac{\hat{\sigma}\Delta_k}{\sqrt{n}}\right)-\frac{k}{K+1}\right\}\\
			=&\hat{\mu}-\mu-\frac{\sigma}{(m+1)\sqrt{n}\sum_{k=1}^K\psi(\Delta_k)}\sum_{k=1}^K\sum_{j\notin\mcB}\left\{\mbI\left(\bar{X}_j\leq\hat{\mu}+\frac{\hat{\sigma}\Delta_k}{\sqrt{n}}\right)-\frac{k}{K+1}\right\}\\
			&+O_{\mbP}\left(\frac{\alpha_n}{\sqrt{n}}+\frac{1}{n^{(3\kappa_2+2)/(2\kappa_2+4)}}\right)\\
			=&\frac{\sigma}{\sqrt{n}\sum_{k=1}^K\psi(\Delta_k)}\sum_{k=1}^K\frac{1}{m+1}\sum_{j\notin\mcB}\left[G_j\left\{\frac{\sqrt{n}(\hat{\mu}-\mu)}{\sigma}+\frac{\hat{\sigma}}{\sigma}\Delta_k\right\}-I_j\left\{\frac{\sqrt{n}(\hat{\mu}-\mu)}{\sigma}+\frac{\hat{\sigma}}{\sigma}\Delta_k\right\}\right]\\
			&+O_{\mbP}\left(\frac{\alpha_n}{\sqrt{n}}+\frac{1}{n^{(3\kappa_2+2)/(2\kappa_2+4)}}+\frac{\log n}{m\sqrt{n}}\right)\\
			=&\frac{1}{m+1}\sum_{j\notin\mcB}\frac{\sigma}{\sqrt{n}\sum_{k=1}^K\psi(\Delta_k)}\sum_{k=1}^K\{G_j(\Delta_k)-I_j(\Delta_k)\}\\
			&+O_{\mbP}\left(\frac{\alpha_n}{\sqrt{n}}+\frac{\log n}{m\sqrt{n}}+\frac{1}{n^{(3\kappa_2+2)/(2\kappa_2+4)}}+\frac{\log^{3/4} n}{n^{1/2}m^{3/4}}\right).
	\end{aligned}
	\end{equation}
	From central limit theorem, we have
	\begin{equation}	\label{eq:mu_bar_concen}
		\bar{\mu}-\mu=O_{\mbP}\left(\frac{\alpha_n}{\sqrt{n}}+\frac{1}{\sqrt{mn}}+\frac{1}{n^{(3\kappa_2+2)/(2\kappa_2+4)}}+\frac{\log^{3/4} n}{n^{1/2}m^{3/4}}\right),
	\end{equation}
	which proves Theorem \ref{thm:concentration_vrmom_byzantine}. Moreover, when $m=o(n^{2\kappa_2/(\kappa_2+2)})=o(\min\{n^{2\kappa/(\kappa+2)},n\}), \log^3n=o(m)$ and $\alpha_n=o(1/\sqrt{m})$, the remainder in \eqref{eq:simplified_diff} will be of the order $o_{\mbP}(\frac{1}{\sqrt{mn}})$, while the major term becomes a summation of i.i.d. sequence with
	\begin{align*}
		&\mbE\left[\sum_{k=1}^K\left\{G_j(\Delta_k)-I_j(\Delta_k)\right\}\right]=0,\:\text{ and}\\
		&\text{Var}\left[\sum_{k=1}^K\left\{G_j(\Delta_k)-I_j(\Delta_k)\right\}\right]\\
			=&\sum_{k_1,k_2=1}^KG_{j}\left\{\min(\Delta_{k_1},\Delta_{k_2})\right\}-G_j(\Delta_{k_1})G_j(\Delta_{k_2})\\
			=& \sum_{k_1,k_2=1}^K\min(\tau_{k_1},\tau_{k_2})\{1-\max(\tau_{k_1},\tau_{k_2})\}+O(n^{-1/2}),
	\end{align*}
	where the last line again follows from Berry-Esseen theorem (\autoref{lemma:berry_esseen}). Applying standard central limit theorem we have
	\begin{equation*}
		\sqrt{N}(\,\bar{\mu}-\mu)\xrightarrow{d}\mathcal{N}(0,\sigma_K^2),
	\end{equation*}
	where $\sigma^2_K$, as defined in \eqref{eq:vrmom_variance}, tends to $\pi\sigma^2/3$ according to Lemma \ref{lem:cl_lim} in Appendix \ref{sec:posit_def}. Therefore Theorem \ref{thm:normality_vrmom} is proved.
	
	{To prove Theorem \ref{thm:concen_multi_vrmom}, we just apply \eqref{eq:mu_bar_concen} to each coordinate and obtain that
	\begin{equation*}
		|\bar{\vect{\mu}}-\vect{\mu}|_2=\sqrt{\sum_{l=1}^p|\bar{\mu}_l-\mu_l|^2}=O_{\mbP}\left(\frac{\alpha_n\sqrt{p}}{\sqrt{n}}+\sqrt{\frac{p}{mn}}+\frac{\sqrt{p}}{n^{(3\kappa_2+2)/(2\kappa_2+4)}}+\frac{p^{1/2}\log^{3/4} n}{n^{1/2}m^{3/4}}\right).
	\end{equation*}}
\end{proof}

\begin{proof}[Proof of Theorem \ref{cor:VRMOM_limit}]
	{
	To prove asymptotic normality of multi-dimensional VRMOM estimator, using \eqref{eq:simplified_diff} and the rate constraints, for each coordinate $l$ (where $1\leq l\leq p$), there is
	\begin{equation*}
		\bar{\mu}_l-\mu_l=\frac{1}{m+1}\sum_{j\notin\mcB}\frac{\sqrt{\sigma_{l,l}}}{\sqrt{n}\sum_{k=1}^K\psi(\Delta_k)}\sum_{k=1}^K\{G_{j,l}(\Delta_k)-I_{j,l}(\Delta_k)\}+o_{\mbP}\left(\frac{1}{\sqrt{pmn}}\right),
	\end{equation*}
	where
	\begin{align*}
		G_{j,l}(x):=\mbP\left\{\frac{\sqrt{n}(\bar{X}_{j,l}-\mu_l)}{\sqrt{\sigma_{l,l}}}\leq x\right\},\quad I_{j,l}(x):=\mbI\left\{\frac{\sqrt{n}(\bar{X}_{j,l}-\mu_l)}{\sqrt{\sigma_{l,l}}}\leq x\right\}.
	\end{align*}
	For any vector $|\vect{v}|_2=1$, there is
	\begin{align*}
		\left\langle\bar{\vect{\mu}}-\vect{\mu},\vect{v}\right\rangle= \frac{1}{m+1}\sum_{j\notin\mcB}\frac{1}{\sqrt{n}\sum_{k=1}^K\psi(\Delta_k)}\sum_{k=1}^K\sum_{l=1}^p \sqrt{\sigma_{l,l}}\big\{ I_{j,l}(\Delta_k)- G_{j,l}(\Delta_k)\big\}v_{l} +o_{\mbP}\left( \frac{1}{\sqrt{mn}}\right).
	\end{align*}
	Now we apply central limit theorem and yield
	\begin{align*}
		&\frac{\sqrt{(m+1)n}}{\tilde{\sigma}_{\vect{v}}}\big\langle\bar{\vect{\mu}}-\vect{\mu},\vect{v}\big\rangle\xrightarrow{d}\mcN(0,1),\\
		&\text{where }\quad \tilde{\sigma}_{\vect{v}}^2=\vect{v}^{\rm T}\tilde{\vect{\mathcal{C}}}\vect{v}.
	\end{align*}
	Here $\tilde{\vect{\mathcal{C}}}\in\mbR^{p\times p}$ has its $(l_1,l_2)$-entry defined as
	\begin{align*}
		\tilde{\mathcal{C}}_{l_1,l_2}=&\frac{\sqrt{\sigma_{l_1,l_1}\sigma_{l_2,l_2}}}{\{\sum_{k=1}^K\psi(\Delta_k)\}^2}\mbE\left[\sum_{k=1}^K\{I_{0,l_1}(\Delta_k)-G_{0,l_1}(\Delta_k)\}\sum_{k=1}^K\{I_{0,l_2}(\Delta_k)-G_{0,l_2}(\Delta_k)\}\right]\\
			=&\frac{\sqrt{\sigma_{l_1,l_1}\sigma_{l_2,l_2}}}{\{\sum_{k=1}^K\psi(\Delta_k)\}^2}\sum_{k_1,k_2}\left\{\mbP\left(\frac{\sqrt{n}(\bar{X}_{j,l_1}-\mu_{l_1})}{\sqrt{\sigma_{l_1,l_1}}}\leq\Delta_{k_1}, \frac{\sqrt{n}(\bar{X}_{j,l_2}-\mu_{l_2})}{\sqrt{\sigma_{l_2,l_2}}}\leq\Delta_{k_2}\right)-G_{0,l_1}(\Delta_{k_1})G_{0,l_2}(\Delta_{k_2})\right\},
	\end{align*}
	Moreover, we can apply multivariate Berry-Esseen theorem (See Theorem 1.3 in \cite{gotze.1991}) and give
	\begin{align*}
		\tilde{\mathcal{C}}_{l_1,l_2}&=\mathcal{C}_{l_1,l_2}+O(n^{-1/2}),
	\end{align*}
	where $\mathcal{C}_{l_1,l_2}$ is defined in \eqref{eq:normal_cov_entry}. Thus the theorem is proved.}
\end{proof}

\begin{proof}[Proof of Proposition \ref{cor:MOM_limit}]
	{
	For the asymptotic normality of multi-dimensional MOM estimator, together with Lemma \ref{lem:mom_bahadur} below and the rate constraint, we can show that
	\begin{equation*}
		\hat{\mu}_{\mathrm{MOM},l}-\mu_l=\frac{\sqrt{2\pi\sigma_{l,l}}}{(m+1)\sqrt{n}}\sum_{j\notin\mcB}\left\{G_{j,l}(0)-I_{j,l}(0)\right\}+o_{\mbP}\left(\frac{1}{\sqrt{pmn}}\right).
	\end{equation*}
	For any vector $|\vect{v}|_2=1$, there is
	\begin{align*}
		\left\langle\hat{\vect{\mu}}_{\mathrm{MOM}}-\vect{\mu},\vect{v}\right\rangle= \frac{\sqrt{2\pi}}{(m+1)\sqrt{n}}\sum_{j\notin\mcB}\sum_{l=1}^p \sqrt{\sigma_{l,l}}\big\{ I_{j,l}(\Delta_k)- G_{j,l}(\Delta_k)\big\}v_{l} +o_{\mbP}\left( \frac{1}{\sqrt{mn}}\right).
	\end{align*}
	Now we apply central limit theorem and yield
	\begin{align*}
		&\frac{\sqrt{(m+1)n}}{\tilde{\sigma}_{\vect{v}}}\big\langle\hat{\vect{\mu}}_{\mathrm{MOM}}-\vect{\mu},\vect{v}\big\rangle\xrightarrow{d}\mcN(0,1),\\
		&\text{where }\quad \tilde{\sigma}_{\vect{v}}^2=\vect{v}^{\rm T}\tilde{\vect{\mathcal{C}}}_{\mathrm{MOM}}\vect{v}.
	\end{align*}
	Here $\tilde{\vect{\mathcal{C}}}_{\mathrm{MOM}}\in\mbR^{p\times p}$ has its $(l_1,l_2)$-entry defined as
	\begin{align*}
		\tilde{\mathcal{C}}_{\mathrm{MOM}, l_1,l_2}=&2\pi\sqrt{\sigma_{l_1,l_1}\sigma_{l_2,l_2}}\;\mbE\left[\{I_{0,l_1}(0)-G_{0,l_1}(0)\}\{I_{0,l_2}(0)-G_{0,l_2}(0)\}\right]\\
			=&2\pi\sqrt{\sigma_{l_1,l_1}\sigma_{l_2,l_2}}\left\{\mbP\left(\frac{\sqrt{n}(\bar{X}_{j,l_1}-\mu_{l_1})}{\sqrt{\sigma_{l_1,l_1}}}\leq0, \frac{\sqrt{n}(\bar{X}_{j,l_2}-\mu_{l_2})}{\sqrt{\sigma_{l_2,l_2}}}\leq0\right)-G_{0,l_1}(0)G_{0,l_2}(0)\right\},
	\end{align*}
	Moreover, we can apply multivariate Berry-Esseen theorem (See Theorem 1.3 in \cite{gotze.1991}) and give
	\begin{align*}
		\tilde{\mathcal{C}}_{\mathrm{MOM}, l_1,l_2}&=\mathcal{C}_{\mathrm{MOM}, l_1,l_2}+O(n^{-1/2}),
	\end{align*}
	where $\mathcal{C}_{\mathrm{MOM}, l_1,l_2}$ is defined in \eqref{eq:med_normal_entry}. Thus the proposition is proved.}
\end{proof}

\begin{lemma}	\label{lem:mom_bahadur}
	Let $N=(m+1)n$ i.i.d. random variables $X_1,...,X_N$ evenly distributed in $(m+1)$ subsets $\mcH_0,\dots,\mcH_m$. There is a subset of index $\mcB\subset\{1,\dots,m\}$ with $\card(\mcB)=\lfloor\alpha_nm\rfloor$. Define
	\begin{equation*}
		\bar{X}_j=
		\begin{cases}
			\frac{1}{n}\sum_{i\in\mcH_j}X_i\quad&j\notin\mcB,\\
			*\quad&j\in\mcB,
		\end{cases}\quad
		\hat{\mu}=\mathrm{med}(\bar{X}_j\mid0\leq j\leq m).
	\end{equation*} 
	Suppose $X_1$ satisfies $\mbE(X_1)=\mu,\var(X_1)=\sigma^2$, and $\mbE|X_1-\mu|^3<\infty$. The fraction $\alpha_n$ satisfies $\alpha_n\leq1/2-\delta$ for some $\delta>0$. Then $\hat{\mu}$ admits the following representation:
	\begin{equation*}
		\hat{\mu}=\mu-\frac{\sqrt{2\pi}\sigma}{(m+1)\sqrt{n}}\sum_{j\notin\mcB}\left\{\mbI(\bar{X}_j\leq \mu)-\mbP(\bar{X}_j\leq\mu)\right\}+O_{\mbP}\left(\frac{\alpha_n}{\sqrt{n}}+\frac{\log n}{m\sqrt{n}}+\frac{\log^{3/4}n}{m^{3/4}n^{1/2}}+\frac{1}{n}\right).
	\end{equation*}
\end{lemma}

\begin{proof}
	Using Taylor expansion we have
	\begin{align*}
		\Phi\left\{\frac{\sqrt{n}(\hat{\mu}-\mu)}{\sigma}\right\}-\Phi(0)=&\frac{\sqrt{n}}{\sqrt{2\pi}\sigma}(\hat{\mu}-\mu)+O(n(\hat{\mu}-\mu)^2)\\
			=&\frac{\sqrt{n}}{\sqrt{2\pi}\sigma}(\hat{\mu}-\mu)+O_{\mbP}\left(\alpha_n^2+\frac{\log n}{m}\right),\stepcounter{equation}\tag{\theequation}\label{eq:mom_lemma_term1}
	\end{align*}
	since $\hat{\mu}-\mu=O_{\mbP}(\frac{\alpha_n}{\sqrt{n}}+\sqrt{\frac{\log n}{mn}})$ by Lemma \ref{lemma:concen_mom_byzantine}. On the other hand, denote
	\begin{equation*}
		G(x)=\mbP\left\{\frac{\sqrt{n}(\bar{X}_1-\mu)}{\sigma}\leq x\right\},\quad I_j(x)=\mbI\left\{\frac{\sqrt{n}(\bar{X}_j-\mu)}{\sigma}\leq x\right\}.
	\end{equation*} 
	From Berry-Essen theorem (Lemma \ref{lemma:berry_esseen}) we know Lemma \ref{lemma:empirical} is applicable and yields
	\begin{align*}
		&\frac{1}{m+1}\sum_{j=0}^mI_j\left\{\frac{\sqrt{n}(\hat{\mu}-\mu)}{\sigma}\right\}-G\left\{\frac{\sqrt{n}(\hat{\mu}-\mu)}{\sigma}\right\}-\frac{1}{m+1}\sum_{j=0}^mI_j(0)+G(0)\\
		=&O_{\mbP}\left(\alpha_n+\frac{\log^{3/4}n}{m^{3/4}}+\frac{1}{\sqrt{n}}\right).
	\end{align*}	
	By definition of median we know 
	\begin{equation*}
		\Phi(0)=\frac{1}{2}=\frac{1}{m+1}\sum_{j=0}^mI_j\left\{\frac{\sqrt{n}(\hat{\mu}-\mu)}{\sigma}\right\}+O\left(\frac{1}{m}\right).
	\end{equation*}	
	Thus
	\begin{align*}
		\Phi\left\{\frac{\sqrt{n}(\hat{\mu}-\mu)}{\sigma}\right\}-\Phi(0)=&G\left\{\frac{\sqrt{n}(\hat{\mu}-\mu)}{\sigma}\right\}-\frac{1}{m+1}\sum_{j=0}^mI_j\left\{\frac{\sqrt{n}(\hat{\mu}-\mu)}{\sigma}\right\}+O\left(\frac{1}{m}+\frac{1}{\sqrt{n}}\right)\\
			=&-\frac{1}{m+1}\sum_{j\notin\mcB}\{I_j(0)-G(0)\}+O_{\mbP}\left(\alpha_n+\frac{\log^{3/4}n}{m^{3/4}}+\frac{1}{\sqrt{n}}\right).\stepcounter{equation}\tag{\theequation}\label{eq:mom_lemma_term2}
	\end{align*}
	 Combining \eqref{eq:mom_lemma_term1} and \eqref{eq:mom_lemma_term2} we have
	 \begin{align*}
	 	\hat{\mu}=&\mu-\frac{\sqrt{2\pi}\sigma}{(m+1)\sqrt{n}}\sum_{j\notin\mcB}\left\{I_j(0)-G(0)\right\}+O_{\mbP}\left(\frac{\alpha_n}{\sqrt{n}}+\frac{\log n}{m\sqrt{n}}+\frac{\log^{3/4}n}{m^{3/4}n^{1/2}}+\frac{1}{n}\right)\\
			=&\mu-\frac{\sqrt{2\pi}\sigma}{(m+1)\sqrt{n}}\sum_{j\notin\mcB}\left\{\mbI(\bar{X}_j\leq\mu)-\mbP(\bar{X}_j\leq\mu)\right\}+O_{\mbP}\left(\frac{\alpha_n}{\sqrt{n}}+\frac{\log n}{m\sqrt{n}}+\frac{\log^{3/4}n}{m^{3/4}n^{1/2}}+\frac{1}{n}\right),
	 \end{align*}
	 which is exactly what we want to prove.
\end{proof}

\section{Positive Definiteness of $\mathcal{C}_{\mathrm{MOM}}-\mathcal{C}$}\label{sec:posit_def}

{From Theorem \ref{cor:VRMOM_limit} and Proposition \ref{cor:MOM_limit}, in order to show that our proposed multi-dimensional VRMOM estimator $\bar{\vect{\mu}}$ has higher statistical efficiency than the corresponding MOM estimator $\hat{\vect{\mu}}$, it is equivalent to prove that $\vect{\mathcal{C}}\preceq \vect{\mathcal{C}}_{\mathrm{MOM}}$ holds true. In this appendix, we will verify that the covariance difference $\vect{\mathcal{C}}_{\mathrm{MOM}}-\vect{\mathcal{C}}$ is positive definite in dimension 2 as $K$ tends to infinity. First of all, we shall give a general formulation for the entry $\mathcal{C}_{l_1,l_2}$ as $K\rightarrow\infty$.}

\begin{lemma}	\label{lem:cl_lim}
	Denote $\mathcal{C}_{l_1,l_2}^K$ as the $(l_1,l_2)$-entry of the matrix $\vect{\mathcal{C}}$ defined in \eqref{eq:normal_cov_entry}. Then we have
	\begin{equation}	\label{eq:cov_entry_infK}
		\lim_{K\rightarrow\infty}\mathcal{C}_{l_1,l_2}^K=\Big\{4\pi\int_{-\infty}^{\infty}\int_{-\infty}^{\infty}\psi(y_1)\psi(y_2)F_{l_1,l_2}(y_1,y_2)\diff y_1\diff y_2-\pi\Big\}\sqrt{\sigma_{l_1,l_1}\sigma_{l_2,l_2}}.
	\end{equation}
	In particular, when $l_1=l_2=l$, we have
	\begin{equation}	\label{eq:cov_diag_infK}
		\lim_{K\rightarrow\infty}\mathcal{C}_{l,l}^K=\frac{\pi}{3}\sigma_{l,l}.
	\end{equation}
\end{lemma}

\begin{proof}
	For the denominator in \eqref{eq:normal_cov_entry}, we compute that
\begin{equation}	\label{eq:part1}
\begin{aligned}
	\lim_{K\rightarrow\infty}\frac{1}{K}\sum_{k=1}^K\psi(\Delta_k)=&\lim_{K\rightarrow\infty}\frac{K+1}{K}\sum_{k=1}^K\psi\Big(\Phi^{-1}\Big(\frac{k}{K+1}\Big)\Big)\frac{1}{K+1}\\
		=&\int_{0}^{1}\psi\big(\Phi^{-1}(x)\big)\diff x\\
		(\text{change variable } x=\Phi(y))=&\int_{-\infty}^{\infty}\psi^2(y)\diff y\\
		=&\frac{1}{2\pi}\int_{-\infty}^{\infty}e^{-x^2}\diff x=\frac{1}{2\sqrt{\pi}}.
\end{aligned}
\end{equation}
For the numerator, on the one hand, we have
\begin{equation}	\label{eq:part2}
\begin{aligned}
	\lim_{K\rightarrow\infty}\frac{1}{K^2}\sum_{k_1,k_2=1}^K\tau_{k_1}\tau_{k_2}=&\lim_{K\rightarrow\infty}\Big(\frac{1}{K}\sum_{k=1}^K\tau_{k}\Big)^2\\
		=&\lim_{K\rightarrow\infty}\Big(\frac{1}{K}\sum_{k=1}^K\frac{k}{K+1}\Big)^2\\
		=&\Big(\int_0^1x\diff x\Big)^2=\frac{1}{4}.
\end{aligned}
\end{equation}
On the other hand, we have that
\begin{equation}	\label{eq:part3}
\begin{aligned}
	\lim_{K\rightarrow\infty}\frac{1}{K^2}\sum_{k_1,k_2=1}^K\tau^{l_1,l_2}_{k_1,k_2}=&\int_{0}^{1}\int_{0}^{1}F_{l_1,l_2}(\Phi^{-1}(x_1),\Phi^{-1}(x_2))\diff x_1\diff x_2\\
		=&\int_{-\infty}^{\infty}\int_{-\infty}^{\infty}\psi(y_1)\psi(y_2)F_{l_1,l_2}(y_1, y_2)\diff y_1\diff y_2,\\
\end{aligned}
\end{equation}
where $F_{l_1,l_2}(y_1,y_2)=\mbP(Z_{l_1}\leq y_1,Z_{l_2}\leq y_2)$. Combining \eqref{eq:part1}, \eqref{eq:part2} and \eqref{eq:part3} we have
\begin{equation}	\label{eq:cl_vrmom}
\begin{aligned}
	\lim_{K\rightarrow\infty}\mathcal{C}_{l_1,l_2}^K=\Big\{4\pi\int_{-\infty}^{\infty}\int_{-\infty}^{\infty}\psi(y_1)\psi(y_2)F_{l_1,l_2}(y_1,y_2)\diff y_1\diff y_2-\pi\Big\}\sqrt{\sigma_{l_1,l_1}\sigma_{l_2,l_2}}.
\end{aligned}
\end{equation}
In particular, when $l_1=l_2=l$, we have
\begin{equation*}
	F_{l,l}(y_1,y_2)=\Phi(\min(y_1,y_2)).
\end{equation*}
Substitute it in \eqref{eq:part3}, we can obtain 
\begin{align*}
	\lim_{K\rightarrow\infty}\frac{1}{K^2}\sum_{k_1,k_2=1}^K\tau^{l,l}_{k_1,k_2}=&\int_{-\infty}^{\infty}\int_{-\infty}^{\infty}\psi(y_1)\psi(y_2)\Phi(\min(y_1,y_2))\diff y_1\diff y_2\\
		=&2\int_{-\infty}^{\infty}\psi(y_1)\int_{-\infty}^{y_1}\psi(y_2)\Phi(y_2)\diff y_2\diff y_1\\
		=&\int_{-\infty}^{\infty}\psi(y_1)\Phi^2(y_1)\diff y_1=\frac{1}{3}.
\end{align*}
Therefore we have
\begin{equation*}
	\lim_{K\rightarrow\infty}\mathcal{C}_{l,l}=\frac{1/3-1/4}{1/(4\pi)}\sigma_{l,l}=\frac{\pi}{3}\sigma_{l,l},
\end{equation*}
which completes the proof.
\end{proof}

\begin{proof}[Verification of $\vect{\mathcal{C}}\preceq\vect{\mathcal{C}}_{\mathrm{MOM}}$]
	In the case of dimension $2$, we assume the gradient $\nabla f(X,\vect{\theta}^*) = (\nabla_1 f(X,\vect{\theta}^*),$ \mbox{} $ \nabla_2 f(X,\vect{\theta}^*))^{\tp}$ has covariance matrix
\begin{equation*}
\vect{\Sigma} = \begin{pmatrix}
	\sigma_{1,1}	&	\sin\phi\sqrt{\sigma_{1,1}\sigma_{2,2}}\\
	\sin\phi\sqrt{\sigma_{1,1}\sigma_{2,2}}	& \sigma_{2,2}
\end{pmatrix},
\quad\text{therefore, }\quad
\vect{\Sigma}_{1,2}= \begin{pmatrix}
	1	&	\sin\phi\\
	\sin\phi	&	1
\end{pmatrix}.
\end{equation*}
From Lemma \ref{lem:cl_lim}, we have that $\mathcal{C}_{\mathrm{MOM},l,l}-\mathcal{C}_{l,l} = \pi\sigma_{l,l}/6$ as $K\rightarrow\infty$, and 
\begin{align*}
	\mathcal{C}_{1,2}=&4\pi\left\{\int_{-\infty}^{\infty}\int_{-\infty}^{\infty}\psi(y_1)\psi(y_2)F_{1,2}(y_1, y_2)\diff y_1\diff y_2-1/4\right\}\sqrt{\sigma_{1,1}\sigma_{2,2}}\\
		=&4\pi\left\{\int_{-\infty}^{\infty}\psi(y_2)\diff y_2\int_{-\infty}^{y_2}\diff x_2\int_{-\infty}^{\infty}\int_{-\infty}^{y_1}\psi(y_1)\psi_{\phi}(x_1, x_2)\diff x_1\diff y_1-1/4\right\}\sqrt{\sigma_{1,1}\sigma_{2,2}}\\
		=&4\pi\left\{\int_{-\infty}^{\infty}(1-\Phi(x_1))\diff x_1\int_{-\infty}^{\infty}\int_{-\infty}^{y_2}\psi(y_2)\psi_{\phi}(x_1, x_2)\diff x_2\diff y_2-1/4\right\}\sqrt{\sigma_{1,1}\sigma_{2,2}}\\
		=&\left\{4\pi\int_{-\infty}^{\infty}\int_{-\infty}^{\infty}\Phi(-x_1)\Phi(-x_2)\psi_{\phi}(x_1,x_2)\diff x_1\diff x_2 - \pi\right\}\sqrt{\sigma_{1,1}\sigma_{2,2}},
\end{align*}
where $\psi_{\phi}(\cdot,\cdot)$ denotes the probability density function of the multivariate normal distribution with covariance matrix $\Sigma_{1,2}$, more precisely, we have
\begin{equation*}
	\psi_{\phi}(x_1,x_2) = \frac{1}{2\pi\cos\phi}\exp\Big\{-\frac{x_1^2-2\sin\phi\cdot x_1x_2+x_2^2}{2\cos^2\phi}\Big\}.
\end{equation*}
By symmetry we clearly see that
\begin{equation*}
	\psi_{\phi}(x_1,x_2) = \psi_{\phi}(-x_1,-x_2),\quad \psi_{\phi}(x_1,-x_2) = \psi_{\phi}(-x_1,x_2)=\psi_{-\phi}(x_1,x_2),\quad\Phi(x)+\Phi(-x)=1.
\end{equation*}
Then $\mathcal{C}_{1,2}$ can be further simplified as follows
\begin{align*}
	&\mathcal{C}_{1,2}/\sqrt{\sigma_{1,1}\sigma_{2,2}} \\
		= &4\pi\int_{-\infty}^0\int_{-\infty}^0\big\{\Phi(-x_1)\Phi(-x_2)+\Phi(-x_1)\Phi(x_2)+\Phi(x_1)\Phi(-x_2)+\Phi(x_1)\Phi(x_2)\big\}\psi_{\phi}(x_1,x_2)\diff x_1\diff x_2\\
		&+4\pi\int_{-\infty}^0\int_{-\infty}^0\big\{\Phi(-x_1)\Phi(x_2)+\Phi(x_1)\Phi(-x_2)\big\}\big\{\psi_{\phi}(x_1,-x_2)-\psi_{\phi}(x_1,x_2)\big\}\diff x_1\diff x_2-\pi\\
		= &2\pi\int_{-\infty}^0\int_{-\infty}^0\big\{1-2\Phi(x_2)\big\}\big\{1-2\Phi(x_1)\big\}\big\{\psi_{\phi}(x_1,x_2)-\psi_{-\phi}(x_1,x_2)\big\}\diff x_1\diff x_2.
\end{align*}
Similarly, $\mathcal{C}_{\mathrm{MOM},1,2}$ can be written in a similar form
\begin{equation*}
	\mathcal{C}_{\mathrm{MOM},1,2}/\sqrt{\sigma_{1,1}\sigma_{2,2}} = \pi\int_{-\infty}^0\int_{-\infty}^0\big\{\psi_{\phi}(x_1,x_2)-\psi_{-\phi}(x_1,x_2)\big\}\diff x_1\diff x_2=\phi.
\end{equation*}
Therefore, to prove the positive definiteness of the matrix $\vect{\mathcal{C}}_{\mathrm{MOM}}-\vect{\mathcal{C}}$, it left to prove that
\begin{equation*}
	|\mathcal{C}_{\mathrm{MOM},1,2}-\mathcal{C}_{1,2}|/\sqrt{\sigma_{1,1}\sigma_{2,2}}\leq |\mathcal{C}_{\mathrm{MOM},1,1}-\mathcal{C}_{1,1}|/\sigma_{1,1}=\frac{\pi}{6}.
\end{equation*}
It is equivalent to the following inequality
\begin{equation}	\label{eq:fphi}
\begin{aligned}
	|h(\phi)|:=&\left|\frac{\phi}{\pi}-2\int_{-\infty}^0\int_{-\infty}^0\big\{1-2\Phi(x_2)\big\}\big\{1-2\Phi(x_1)\big\}\big\{\psi_{\phi}(x_1,x_2)-\psi_{-\phi}(x_1,x_2)\big\}\diff x_1\diff x_2\right|\\
	\leq&\frac{1}{6},
\end{aligned}
\end{equation}
holds for all $\phi\in[-\frac{\pi}{2},\frac{\pi}{2}]$. In order to show this bound, we can draw the graph of $h(\phi)$ numerically.
\begin{figure}[t]
	\begin{center}
		\includegraphics[width=0.6\textwidth]{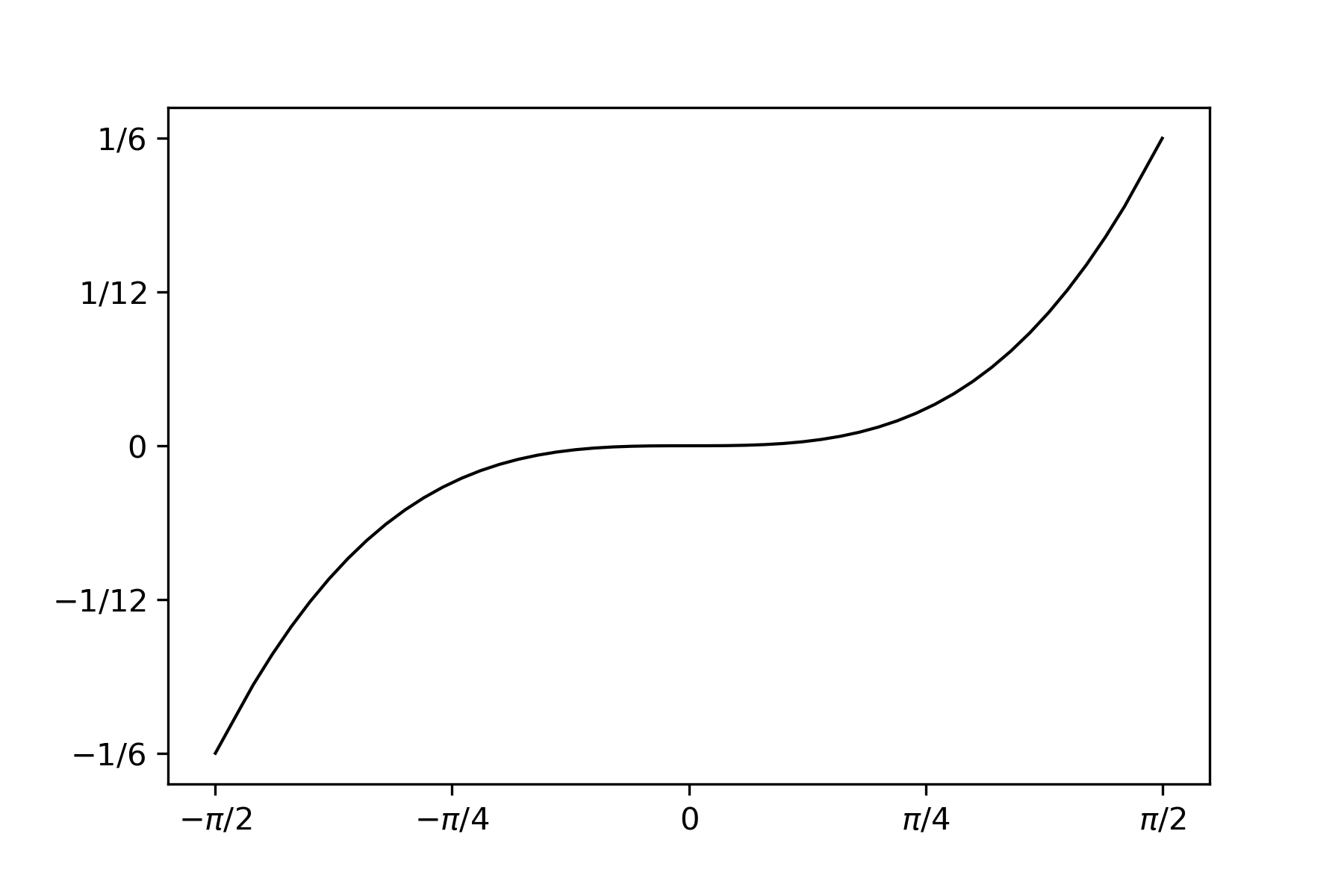}
	\end{center}
	\caption{The graph of function $h(\phi)$ on the interval $[-\pi/2,\pi/2]$. We can see that $h$ is increasing and has absolute value uniformly bounded by $1/6$.}
	\label{fig:dim2_positive}
\end{figure}
As shown in Figure \ref{fig:dim2_positive}, we can see that \eqref{eq:fphi} holds true, which implies that $\vect{\mathcal{C}}\preceq\vect{\mathcal{C}}_{\mathrm{MOM}}$.
\end{proof}


\section{Theories and Proofs for RCSL Estimator}
\label{sec:rcsl_theory}

{This appendix consists of the theoretical results and proofs for the robust CSL estimator. In Appendix \ref{subsec:rdane_theory}, we present the technical assumptions and the main theories for the robust CSL estimator. The proofs of the results will be given in Appendix \ref{sec:proof_RCSL}.}


\subsection{Theoretical Results for Robust CSL Estimator}
\label{subsec:rdane_theory}

{In this part, we present the theoretical results for the robust CSL estimator. Before that, we first introduce some notations and our technical assumptions.}

\subsubsection{Notations and technical assumptions}
\label{sec:tech_assump}

{
As we will demonstrate in Appendix \ref{sec:examples_verification}, all the technical assumptions hold on common statistical models, which suggests the wide applicability of these assumptions. For the loss function $f(x,\vect{\theta})$, we assume that $f(x,\vect{\theta})$ is differentiable with respect to $\vect{\theta}$ and denote $\nabla f(x,\vect{\theta}):=\nabla|_{\vect{\theta}} f(x,\vect{\theta})$ as the gradient of $f(x,\vect{\theta})$ at $\vect{\theta}$. For a given $\vect{\theta}$, we define the expected gradient $\vect{\mu}(\vect{\theta})$ and population standard deviation of the gradient for each coordinate $l$, $\sigma_l(\vect{\theta})$, as follows,
\begin{equation}	\label{eq:mu_sigma}
	\begin{aligned}
		&\vect{\mu}(\vect{\theta})=(\mu_1(\vect{\theta}),\dots\mu_p(\vect{\theta}))^{\tp}:=\mbE_{X \sim \mfX}\{\nabla f(X,\vect{\theta})\},\\
		&\sigma_l^2(\vect{\theta}):=\mbE_{X \sim \mfX}\left[\{\nabla f_l(X,\vect{\theta})-\mu_l(\vect{\theta})\}^2\right], \;\; \text{for} \; 1\leq l\leq p.
	\end{aligned}
\end{equation}
For notational simplicity, we will denote the expectation taken over the randomness of $X$ by $\mbE:=\mbE_{X \sim \mfX}$.
Recall that $\vect{\theta}^*$ is the true parameter that minimizes the loss function $\mbE f(X,\vect{\theta})$. Throughout the paper, we assume that $\vect{\mu}(\vect{\theta}^*)=\mbE\{\nabla f(X,\vect{\theta}^*)\}=\vect{0}$, which holds as long as the expectation and $\nabla$ can be interchanged. Finally, for the ease of presentation, some parameter-independent constants such as $\rho,\eta, C_M$ will be used across different assumptions when there is no confusion.}

{
\begin{assumption}	\label{assump:convexity}
	 For every fixed $x\in\mfX$, $f(x,\vect{\theta})$ is a convex function of $\vect{\theta}$ on $\mbR^p$.
\end{assumption}
\begin{assumption}	\label{assump:hess_eigen}
	 There exists a constant $\rho>0$ such that 
	\begin{equation*}
		\rho\leq\Lambda_{\min}\left[\nabla\vect{\mu}(\vect{\theta}^*)\right]\leq \Lambda_{\max}\left[\nabla\vect{\mu}(\vect{\theta}^*)\right]\leq\rho^{-1}.
	\end{equation*}
\end{assumption}
\begin{assumption}	\label{assump:liptchitz_hess}
	The Hessian of the population loss $\nabla\vect{\mu}$ is Lipschitz continuous. In particular, there exists a constant $C_H>0$ such that
	\begin{equation*}
		\Norm{\,\nabla\vect{\mu}(\vect{\theta}_1)-\nabla\vect{\mu}(\vect{\theta}_2)\,}\leq C_H|\vect{\theta}_1-\vect{\theta}_2|_2,
	\end{equation*}
	holds for arbitrary $\vect{\theta}_1,\vect{\theta}_2\in\mbR^p$.
\end{assumption}
\begin{assumption}	\label{assump:liptchitz_gradient}
	 For every $\vect{v}\in\mbS^{p-1}$ and $x\in\mfX$, define
	\begin{align*}
		M_{\vect{\theta}_1,\vect{\theta}_2}(x,\vect{v}):=&\frac{|\langle \vect{v},\nabla f(x,\vect{\theta}_1)-\nabla f(x,\vect{\theta}_2)\rangle|}{|\vect{\theta}_1-\vect{\theta}_2|_2},\\
		\bar{M}(x,\vect{v}):=&\sup_{\vect{\theta}_1,\vect{\theta}_2 \in \mathbb{R}^p}M_{\vect{\theta}_1,\vect{\theta}_2}(x,\vect{v}).
	\end{align*}
	There exists $C_M,\gamma_0, \eta>0$ such that
	\begin{subequations}		\label{eq:lipgrad_assumption}
		\begin{alignat}{2}
			&\sup_{\vect{v}\in\mbS^{p-1}}\sup_{\vect{\theta}_1,\vect{\theta}_2}\mbE\left[\exp\{\eta M_{\vect{\theta}_1,\vect{\theta}_2}(X,\vect{v})\}\right]\leq C_M,	\label{eq:lipgrad_assumption1}\\
			&\mbE\left[\sup_{\vect{v}\in\mbS^{p-1}}\exp\left\{p^{-\gamma_0}\bar{M}(X,\vect{v})\right\}\right]\leq C_M.	\label{eq:lipgrad_assumption2}
		\end{alignat}
	\end{subequations}
\end{assumption}
\begin{assumption}	\label{assump:bound_variance}
	 There exists a constant $\rho>0$ such that
	\begin{equation*} 
		\rho\leq\min_{1\leq l\leq p}\{\sigma_l(\vect{\theta}^*)\}\leq\max_{1\leq l\leq p}\{\sigma_l(\vect{\theta}^*)\}\leq\rho^{-1},
	\end{equation*}
	where $\sigma_l(\vect{\theta}^*)$ is defined in \eqref{eq:mu_sigma}.
\end{assumption}
\begin{assumption}	\label{assump:coordinate_exp}
	 There exists $\eta>0$ such that
	\begin{equation} 
		\max_{1\leq l\leq p}\mbE\left[\exp \left\{\eta\left|\nabla f_l(X,\vect{\theta}^*)-\mu_l(\vect{\theta}^*)\right|\right\}\right]\leq C_M. \label{eq:coorexp_assump}
	\end{equation}
\end{assumption}
\begin{assumption}	\label{assump:mnp_dependence}
	 The number of machines $m$, distributed sample size $n$, dimension of parameters $p$, and convergence rate $r_n$ of the initial estimator $\hat{\vect{\theta}}^{(0)}$(i.e. $|\hat{\vect{\theta}}^{(0)}-\vect{\theta}^*|_2=O_{\mbP}(r_n)$) satisfy the following relationships
	\begin{equation}	\label{eq:number_assumption}
	\begin{aligned}
		m=o(n), \quad p=O\left(\frac{n^{1/3}}{\sqrt{\log n}}\right), \quad r_n=O\left(\min\left\{\frac{1}{\log n},\; \frac{1}{\sqrt{p\log n}}\right\}\right).
	\end{aligned}
	\end{equation}
\end{assumption}
}

{
Assumptions \ref{assump:convexity} and \ref{assump:hess_eigen} assume the convexity of the loss and the local strong convexity of the population loss function around $\vect{\theta}^*$, which are commonly assumed in empirical risk minimization literature.  Assumptions \ref{assump:liptchitz_hess} and \ref{assump:liptchitz_gradient} are the standard smoothness assumptions, which appear in distributed learning literature  \citep{zhang_etal.2013, jordan_etal.2019}. 
In particular, the Lipschitz gradient assumption \ref{assump:liptchitz_gradient} is represented by two formulas \eqref{eq:lipgrad_assumption1} and \eqref{eq:lipgrad_assumption2} and the \eqref{eq:lipgrad_assumption2} mainly handles the diverging dimensionality, which allows $p$ to go to infinity. Similar conditions can be found in \cite{chen_liu_etal.2018} and \cite{su_xu.2018}. An additional remark is that, our smoothness assumptions \ref{assump:liptchitz_hess} and \ref{assump:liptchitz_gradient} are relatively weaker than those in existing literature. For example, the Assumption PD in \cite{jordan_etal.2019} requires Lipschitz continuity of the second-order derivatives of the loss function $f(X, \vect\theta)$. In contrast, we only require the expectation of the loss function $\mbE(f(X, \vect\theta))$ to be second-order differentiable and Lipschitz continuous (Assumption \ref{assump:liptchitz_hess}), and allow the gradient $\nabla f$ to be non-differentiable. Thus, our theoretical framework handles the Huber loss function as shown in Example \ref{exp:huber_reg}, while \cite{jordan_etal.2019} can not.}

{
Assumptions \ref{assump:bound_variance} and \ref{assump:coordinate_exp}  guarantee concentrating properties for coordinate-wise gradient variance. In Assumption \ref{assump:coordinate_exp} we assume the gradients to be sub-exponential, which is  weaker than the boundedness condition in \cite{alistarh_zhu_li.2018} and the sub-gaussian condition in \cite{yin_chen_etal.2018}. The sub-exponential condition is also assumed in \cite{chen_su_etal.2017} and \cite{su_xu.2018}. We note that the gradients can be sub-exponential in the case of least square regression (see Example \ref{exp:lin_reg} below for more details), which brings additional technical challenges in establishing concentration inequalities. In particular, \cite{yin_etal.2018} imposed bounded absolute skewness (the third-order moment) condition, which is weaker than ours. However, their theory did not consider diverging dimension $p$.}

{
The rate constraints on the quantities $m,n,p,r_n$ are given in Assumption \ref{assump:mnp_dependence}. The relationships on $m$ are inherited from Theorems \ref{thm:normality_vrmom} by letting $\kappa\geq2$. The condition on $p$  indicates that the dimension cannot diverge too fast. The final condition on $r_n$ ensures that the initial estimator is consistent. It is worth noting that the constraint on the initial rate $r_n$ is attainable. By definition, the initial estimator $\hat{\vect{\theta}}^{(0)}$ is the minimizer of the local empirical loss function \eqref{eq:init}. From the regularity assumptions \ref{assump:convexity}--\ref{assump:coordinate_exp}, it is not hard to show that $|\hat{\vect{\theta}}^{(0)}-\vect{\theta}^*|_2 = O_{\mbP}(\sqrt{p\log n/n})$. Plug in the constraint on dimension $p$ in Assumption \ref{assump:mnp_dependence}, we have that $|\hat{\vect{\theta}}^{(0)}-\vect{\theta}^*|_2 = O_{\mbP}(n^{-1/4})$. On the other hand, we can easily verify that $n^{-1/4}=O(\min\{1/\log n,1/\sqrt{p\log n}\})$. Therefore Assumption \ref{assump:mnp_dependence} can be satisfied.}

{
We provide the following two examples for better understanding of the proposed assumptions. In Appendix \ref{sec:examples_verification}, we will verify that Assumptions \ref{assump:convexity}--\ref{assump:coordinate_exp} hold for generalized linear models and a large class of $M$-estimators.
\begin{example}	\label{exp:lin_reg}
	(Linear regression) In linear regression model $Y=\vect{X}^{\tp}\vect{\theta}^*+\epsilon$, we assume that the covariate $\vect{X}=(X_1,\dots,X_p)^{\tp}$ and the noise $\epsilon$ are both sub-gaussian random variables. Define $\vect{\xi}=(Y,\vect{X}^{\tp})^{\tp}$ and the loss function $f(\vect{\xi},\vect{\theta})=(Y-\vect{X}^{\tp}\vect{\theta})^2$. Then we can compute the gradient of the loss function as
	\begin{equation*}
		\nabla f(\vect{\xi},\vect{\theta}^*)=\epsilon\vect{X}.
	\end{equation*}
	In the $l$-th coordinate, the gradient $\nabla f_l(\vect{\xi},\vect{\theta}^*)=\epsilon X_l$ is a product of two sub-gaussian random variables, hence is sub-exponential (See Proposition 2.7.1 of \cite{vershynin.2018}).
\end{example}
\begin{example}	\label{exp:huber_reg}
	(Huber regression) In Huber regression model $Y=\vect{X}^{\tp}\vect{\theta}^*+\epsilon$, similarly define $\vect{\xi}=(Y,\vect{X}^{\tp})^{\tp}$. The loss function is constructed as $f(\vect{\xi},\vect{\theta})=\mcL(Y-\vect{X}^{\tp}\vect{\theta})$, where
	\begin{equation*}
		\mcL(x)=
		\begin{cases}
			x^2/2\quad&\text{for }|x|\leq\delta,\\
			\delta(|x|-\delta/2)\quad&\text{otherwise.}
		\end{cases}
	\end{equation*} 
	Then we can compute that
	\begin{equation*}
		\mcL'(x)=
		\begin{cases}
			x\quad&\text{for }|x|\leq\delta,\\
			\delta\,\mathrm{sign}(x)\quad&\text{otherwise},
		\end{cases}\qquad\mcL''(x)=\mbI(|x|\leq\delta).
	\end{equation*}
	In this case, the Hessian $\nabla^2f(\vect{\xi}, \vect{\theta})=\vect{X}\vect{X}^{\tp}\mbI(|Y-\vect{X}^{\tp}\vect{\theta}|\leq\delta)$ is not continuous with respect to the parameter $\vect{\theta}$. However, if we assume the noise $\epsilon$ has a symmetric distribution and uniformly bounded probability density function, we can prove that Huber regression model fulfills the proposed assumptions. The detailed verification is delegated to Appendix \ref{sec:examples_verification}. 
\end{example}
}


\subsubsection{Theoretical results}
\label{sec:theory:RCSL}

{In this part, we provide the main theorems for the proposed RCSL estimator. We firstly provide single round convergence rate of RCSL estimator. To show the superiority in statistical efficiency of our VRMOM-based RCSL method, we present the asymptotic normality for our RCSL method and the MOM-based RCSL method. Then we will show that the VRMOM-RCSL method has smaller asymptotic variance than the MOM-based counter part.}

{
Firstly, we present our estimation result for one round of communication in the following theorem, which helps understand the improvement from the initial estimator for only one iteration.}

{
\begin{theorem}[One-round convergence rate of the RCSL method]	\label{thm:rdane_byzantine_conv}
	Suppose Assumptions \ref{assump:convexity}-\ref{assump:mnp_dependence} hold and the initial estimator $\hat{\vect{\theta}}^{(0)}$ satisfies $|\hat{\vect{\theta}}^{(0)}-\vect{\theta}^*|_2=O_{\mbP}(r_n)$. Further assume the fraction $\alpha_n$ of Byzantine machines satisfies $\alpha_n\leq 1/2-\delta$ for some fixed $\delta\in(0,1/2)$.
	Then the robust CSL estimator $\hat{\vect{\theta}}^{(1)}$ defined in (\ref{eq:rcsl_estimator}) satisfies
	\begin{equation}	\label{eq:first_rcsl_conv}
		|\hat{\vect{\theta}}^{(1)}-\vect{\theta}^*|_2= O_{\mbP}\left(\frac{\alpha_n\sqrt{p}}{\sqrt{n}}+\sqrt{\frac{p\log n}{mn}}+\frac{p^{1/2}\log^{3/4}n}{n^{1/2}m^{3/4}}+r_n\sqrt{\frac{p^2\log n}{n}}\right).
	\end{equation}
\end{theorem}
Applying Theorem \ref{thm:rdane_byzantine_conv} inductively, we can obtain the convergence result for our RCSL estimator with $t$ rounds of aggregations, which is presented in Theorem \ref{thm:iter_rdane_byzantine_conv} in Section \ref{sec:robust-csl} of the main paper.
}

{
Moreover, similar as Theorem \ref{cor:VRMOM_limit} in the main paper, we can establish asymptotic normality for our RCSL estimator $\hat{\vect{\theta}}^{(t)}$, which has not been studied in previous robust distributed learning literature. To save symbols and avoid repeated definitions, we denote $\sigma_{l_1,l_2}=\cov\{\nabla f_{l_1}(X,\vect{\theta}^*),\nabla f_{l_2}(X,\vect{\theta}^*)\}$ to be the $(l_1,l_2)$-entry of covariance matrix of $\nabla f(X,\vect{\theta}^*)$, which is coincident with the notations in Theorem \ref{cor:VRMOM_limit}. Then we can define $\vect{\mathcal{C}}$ exactly the same as in \eqref{eq:bivariate_normal} and \eqref{eq:normal_cov_entry}. Then we can prove the following asymptotic normality result:
\begin{theorem}[Asymptotic normality of the RCSL method]	\label{cor:RCSL-VRMOM_limit}
	Suppose Assumptions \ref{assump:convexity}-\ref{assump:mnp_dependence} hold, and additionally, we assume the rate constraints $p=o(\min\{\frac{n^{1/3}}{\log^{2/3}n}, \frac{m^{1/2}}{\log^{3/2}n},\frac{n}{m}\}), \alpha_n=o(1/\sqrt{mp})$, and $\log^3n=o(m)$. Then for every iteration number $t$ satisfies \eqref{eq:iteration} and any vector $\vect{v}\in\mbR^p$ with $|\vect{v}|_2 = 1$, we have that
	\begin{equation}	\label{eq:normality_cor}
		\frac{\sqrt{N}}{\sigma_{\vect{v}}}\left\langle \vect{v},\,\hat{\vect{\theta}}^{(t+1)}-\vect{\theta}^*\right\rangle\xrightarrow{d}\mcN(0,1),
	\end{equation}	
	as $n\to\infty$, where
	\begin{equation}	\label{eq:normality_sigma}
		\sigma^2_{\vect{v}}=\vect{v}^{\tp}\{\nabla\vect{\mu}(\vect{\theta}^*)\}^{-1}\vect{\mathcal{C}}\{\nabla\vect{\mu}(\vect{\theta}^*)\}^{-1}\vect{v}.
	\end{equation}
\end{theorem}
Compared with the constraints in Theorem \ref{cor:VRMOM_limit}, both of the theorems require the fraction $\alpha_n=o(1/\sqrt{mp})$, namely, the number of Byzantine machines is $o(\sqrt{m/p})$. However, the normality result of the RCSL method needs more restrictive constraints on the dimension $p$ than the one in Assumption \ref{assump:mnp_dependence} and in Theorem \ref{cor:VRMOM_limit}. 
}

{
As we can see from \eqref{eq:normality_sigma}, the asymptotic variance of the proposed RCSL estimator has a sandwich structure, which is commonly appeared in literatures (see, \eg, \cite{polyak_juditsky.1992, jordan_etal.2019, chen_lee_etal.2020}). However, since the past works only aggregate the gradients by sample mean, the centered covariance matrix in \eqref{eq:normality_sigma} is usually the covariance of the gradient, namely, $\vect{\mathcal{C}}=\mbE\{\nabla f(X,\vect{\theta}^*)\nabla f(X,\vect{\theta}^*)^{\tp}\}$. In contrast, as we aggregate the gradient robustly by our proposed multivariate VRMOM estimator, the structure of the matrix $\vect{\mathcal{C}}$ is much more complex, as we can see in \eqref{eq:bivariate_normal} and \eqref{eq:normal_cov_entry}. To the best of our knowledge, this is the first asymptotic normality result in the setting of Byzantine-robust distributed learning.
}

{
In order to illustrate the efficiency of our RCSL method, it is possible to prove an asymptotic normality result for the median-of-means (MOM) based RCSL method, which is named MOM-RCSL. The explicit construction of the MOM-RCSL method is described in the paragraph before Section \ref{sec:linear}. 
We can show that the asymptotic variance of the MOM-RCSL estimator has the same formulation as \eqref{eq:normality_sigma}, with $\vect{\mcC}$ replaced by $\vect{\mcC}_{\mathrm{MOM}}$ in \eqref{eq:med_normal_entry} of Proposition \ref{cor:MOM_limit}. The proof technique is simply the combination of Proposition \ref{cor:MOM_limit} and Theorem \ref{cor:RCSL-VRMOM_limit}. Therefore, we omit the presentation of this parallel result for brevity. The simulation results of comparison between our RCSL method and the MOM-RCSL method are already presented in Section \ref{subsec:rcsl}.
}


\subsection{Proofs of the Results for Robust CSL Estimator}\label{sec:proof_RCSL}

\subsubsection{More Technical Lemmas}

{In the following, we introduce additional lemmas that will be used for proofing the results related to RCSL estimator. For consistency with Assumption \ref{assump:coordinate_exp}, we assume the random variables admits sub-exponential tail.}


\begin{lemma}	\label{lemma:quantile_gap_mom}
	(Quantile gap of median-of-mean) Let $N$($=(m+1)n$) i.i.d. random variables $X_1,...,X_{N}$ evenly distributed in $m$ subsets $\mcH_0,...,\mcH_m$. Let $\bar{X}_j=n^{-1}\sum_{i\in\mcH_j}X_i$ and $\hat{X}^{\tau}$ be the $\tau$-th quantile of $\{\bar{X}_0,\dots,\bar{X}_m\}$. Suppose $\mbE(X_1)=0,\var(X_1)=\sigma^2$, and $\mbE[e^{\eta|X_1|}]<\infty$ for some $\eta>0$. There are two quantile levels $\tau_2>\tau_1$ satisfying $|\tau_2-\tau_1|=o(1)$ and $\tau_1,\tau_2\in(\delta,1-\delta)$ for some $\delta\in(0,1/2)$. Then for every $\gamma>1$, there exists $\tilde{C}$ such that
	\begin{equation*}
		\mbP\left\{\hat{X}^{\tau_2}-\hat{X}^{\tau_1}\geq \tilde{C}\left(\frac{\tau_2-\tau_1}{\sqrt{n}}+\frac{1}{n}+\frac{\log n}{m\sqrt{n}}\right)\right\}=O(n^{-\gamma}).
	\end{equation*}
\end{lemma}

\begin{proof}
	Follow the proof of Lemma \ref{lemma:concen_mom_byzantine}, we can show that
	\begin{equation}	\label{eq:quantile_gap_term1}	
		\mbP\left\{\left|\hat{X}^{\tau_1}\right|\leq \frac{\sigma}{\sqrt{n}}\Phi^{-1}\left(\frac{1}{2}+\left|\tau_1-\frac{1}{2}\right|+\frac{C_1}{\sqrt{n}}+C_1\sqrt{\frac{\log n}{m}}\right)\right\}\geq 1-O(n^{-\gamma}),
	\end{equation}
	holds for some $C_1$ large enough. Denote
	\begin{equation*}
		\delta_n:=\frac{\sigma}{\sqrt{n}}\Phi^{-1}\left(\frac{1}{2}+\left|\tau_1-\frac{1}{2}\right|+\frac{C_1}{\sqrt{n}}+C_1\sqrt{\frac{\log n}{m}}\right).
	\end{equation*}
	Evenly divide the interval $[-\delta_n,\delta_n]$ into $2n$ pieces and define the set $\mfN:=\{-\delta_n,-\frac{n-1}{n}\delta_n,\dots,\delta_n\}$. Similar as proof in Lemma \ref{lemma:empirical}, we can find some $C_2>0$ such that for every $\tilde{x}\in\mfN$, there is
	\begin{equation}	\label{eq:quantile_gap_term2}
	\begin{aligned}
		&\frac{1}{m+1}\sum_{j=0}^m\sup_{\{x:|\tilde{x}-x|\leq\delta_n/n\}}\left|\mbI(\bar{X}_j\leq x+y)-\mbI(\bar{X}_j\leq x)-\mbI(\bar{X}_j\leq \tilde{x}+y)+\mbI(\bar{X}_j\leq \tilde{x})\right|\\
		\leq&\frac{1}{m+1}\sum_{j=0}^m\mbI\left(|\bar{X}_j-\tilde{x}-y|\leq \frac{\delta_n}{n}\right)+\frac{1}{m+1}\sum_{j=0}^m\mbI\left(|\bar{X}_j-\tilde{x}|\leq \frac{\delta_n}{n}\right)\leq C_2\left(\frac{\log n}{m}+\frac{1}{\sqrt{n}}\right),
	\end{aligned}
	\end{equation}
	holds with probability $1-O(n^{-\gamma-1})$. For $0\leq j\leq m$, further define the random variable
	\begin{equation*}
		Y_j(x,y):=\mbI(\bar{X}_j\leq x+y)-\mbI(\bar{X}_j\leq x)-\mbP(\bar{X}_j\leq x+y)+\mbP(\bar{X}_j\leq x).
	\end{equation*}
	For every $x\in[-\delta_n,\delta_n]$, there exist constants $C_3,C_4$ such that
	\begin{align*}
		&\mbP\left[\frac{1}{m+1}\sum_{j=0}^m\left\{\mbI(\bar{X}_j\leq x+y)-\mbI(\bar{X}_j\leq x)\right\}\leq \tau_2-\tau_1+\frac{C_2\log n}{m}+\frac{C_2}{\sqrt{n}}\right]\\
		=&\mbP\left[\frac{1}{m+1}\sum_{j=0}^mY_j(x,y)\leq \tau_2-\tau_1+\frac{C_2\log n}{m}+\frac{C_2}{\sqrt{n}}-\mbP(\bar{X}_1\leq x+y)+\mbP(\bar{X}_1\leq x)\right]\\
		\leq&\mbP\left[\frac{1}{m+1}\sum_{j=0}^mY_j(x,y)\leq \tau_2-\tau_1+\frac{C_2\log n}{m}+\frac{C_2+C_3}{\sqrt{n}}-\Phi\left(\frac{\sqrt{n}(x+y)}{\sigma}\right)+\Phi\left(\frac{\sqrt{n}x}{\sigma}\right)\right]\\
		\leq&\mbP\left[\frac{1}{m+1}\sum_{j=0}^mY_j(x,y)\leq -C_4\left(\tau_2-\tau_1+\frac{\log n}{m}+\frac{1}{\sqrt{n}}\right)\right],
	\end{align*}
	by finding some $y>0$ such that
	\begin{equation}	\label{eq:quantile_gap_term3}
		\min_{x\in[-\delta_n,\delta_n]}\left\{\Phi\left(\frac{\sqrt{n}(x+y)}{\sigma}\right)-\Phi\left(\frac{\sqrt{n}x}{\sigma}\right)\right\}\geq (C_4+1)(\tau_2-\tau_1)+\frac{C_3}{\sqrt{n}}+(C_2+C_4)\left(\frac{\log n}{m}+\frac{1}{\sqrt{n}}\right).
	\end{equation}
	Taking
	\begin{equation*}
		y_0=\frac{\sigma}{\sqrt{n}\psi(\Phi^{-1}(\delta/2))}\left\{(C_4+1)(\tau_2-\tau_1)+\frac{C_3}{\sqrt{n}}+(C_2+C_4)\left(\frac{\log n}{m}+\frac{1}{\sqrt{n}}\right)\right\},
	\end{equation*}
	since $\tau_2-\tau_1=o(1)$, and $m,n$ tends to infinity, we can assume
	\begin{equation*}
		\frac{\sqrt{n}(x+y_0)}{\sigma},\frac{\sqrt{n}x}{\sigma}\in \left(\Phi^{-1}(\delta/2),\Phi^{-1}(1-\delta/2)\right),
	\end{equation*}
	always hold. Thus by applying mean value theorem to continuous function $\Phi(x)$, \eqref{eq:quantile_gap_term3} can be guaranteed. Further compute that
	\begin{align*}
		\sup_{x\in[-\delta_n,\delta_n]}\mfE\left\{1,Y_j(x,y_0)\right\}=O\left(\tau_2-\tau_1+\frac{\log n}{m}+\frac{1}{\sqrt{n}}\right).
	\end{align*}
	From Lemma \ref{lemma:cai_liu} we have
	\begin{equation}	\label{eq:quantile_gap_term4}
	\begin{aligned}
		&\mbP\left[\frac{1}{m+1}\sum_{j=0}^m\left\{\mbI(\bar{X}_j\leq x+y_0)-\mbI(\bar{X}_j\leq x)\right\}\leq \tau_2-\tau_1+\frac{C_2\log n}{m}+\frac{C_2}{\sqrt{n}}\right]\\
		\leq&\mbP\left[\frac{1}{m+1}\sum_{j=0}^mY_j(x,y_0)\leq -C_4\left(\tau_2-\tau_1+\frac{\log n}{m}+\frac{1}{\sqrt{n}}\right)\right]=O(n^{-\gamma-1}).
	\end{aligned}
	\end{equation}
	Combining \eqref{eq:quantile_gap_term1},\eqref{eq:quantile_gap_term2} and \eqref{eq:quantile_gap_term4}, finally we have
	\begin{align*}
		&\mbP\left[\hat{X}^{\tau_2}-\hat{X}^{\tau_1}\geq y_0\right]\\
		\leq&\mbP\left[\sup_{x\in[-\delta_n,\delta_n]}\sum_{j=0}^m\left\{\mbI(\bar{X}_j\leq x+y_0)-\mbI(\bar{X}_j\leq x)\right\}\leq (\tau_2-\tau_1)(m+1)\right]+O(n^{-\gamma})\\
		\leq&2n\max_{\tilde{x}\in\mfN}\mbP\left[\frac{1}{m+1}\sum_{j=0}^m\left\{\mbI(\bar{X}_j\leq \tilde{x}+y_0)-\mbI(\bar{X}_j\leq \tilde{x})\right\}\leq \tau_2-\tau_1+\frac{C_2\log n}{m}+\frac{C_2}{\sqrt{n}}\right]\\
			&+2n\max_{\tilde{x}\in\mfN}\mbP\Big[\frac{1}{m+1}\sum_{j=0}^m\sup_{\{x:|\tilde{x}-x|\leq\delta_n/n\}}\left|\mbI(\bar{X}_j\leq x+y_0)-\mbI(\bar{X}_j\leq x)-\mbI(\bar{X}_j\leq \tilde{x}+y_0)+\mbI(\bar{X}_j\leq \tilde{x})\right|\\
			&\geq \frac{C_2\log n}{m}+\frac{C_2}{\sqrt{n}}\Big]+O(n^{-\gamma})\\
		=&O(n^{-\gamma}),
	\end{align*}
	therefore prove the lemma.
\end{proof}


\begin{lemma}	\label{lemma:stab_vrmom}
	(Stability of quantile with Byzantine machines) Let $X_0,...,X_m$ be fixed points, and $\mcB\subset\{1,\dots,m\}$ be a subset of index with $\card(\mcB)=\lfloor\alpha m\rfloor$, where $0<\alpha<1/2$. Impose perturbation $\epsilon_j$ on each $X_j$ such that $|\epsilon_j|\leq \delta$ for $j\notin\mcB$, and $\epsilon_j$ can be arbitrary for $j\in\mcB$. Denote the $\tau$'th quantile of the three sets $\{X_0,\dots,X_m\},\{X_0+\epsilon_0,\dots,X_m+\epsilon_m\},\{X_j\mid j\notin\mcB\}$ as $\hat{X}^\tau,\hat{X}^\tau_{\epsilon}$ and $\hat{X}^{\mcB,\tau}$ respectively, Then there is
	\begin{equation*}
		|\hat{X}^{\tau}_{\epsilon}-\hat{X}^{\tau}|\leq \left|\hat{X}^{\mcB,(\tau-\alpha)/(1-\alpha)}-\hat{X}^{\mcB,\tau/(1-\alpha)}\right|+2\delta.
	\end{equation*}
\end{lemma}

\begin{proof}
	By definition of quantiles, $\hat{X}^{\tau}_{\epsilon}$ satisfies the following inequalities
	\begin{equation*}
		\sum_{j=0}^{m}\mbI\left(X_j+\epsilon_j\leq \hat{X}_{\epsilon}^{\tau}\right)\geq \tau (m+1),\quad \sum_{j=0}^{m}\mbI\left(X_j+\epsilon_j\geq \hat{X}_{\epsilon}^{\tau}\right)\geq (1-\tau)(m+1).
	\end{equation*}
	Therefore we have
	\begin{align*}
		&\sum_{j\notin\mcB}\mbI\left(X_j+\epsilon_j\leq \hat{X}^{\tau}_{\epsilon}\right)\geq (\tau-\alpha)(m+1),\quad\sum_{j\notin\mcB}\mbI\left(X_j+\epsilon_j\geq \hat{X}^{\tau}_{\epsilon}\right)\geq (1-\tau-\alpha)(m+1),\\
		\Rightarrow&\sum_{j\notin\mcB}\mbI\left(X_j\leq \hat{X}^{\tau}_{\epsilon}+\delta\right)\geq (\tau-\alpha)(m+1),\quad\sum_{j\notin\mcB}\mbI\left(X_j\geq \hat{X}^{\tau}_{\epsilon}-\delta\right)\geq (1-\tau-\alpha)(m+1).
	\end{align*}
	This implies that
	\begin{equation*}
		\hat{X}^{\mcB,(\tau-\alpha)/(1-\alpha)}\leq\hat{X}^\tau_{\epsilon}+\delta,\quad \hat{X}^{\mcB,\tau/(1-\alpha)}\geq\hat{X}^{\tau}_{\epsilon}-\delta.
	\end{equation*} 
	Similarly we can show
	\begin{equation*}
		\hat{X}^{\mcB,(\tau-\alpha)/(1-\alpha)}\leq\hat{X}^{\tau}\leq \hat{X}^{\mcB,\tau/(1-\alpha)}.
	\end{equation*} 
	So we have $|\hat{X}^{\tau}_{\epsilon}-\hat{X}^{\tau}|\leq \left|\hat{X}^{\mcB,(\tau-\alpha)/(1-\alpha)}-\hat{X}^{\mcB,\tau/(1-\alpha)}\right|+2\delta$.
\end{proof}


\begin{lemma}	\label{lemma:concen_sigma}
	(Exponential concentration of variance) Let $X_1,...,X_n$ be i.i.d. random variables with $\mbE(X_1)=0, \var(X_1)=\sigma^2$, and $\mbE[e^{\eta|X_1|}]\leq C$ for some $\eta,C>0$. Construct sample mean $\bar{X}=n^{-1}\sum_{i=1}^n X_i$ and sample variance $\hat{\sigma}^2=n^{-1}\sum_{i=1}^n\left(X_i-\bar{X}\right)^2$. Then for every $\gamma\geq1$, there exists $\tilde{C}$ large enough, such that
	\begin{equation*}
		\mbP\left(|\hat{\sigma}-\sigma|\geq \tilde{C}\sqrt{\frac{\log n}{n}}\right)=O(n^{-\gamma}).
	\end{equation*}
\end{lemma}

\begin{proof}
	Define
	\begin{align*}
		&\bar{Y}_i:= X_i^2\mbI\left\{X_i^2\leq C_1^2(\log n)^2 \right\}-\mbE\left[X_i^2\mbI\left\{X_i^2\leq C_1^2(\log n)^2\right\}\right],\\
		&\tilde{Y}_i:= X_i^2-\sigma^2-\bar{Y}_i.
	\end{align*}
	We can compute that
	\begin{align*}
		\mbE\left[X_i^2I\left\{X_i^2\geq C_1^2(\log n)^2\right\}\right]\leq&\int_{C_1\log n}^{\infty}2s\mbP(|X_i|\geq s)\diff s+C_1^2(\log n)^2\mbP(|X_i|>C_1\log n)\\
			\leq&\int_{C_1\log n}^{\infty}2Cse^{-\eta s}\diff s+CC_1^2(\log n)^2 n^{-\eta C_1}\\
			=&\left\{2\eta^{-1}CC_1\log n+2\eta^{-2}C+CC_1^2(\log n)^2\right\}n^{-\eta C_1}\leq n^{-\eta C_1+1},
	\end{align*}
	for $n$ sufficiently large. Then
	\begin{align*}
		&\mbP\left(\left|\frac{1}{n}\sum_{i=1}^n\tilde{Y}_i\right|\geq 2n^{-\eta C_1+1}\right)\\
			\leq&\mbP\left[\left|\frac{1}{n}\sum_{i=1}^n|X_i|^2\mbI\left\{|X_i|^2>C_1^2(\log n)^2\right\}\right|\geq n^{-\eta C_1+1}\right]\\
			\leq&\mbP\left\{\max_{1\leq i\leq n}|X_i|^2\geq C_1^2(\log n)^2\right\}\\
			\leq&n\max_{1\leq i\leq n}\mbP\left\{|X_i|\geq C_1\log n\right\}\leq Cn^{-\eta C_1+1}.
	\end{align*}
	Take $C_1\geq \eta^{-1}(\gamma+1)$ we have
	\begin{equation}	\label{eq:square_concentration_term1}
		\mbP\left(\frac{1}{n}\sum_{i=1}^n\tilde{Y}_i\geq 2n^{-\gamma}\right)=O(n^{-\gamma}).
	\end{equation}
	Next we apply Bernstein's inequality \cite{bennett.1962} for bounded random variables $\bar{Y}_i$. Together with the elementary inequality $\mbE(\bar{Y}_1^2)\leq \mbE(X_1^{4})\leq 5\eta^{-4}\mbE\{\exp(\eta |X_1|)\}$, we have
	\begin{align*}
		\mbP\left(\frac{1}{n}\sum_{i=1}^n\bar{Y}_i\geq x\right)\leq&\exp\left\{-\frac{3nx^2}{6\mbE(\bar{Y}_i^2)+2 C_1^2(\log n)^{2}x}\right\}\stepcounter{equation}\tag{\theequation}\label{eq:square_concentration_term2}\\
			\leq&\exp\left[-\min\left\{\frac{\eta^4nx^2}{20C},\frac{3nx}{4 C_1^2(\log n)^2}\right\}\right]=O(n^{-\gamma}),
	\end{align*}
	by taking $x\geq \eta^{-2}\sqrt{20\gamma C}\sqrt{\frac{\log n}{n}}$. Combining \eqref{eq:square_concentration_term1} and \eqref{eq:square_concentration_term2}, we have proved
	\begin{align*}
		&\mbP\left(\frac{1}{n}\sum_{i=1}^nX_i^2-\sigma^2\geq \eta^{-2}\sqrt{20\gamma C} \sqrt{\frac{\log n}{n}}+2n^{-\gamma}\right)\stepcounter{equation}\tag{\theequation}\label{eq:variance_concentration_term1}\\
			\leq& \mbP\left\{\frac{1}{n}\sum_{i=1}^n\bar{Y}_i\geq \eta^{-2}\sqrt{20\gamma C} \sqrt{\frac{\log n}{n}}\right\}+\mbP\left(\frac{1}{n}\sum_{i=1}^n\tilde{Y}_i\geq 2n^{-\gamma}\right)=O(n^{-\gamma}).
	\end{align*}
	 On the other hand, using the fact $x_0^2\leq e^{|x_0|}$ we have
	\begin{align*}
		\mfE\left(\frac{\eta}{2},X_1\right)=\frac{4}{\eta^2}\mbE\left\{\frac{\eta^2}{4}X_1^2\exp\frac{\eta}{2}|X_1|\right\}\leq\frac{4}{\eta^2}\mbE\left(e^{\eta|X_1|}\right)\leq\frac{4C}{\eta^2}.
	\end{align*}
	 Then Lemma \ref{lemma:cai_liu} yields
	 \begin{equation}	\label{eq:variance_concentration_term2}
	 	\mbP\left\{\left|\bar{X}\right|\geq(\frac{4}{\eta^2}+1)\sqrt{\gamma C}\sqrt{\frac{\log n}{n}}\right\}=O(n^{-\gamma}).
	 \end{equation}
	Combining \eqref{eq:variance_concentration_term1} and \eqref{eq:variance_concentration_term2} we have
	\begin{align*}
		&\mbP\left(|\hat{\sigma}-\sigma|\geq\tilde{C}\sqrt{\frac{\log n}{n}} \right)\leq\mbP\left(\left|\hat{\sigma}^2-\sigma^2\right|\geq \sigma \tilde{C}\sqrt{\frac{\log n}{n}}\right)\\
			\leq&\mbP\left(\left|\frac{1}{n}\sum_{i=1}^nX_i^2-\sigma^2\right|+\bar{X}^2\geq \sigma \tilde{C}\sqrt{\frac{\log n}{n}}\right)= O(n^{-\gamma}),
	\end{align*}
	by letting $\tilde{C}$ large enough.
\end{proof}

Lemma \ref{lemma:concen_sigma} directly implies that, under Assumption \ref{assump:bound_variance}, \ref{assump:coordinate_exp} and \ref{assump:mnp_dependence}, there is
\begin{equation}	\label{eq:concen_sigma_theta}
	\mbP\left\{\max_{1\leq l\leq p}\left|\hat{\sigma}_l(\vect{\theta}^*)-\sigma_l(\vect{\theta}^*)\right|\geq \tilde{C}\sqrt{\frac{\log n}{n}}\right\}= O(n^{-\gamma}),
\end{equation}
holds every $\gamma>1$.


\begin{lemma}	\label{lemma:sup_lipschitz_sigma}
	(Uniform bound of variance difference) Under Assumption \ref{assump:liptchitz_gradient}, \ref{assump:bound_variance}, \ref{assump:coordinate_exp} and \ref{assump:mnp_dependence}, there exists $\tilde{C}$ large enough, such that
	\begin{equation*}
		\mbP\left\{\sup_{\vect{\theta}\in\Theta_0}\max_{1\leq l\leq p}\left|\hat{\sigma}_l(\vect{\theta})-\hat{\sigma}_l(\vect{\theta}^*)\right|\geq \tilde{C}r_n\right\}= O(n^{-\gamma}),
	\end{equation*}
	where $\Theta_0=\{\vect{\theta}:|\vect{\theta}-\vect{\theta}^*|_2\leq r_n\}$ is defined in \eqref{eq:theta_0_def}.
\end{lemma}

\begin{proof}
	From (\ref{eq:lipgrad_assumption2}) in Assumption \ref{assump:liptchitz_gradient} and Assumption \ref{assump:coordinate_exp}, we know 
	\begin{align}
		&\mbP\left\{\max_{ i\in\mcH_0}\max_{1\leq l\leq p}\bar{M}(X_i,\vect{e}_l)<p^{\gamma_0+1}\log n\;\right\}=1-O(n^{-\gamma }), 	\label{sup_lipschitz_bound.ineq}\\
		&\mbP\left[\max_{i\in\mcH_0}\max_{1\leq l\leq p}|\nabla f_l(X_i,\vect{\theta}^*)-\mu_l(\vect{\theta}^*)|<\eta^{-1}(\gamma +2)\log n\;\right]=1-O(n^{-\gamma})	\label{gradient_bound.ineq}.
	\end{align}
	Construct the set of event
	\begin{equation}	\label{gradient_bound.event}
		\mfX_0:=\left\{X_1,...,X_n\middle|\text{ (\ref{gradient_bound.ineq}),  (\ref{sup_lipschitz_bound.ineq}) holds, and }\frac{\rho}{2}\leq\min_{1\leq l\leq p}\hat{\sigma}_l(\vect{\theta}^*)\right\}.
	\end{equation}
	Together with \eqref{eq:concen_sigma_theta}, we know $\mbP(\mfX_0)=1-O(n^{-\gamma})$. 
	
	Let $\mfN_0$ be the $n^{-M}$-net of $\Theta_0$, we know $\card(\mfN_0)\leq(1+2n^M)^p$. Then we will show
	\begin{equation}	\label{ineq:stab_sigma}
		\max_{1\leq l\leq p}\sup_{|\vect{\theta}-\tilde{\vect{\theta}}|_2\leq n^{-M}}|\hat{\sigma}^2_l(\vect{\theta})-\hat{\sigma}^2_l(\tilde{\vect{\theta}})|\leq n^{-2}
	\end{equation}
	always hold under the event $\mfX_0$ and $M$ sufficiently large. Indeed, From \eqref{gradient_bound.ineq} and \eqref{sup_lipschitz_bound.ineq}, we have
	\begin{align*}
		&\max_{1\leq l\leq p}\sup_{|\tilde{\vect{\theta}}-\vect{\theta}|_2\leq n^{-M}}|\hat{\sigma}^2_l(\vect{\theta})-\hat{\sigma}^2_l(\tilde{\vect{\theta}})|\\
			=&\max_{1\leq l\leq p}\sup_{|\tilde{\vect{\theta}}-\vect{\theta}|_2\leq n^{-M}}\Big|\frac{1}{n}\sum_{i\in\mcH_0}\left\{\nabla f_l(X_i,\vect{\theta})-g_l(\vect{\theta})\right\}^2-\frac{1}{n}\sum_{i\in\mcH_0}\left\{\nabla f_l(X_i,\tilde{\vect{\theta}})-g_l(\tilde{\vect{\theta}})\right\}^2\Big| \\
			\leq& \max_{1\leq l\leq p}\sup_{|\tilde{\vect{\theta}}-\vect{\theta}|_2\leq n^{-M}}\Big|\frac{1}{n}\sum_{i\in\mcH_0}\Big\{\bar{M}(X_i,\vect{e}_l)+\frac{1}{n}\sum_{i\in\mcH_0}\bar{M}(X_i,\vect{e}_l)\Big\}\\
			&\times\Big\{\nabla f_l(X_i,\vect{\theta})-g_l(\vect{\theta})+\nabla f_l(X_i,\tilde{\vect{\theta}})-g_l(\tilde{\vect{\theta}})\Big\}n^{-M}\Big|\\
			\leq&\frac{2p^{\gamma_0+1}\log n}{n^{M}}\max_{1\leq l\leq p}\sup_{|\tilde{\vect{\theta}}-\vect{\theta}|_2\leq n^{-M}}\frac{1}{n}\sum_{i\in\mcH_0}\Big[2|\nabla f_l(X_i,\vect{\theta}^*)-\mu_l(\vect{\theta}^*)|+2|g_l(\vect{\theta}^*)-\mu_l(\vect{\theta}^*)|\\
			&+\Big\{\bar{M}(X_i,\vect{e}_l)+\frac{1}{n}\sum_{i\in\mcH_0}\bar{M}(X_i,\vect{e}_l)\Big\}\big(|\vect{\theta}-\vect{\theta}^*|_2+|\tilde{\vect{\theta}}-\vect{\theta}^*|_2\big)\Big]\\
			\leq&\frac{8p^{2\gamma_0+2}\log ^2n}{n^{M}}\leq n^{-(M-2\gamma_0-4)},
	\end{align*}
	Taking $M\geq 2\gamma_0+6$ we obtain \eqref{ineq:stab_sigma}. Thus we have
	\begin{align*}
		&\mbP\left\{\max_{1\leq l\leq p}\sup_{\vect{\theta}\in\Theta_0}|\hat{\sigma}_l(\vect{\theta})-\hat{\sigma}_l(\vect{\theta}^*)|\geq 4x\right\}\\
			\leq& \mbP\left\{\max_{1\leq l\leq p}\max_{\tilde{\vect{\theta}}\in\mfN_0}\left|\hat{\sigma}_l^2(\tilde{\vect{\theta}})-\hat{\sigma}^2_l(\vect{\theta}^*)\right|\geq \rho x,\;\mfX_0\right\}\\
			& +\mbP\left\{\max_{1\leq l\leq p}\sup_{|\vect{\theta}-\tilde{\vect{\theta}}|_2\leq n^{-M}}\left|\hat{\sigma}^2_l(\vect{\theta})-\hat{\sigma}^2_l(\tilde{\vect{\theta}})\right|\geq\rho x,\;\mfX_0\right\}+\mbP(\mfX_0^c)\\
			\leq&p(1+2n^{M})^p\max_{1\leq l\leq p}\sup_{\vect{\theta}\in\Theta_0}\mbP\left\{\left|\hat{\sigma}^2_l(\vect{\theta})-\hat{\sigma}^2_l(\vect{\theta}^*)\right|\geq\rho x,\;\mfX_0\right\}+O(n^{-\gamma}).
	\end{align*}
	The next lemma will show, when we take $x=O( r_n)$, there is
	\begin{equation*}
		\sup_{\{\vect{\theta}:|\vect{\theta}-\vect{\theta}^*|_2\leq r_n\}}\mbP\left\{\left|\hat{\sigma}^2_l(\vect{\theta})-\hat{\sigma}^2_l(\vect{\theta}^*)\right|\geq\rho x,\;\mfX_0\right\}=O(n^{-p(M+\gamma+1)}),
	\end{equation*}
	and hence proves this lemma.
\end{proof}

	
\begin{lemma}	\label{lemma:lipschitz_sigma}
	(Point-wise bound of variance difference) Under Assumption \ref{assump:liptchitz_gradient}, \ref{assump:bound_variance}, \ref{assump:coordinate_exp} and \ref{assump:mnp_dependence}, for every $\gamma>1$, there exists $\tilde{C}$ large enough, such that
	\begin{equation*}
		\max_{1\leq l\leq p}\sup_{\vect{\theta}\in\Theta_0}\mbP\left\{\left|\hat{\sigma}_l^2(\vect{\theta})-\hat{\sigma}_l^2(\vect{\theta}^*)\right|\geq \tilde{C}r_n,\;\mfX_0\right\}= O(n^{-p\gamma}),
	\end{equation*}
	where $\mfX_0$ is the set of events defined in (\ref{gradient_bound.event}), and $\Theta_0=\{\vect{\theta}:|\vect{\theta}-\vect{\theta}^*|_2\leq r_n\}$ is defined in \eqref{eq:theta_0_def}.
\end{lemma}

\begin{proof}	
	For every $\vect{\theta}\in\Theta_0$ and $1\leq l\leq p$, there is
	\begin{align*}	
		|\hat{\sigma}^2_l(\vect{\theta})-\hat{\sigma}^2_l(\vect{\theta}^*)|
			\leq&\underbrace{\Big|\frac{1}{n}\sum_{i\in\mcH_0}\left\{\nabla f_l(X_i,\vect{\theta})-\mu_l(\vect{\theta})\right\}^2-\left\{\nabla f_l(X_i,\vect{\theta}^*)-\mu_l(\vect{\theta}^*)\right\}^2-\sigma^2_l(\vect{\theta})+\sigma^2_l(\vect{\theta}^*)\Big|}_{T}\\
			&+\left|\sigma^2_l(\vect{\theta})-\sigma^2_l(\vect{\theta}^*)\right|+\Big|\Big\{\frac{1}{n}\sum_{i\in\mcH_0}\nabla f_l(X_i,\vect{\theta})-\mu_l(\vect{\theta})\Big\}^2\stepcounter{equation}\tag{\theequation}\label{eq:lipschitz_sigma_lemma}\\
			&-\Big\{\frac{1}{n}\sum_{i\in\mcH_0}\nabla f_l(X_i,\vect{\theta}^*)-\mu_l(\vect{\theta}^*)\Big\}^2\Big|.
	\end{align*}

	Let's firstly deal with the term $|\sigma_l^2(\vect{\theta})-\sigma_l^2(\vect{\theta}^*)|$,
	\begin{equation}	\label{lipschitz_sigma_lemma.term0}
		\begin{aligned}
			&|\sigma_l^2(\vect{\theta})-\sigma_l^2(\vect{\theta}^*)|\\
				=&\left|\mbE\left[\left\{\nabla f_l(X_i,\vect{\theta})-\mu_l(\vect{\theta})\right\}^2-\left\{\nabla f_l(X_i,\vect{\theta}^*)-\mu_l(\vect{\theta}^*)\right\}^2\right]\right|\\
				\leq& r_n\mbE\Big[\big\{2|\nabla f_l(X_i,\vect{\theta}^*)-\mu_l(\vect{\theta}^*)|+r_n\left(M_{\vect{\theta},\vect{\theta}^*}(X_i,\vect{e}_l)+\mbE[M_{\vect{\theta},\vect{\theta}^*}(X_i,\vect{e}_l)]\right)\big\}\\
				&\times\left\{M_{\vect{\theta},\vect{\theta}^*}(X_i,\vect{e}_l)+\mbE[M_{\vect{\theta},\vect{\theta}^*}(X_i,\vect{e}_l)]\right\}\Big]\\
				\leq&r_n\mbE\left[\left(4+4r_n\right)M^2_{\vect{\theta},\vect{\theta}^*}(X_i,\vect{e}_l)+\left|\nabla f_l(X_i,\vect{\theta}^*)-\mu_l(\vect{\theta}^*)\right|^2\right]=O\left(r_n\right).
		\end{aligned}
	\end{equation}
	For the last term in \eqref{eq:lipschitz_sigma_lemma}, there is
	\begin{align*}
		&\Big|\Big\{\frac{1}{n}\sum_{i\in\mcH_0}\nabla f_l(X_i,\vect{\theta})-\mu_l(\vect{\theta})\Big\}^2-\Big\{\frac{1}{n}\sum_{i\in\mcH_0}\nabla f_l(X_i,\vect{\theta}^*)-\mu_l(\vect{\theta}^*)\Big\}^2\Big|\\
			=&\Big|\frac{1}{n}\sum_{i\in\mcH_0}\left\{\nabla f_l(X_i,\vect{\theta})-\mu_l(\vect{\theta})-\nabla f_l(X_i,\vect{\theta}^*)+\mu_l(\vect{\theta}^*)\right\}\Big|\\
			&\times\Big|\frac{1}{n}\sum_{i\in\mcH_0}\left\{\nabla f_l(X_i,\vect{\theta})-\mu_l(\vect{\theta})+\nabla f_l(X_i,\vect{\theta}^*)-\mu_l(\vect{\theta}^*)\right\}\Big|.
	\end{align*}
	Compute that
	\begin{align*}
		&\mfE\left\{\frac{\eta}{2},\nabla f_l(X_1,\vect{\theta})-\mu_l(\vect{\theta})-\nabla f_l(X_1,\vect{\theta}^*)+\mu_l(\vect{\theta}^*)\right\}\\
		\leq &\mfE\left\{\frac{\eta}{2}, r_n\big[M_{\vect{\theta},\vect{\theta}^*}(X_1,\vect{e}_l)+\mbE\{M_{\vect{\theta},\vect{\theta}^*}(X_1,\vect{e}_l)\}\big]\right\}\\
		\leq &r_n^2\mbE\left\{\big[M_{\vect{\theta},\vect{\theta}^*}(X_1,\vect{e}_l)+\mbE\{M_{\vect{\theta},\vect{\theta}^*}(X_1,\vect{e}_l)\}\big]^2\exp\frac{\eta r_n}{2} \big[M_{\vect{\theta},\vect{\theta}^*}(X_1,\vect{e}_l)+\mbE\{M_{\vect{\theta},\vect{\theta}^*}(X_1,\vect{e}_l)\}\big]\right\}\\
		\leq &4\eta^{-2}r_n^2\mbE\left\{\exp\eta\big[M_{\vect{\theta},\vect{\theta}^*}(X_1,\vect{e}_l)+\mbE\{M_{\vect{\theta},\vect{\theta}^*}(X_1,\vect{e}_l)\}\big]\right\}=O(r_n^2),\\
		&\mfE\left\{\frac{\eta}{8},\nabla f_l(X_1,\vect{\theta})-\mu_l(\vect{\theta})+\nabla f_l(X_1,\vect{\theta}^*)-\mu_l(\vect{\theta}^*)\right\}\\
		\leq &\mfE\left\{\frac{\eta}{8}, 2|\nabla f_l(X_1,\vect{\theta}^*)-\mu_l(\vect{\theta}^*)|+r_n\big[M_{\vect{\theta},\vect{\theta}^*}(X_1,\vect{e}_l)+\mbE\{M_{\vect{\theta},\vect{\theta}^*}(X_1,\vect{e}_l)\}\big]\right\}\\
		\leq &64\eta^{-2}\mbE\left\{\exp\frac{\eta}{4}\big[2|\nabla f_l(X_1,\vect{\theta}^*)-\mu_l(\vect{\theta}^*)|+r_nM_{\vect{\theta},\vect{\theta}^*}(X_1,\vect{e}_l)+r_n\mbE\{M_{\vect{\theta},\vect{\theta}^*}(X_1,\vect{e}_l)\}\big]\right\}
		=O(1).
	\end{align*}
	Applying Lemma \ref{lemma:cai_liu} to each averaged term, we have
	\begin{equation}	\label{lipschitz_sigma_lemma.term2}
		\Big|\Big\{\frac{1}{n}\sum_{i\in\mcH_0}\nabla f_l(X_i,\vect{\theta})-\mu_l(\vect{\theta})\Big\}^2-\Big\{\frac{1}{n}\sum_{i\in\mcH_0}\nabla f_l(X_i,\vect{\theta}^*)-\mu_l(\vect{\theta}^*)\Big\}^2\Big|=O_{\mbP}\left(\frac{r_np\log n}{n}\right).
	\end{equation}
	
	Lastly we shall focus on the term $T$ in \eqref{eq:lipschitz_sigma_lemma}. For each $i\in\mcH_0$, denote
	\begin{equation*}
		\begin{array}{ll}
			\mfX_i:=&\left\{X_i\;\middle|\;
				\begin{array}{ll}
					 M_{\vect{\theta},\vect{\theta}^*}(X_i,\vect{e}_l)<\eta^{-1}(\gamma+2)p\log n,\\
					|\nabla f_l(X_i,\vect{\theta}^*)-\mu_l(\vect{\theta}^*)|<\eta^{-1}(\gamma +2)\log n
				\end{array}\right\},\\
			Y_i:=&\{\nabla f_l(X_i,\vect{\theta})-\mu_l(\vect{\theta})\}^2-\{\nabla f_l(X_i,\vect{\theta}^*)-\mu_l(\vect{\theta}^*)\}^2,\\
			\bar{Y}_i:=&Y_i\mbI(\mfX_i)-\mbE\{Y_i\mbI(\mfX_i)\},\quad
			\tilde{Y}_i:=Y_i\mbI(\mfX_i^c)-\mbE\{Y_i\mbI(\mfX_i^c)\}.
		\end{array}
	\end{equation*}
	We can compute that
	\begin{align*}
		\big|\mbE\{Y_1\mbI(\,\mfX^c_1)\}\big|=&\left|\mbE\left[\{\nabla f_l(X_1,\vect{\theta})-\mu_l(\vect{\theta})\}^2\mbI(\,\mfX^c_1)-\{\nabla f_l(X_1,\vect{\theta}^*)-\mu_l(\vect{\theta}^*)\}^2\mbI(\,\mfX^c_1)\right]\right|\\
			\leq& \mbE\Big[\big\{2|\nabla f_l(X_1,\vect{\theta}^*)-\mu_l(\vect{\theta}^*)|+\big(M_{\vect{\theta},\vect{\theta}^*}(X_1,\vect{e}_l)+\mbE[M_{\vect{\theta},\vect{\theta}^*}(X_1,\vect{e}_l)]\big)\big\}\\
			&\times\left\{M_{\vect{\theta},\vect{\theta}^*}(X_1,\vect{e}_l)+\mbE[M_{\vect{\theta},\vect{\theta}^*}(X_1,\vect{e}_l)]\right\}\mbI(\,\mfX^c_1)\Big]\\
			\leq& \mbE\left[\big|\nabla f_l(X_1,\vect{\theta}^*)-\mu_l(\vect{\theta}^*)\big|^2\mbI(\,\mfX^c_1)+4\big|M_{\vect{\theta},\vect{\theta}^*}(X_1,\vect{e}_l)\big|^2\mbI(\,\mfX^c_1)\right]\\
			=&O(n^{-2})=o(r_n).
	\end{align*}
	Then from Assumption \ref{assump:liptchitz_gradient} we know 
	\begin{equation}	\label{ineq:complement_sum.sigma_lemma}
		\begin{aligned}
			&\mbP\left\{\left|\frac{1}{n}\sum_{i\in\mcH_0}\tilde{Y}_i\right|>\big|\mbE\{Y_1\mbI(\,\mfX^c_1)\}\big|,\;\mfX_0\right\}\leq\mbP\left\{\Big(\bigcup_{i\in\mcH_0}\mfX_i^c\Big)\bigcap\,\mfX_0\right\}\\
				\leq&\mbP\left\{\max_{i\in\mcH_0}M_{\vect{\theta},\vect{\theta}^*}(X_i,\vect{e}_l)\geq\eta^{-1}(\gamma+2)p\log n\right\}=O(n^{-\gamma p}).
		\end{aligned}
	\end{equation}
	 Thus we have
	\begin{align*}
		&\left|\bar{Y}_1\right|=\left|Y_1\mbI(\mfX_1)-\sigma_l^2(\vect{\theta})+\sigma_l^2(\vect{\theta}^*)+\mbE\{Y_1\mbI(\,\mfX^c_1)\}\right|\\
			\leq&\Big|\big\{\nabla f_l(X_1,\vect{\theta})-\mu_l(\vect{\theta})-\nabla f_l(X_1,\vect{\theta}^*)+\mu_l(\vect{\theta}^*)\big\}\big\{\nabla f_l(X_1,\vect{\theta})-\mu_l(\vect{\theta})+\nabla f_l(X_1,\vect{\theta}^*)-\mu_l(\vect{\theta}^*)\big\}\Big|\mbI(\,\mfX_1)\\
			&+O\left(r_n\right)\\
			\leq&r_n\left[M_{\vect{\theta},\vect{\theta}^*}(X_1,\vect{e}_l)+\mbE\{M_{\vect{\theta},\vect{\theta}^*}(X_1,\vect{e}_l)\}\right]\Big(2|\nabla f_l(X_1,\vect{\theta}^*)-\mu_l(\vect{\theta}^*)|+r_n\big[M_{\vect{\theta},\vect{\theta}^*}(X_1,\vect{e}_l)\\
			&+\mbE\{M_{\vect{\theta},\vect{\theta}^*}(X_1,\vect{e}_l)\}\big]\Big)\mbI(\,\mfX_1) +O\left(r_n\right)\\
			\leq&\left\{\frac{2(\gamma+2)}{\eta} r_n\log n+\frac{(\gamma+2)}{\eta} r_n^2\,p\log n\right\}\left[M_{\vect{\theta},\vect{\theta}^*}(X_1,\vect{e}_l)+\mbE\{M_{\vect{\theta},\vect{\theta}^*}(X_1,\vect{e}_l)\}\right]\mbI(\,\mfX_1)+O\left(r_n\right).
	\end{align*}
	From Assumption \ref{assump:mnp_dependence} and \ref{assump:liptchitz_gradient} we know $r_n\log n+r_n^2\,p\log n=O(1)$, thus
	\begin{equation*}
		\mfE\left(\frac{\eta}{2}, \bar{Y}_1\right)= O\left(r^2_n\log^2 n+r_n^4\,p^2\log^2 n\right)=O(1).
	\end{equation*}
	Then Lemma \ref{lemma:cai_liu} yields
	\begin{equation}	\label{ineq:main_sum.sigma_lemma}
		\mbP\left\{\left|\frac{1}{n}\sum_{i\in\mcH_0}\bar{Y}_i\right|\geq C\left(r_n\sqrt{\frac{p\log^3 n}{n}}+r_n^2\sqrt{\frac{p^3\log^3 n}{n}}\right)\right\}=O(n^{-\gamma p}),
	\end{equation}
	 holds for some $C>0$ large enough. Thus from \eqref{ineq:complement_sum.sigma_lemma} and \eqref{ineq:main_sum.sigma_lemma} we have
	\begin{equation}	\label{lipschitz_sigma_lemma.term1}
		\begin{aligned}
			&\mbP\left\{\left|\frac{1}{n}\sum_{i\in\mcH_0}Y_i\right|\geq C\left(r_n\sqrt{\frac{p\log^3 n}{n}}+r_n^2\sqrt{\frac{p^3\log^3 n}{n}}+r_n\right),\;\mfX_0\right\}\\
			\leq&\mbP\left\{\left|\frac{1}{n}\sum_{i\in\mcH_0}\bar{Y}_i\right|\geq C\left(r_n\sqrt{\frac{p\log^3 n}{n}}+r_n^2\sqrt{\frac{p^3\log^3 n}{n}}\right)\right\}\\
				&+\mbP\left\{\left|\frac{1}{n}\sum_{i\in\mcH_0}\tilde{Y}_i\right|>Cr_n,\;\mfX_0\right\}= O(n^{-\gamma p}).
		\end{aligned}
	\end{equation}
	Finally taking \eqref{lipschitz_sigma_lemma.term0}, \eqref{lipschitz_sigma_lemma.term2} and \eqref{lipschitz_sigma_lemma.term1} back to \eqref{eq:lipschitz_sigma_lemma}, we conclude that, under $\mfX_0$, there is
	\begin{align*}
		&\left|\hat{\sigma}^2_l(\vect{\theta})-\hat{\sigma}^2_l(\vect{\theta}^*)\right|=O_{\mbP}\left(r_n\sqrt{\frac{p\log^3 n}{n}}+r_n^2\sqrt{\frac{p^3\log^3 n}{n}}+r_n\right)=O\left(r_n\right),
	\end{align*}
	since we already supposed $p=O(\sqrt{n}\log^{-1} n)$ and $r_n^2\,p\log n=O(1)$ in Assumption \ref{assump:mnp_dependence}.
\end{proof}


\begin{lemma}	\label{stab_modifier_byzantine.lemma}
	(Stability of correction terms) Let $N$($=(m+1)n$) i.i.d. random variables $X_1,...,X_{N}$ evenly distributed in $m+1$ subsets $\mcH_0,...,\mcH_m$. Let
	\begin{align*}
		\bar{X}_j=&\,n^{-1}\sum_{i\in\mcH_j}X_i,\\
		\hat{X}=&\,\med(\bar{X}_j\mid0\leq j\leq m),\\
		\hat{\sigma}^2=&\,n^{-1}\sum_{i\in\mcH_0}(X_i-\bar{X}_0)^2.
	\end{align*}
	Suppose $X_1$ satisfies $\mbE(X_1)=0,\var(X_1)=\sigma^2$ and $\mbE[e^{\eta|X_1|}]\leq C$. Then there exists $\tilde{C}>0$ sufficiently large, such that
	\begin{align*}
		\mbP\left[\frac{1}{m+1}\sum_{j=0}^{m}\mbI\left\{\left|\bar{X}_j-\hat{X}-\frac{\hat{\sigma}\Delta_k}{\sqrt{n}}\right|\leq \delta_n\right\}\geq \tilde{C}\max\left\{\frac{\log n}{m},\sqrt{n}\delta_n\right\}\right]\leq O(n^{-\gamma}),
	\end{align*}
	provided $\max\{1/(n\delta_n),\sqrt{n}\delta_n\}=O(1)$.
\end{lemma}

\begin{proof}
	Construct the set
	\begin{equation*}
		\mfX_Q\triangleq\left\{X_1,...,X_N\,\middle|\,\hat{\sigma}\leq\sqrt{n},\, |\hat{X}|\leq \sqrt{n}\right\}.
	\end{equation*}
	From Lemma \ref{lemma:concen_mom_byzantine} and Lemma \ref{lemma:concen_sigma}, it is easy to see $\mfX_Q$ holds with probability greater than $1-O(n^{-\gamma})$. Construct an $n^{-2}$-net $\mfN_Q$ for the square $\Theta_Q\triangleq[0,\sqrt{n}]\times[-\sqrt{n},\sqrt{n}]$, then there is
	\begin{align*}
		&\frac{1}{m+1}\sum_{j=0}^{m}\mbI\left\{\left|\bar{X}_j-\hat{X}-\frac{\hat{\sigma}\Delta_k}{\sqrt{n}}\right|\leq \delta_n\right\}\\
			\leq&\sup_{(\sigma,y)\in\Theta_Q}\frac{1}{m+1}\sum_{j=0}^{m}\mbI\left\{\left|\bar{X}_j-y-\frac{\sigma\Delta_k}{\sqrt{n}}\right|\leq \delta_n\right\}\\
			\leq&\max_{(\tilde{\sigma},\tilde{y})\in\mfN_Q}\frac{1}{m+1}\sum_{j=0}^{m}\mbI\left\{\left|Y_j-\tilde{y}-\frac{\tilde{\sigma}\Delta_k}{\sqrt{n}}\right|\leq \delta_n\right\}\\
			&+\max_{(\tilde{\sigma},\tilde{y})\in\mfN_Q}\sup_{\{(\sigma,y):|\tilde{\sigma}-\sigma|,|\tilde{y}-y|\leq n^{-2}\}} \frac{1}{m+1}\sum_{j=0}^{m} \left[\mbI\left\{\left|\bar{X}_j-\tilde{y}-\frac{\tilde{\sigma}\Delta_k}{\sqrt{n}}\right|\leq \delta_n\right\}- \mbI\left\{\left|\bar{X}_j-y-\frac{\sigma\Delta_k}{\sqrt{n}}\right|\leq \delta_n\right\}\right]\\
			\leq&2\max_{(\tilde{\sigma},\tilde{y})\in\mfN_Q}\frac{1}{m+1}\sum_{j=0}^{m}\mbI\left\{\left|\bar{X}_j-\tilde{y}-\frac{\tilde{\sigma}\Delta_k}{\sqrt{n}}\right|\leq 2\delta_n\right\}.
	\end{align*}
	Denote
	\begin{align*}
		&Z_{j}(y,\sigma)=: \mbI\left\{\left|\bar{X}_j-y-\frac{\sigma\Delta_k}{\sqrt{n}}\right|\leq 2\delta_n\right\}-\mbP\left\{\left|\bar{X}_j-y-\frac{\sigma\Delta_k}{\sqrt{n}}\right|\leq 2\delta_n\right\},\\
		&\mbP_{j} (y,\sigma) =: \mbP \left\{ \left|\bar{X}_j-y-\frac{\sigma\Delta_k}{\sqrt{n}}\right|\leq 2\delta_n  \right\}.
	\end{align*}
	Apply Lemma \ref{lemma:berry_esseen} to $\sqrt{n}Y_j$, we know there exist constants $C_1,C_2>0$ such that
	\begin{align*}
		&\mbP_{j}(y,\sigma)\leq \mbP\left\{\left|\mcN(0,1)-\sqrt{n}y-\sigma\Delta_k\right|\leq 2\sqrt{n}\delta_n\right\}+\frac{C_1}{\sqrt{n}}\leq 4C_1\sqrt{n}\delta_n\\
		&\mfE\left\{1, Z_{j}(y,\sigma)\right\}  \leq e\left\{1+\mbP_{j}(y,\sigma)\right\}\mbP_{j}(y,\sigma)\\
		\leq& 8ec\sqrt{n}\delta_n\leq C_2\max \left\{ \frac{\log n}{m}, \sqrt{n}\delta_n \right\}.
	\end{align*}
	Apply Lemma \ref{lemma:cai_liu} together with $|\mfN_Q|\leq 2n^{5}$, we have
	\begin{align*}
		&\mbP\left\{\max_{(\tilde{\sigma},\tilde{y})\in\mfN_Q}\frac{1}{m+1}\sum_{j=0}^{m}Z_{j}(\tilde{y},\tilde{\sigma})\geq x\right\}\\
		\leq&2n^5\sup_{(\sigma,y)\in\Theta_Q}\mbP\left\{\frac{1}{m+1}\sum_{j=0}^{m}Z_{j}(y,\sigma)\geq x\right\}\\
		\leq& n^{-\gamma },
	\end{align*}
	by taking $x\geq (\gamma+6)\sqrt{C_2}\max \left\{ \frac{\log n}{m}, \sqrt{n}\delta_n \right\}$. Then we conclude that
	\begin{align*}
		&\frac{1}{m+1}\sum_{j=0}^{m}\mbI\left\{\left|\bar{X}_j-\hat{X}-\frac{\hat{\sigma}\Delta_k}{\sqrt{n}}\right|\leq\delta_n\right\}\\
			\leq&\max_{(\tilde{\sigma},\tilde{y})\in\mfN_Q}\frac{2}{m+1}\sum_{j=0}^{m}Z_{j}(\tilde{y},\tilde{\sigma})+4c\sqrt{n}\delta_n\\
			\leq& 4(\gamma+6)\sqrt{C_2}\max\left\{\frac{\log n}{m},\sqrt{n}\delta_n\right\},
	\end{align*}
	holds with probability larger than $1-O(n^{-\gamma })$. Thus we can obtain the desired result.
\end{proof}

\subsection{Proofs of results in Appendix \ref{sec:theory:RCSL}}

Now, we are ready to provide the proofs of the convergence rate and asymptotic normality for the RCSL estimator.


\begin{proof}[Proof of Theorem \ref{thm:iter_rdane_byzantine_conv} and Theorem \ref{thm:rdane_byzantine_conv}]
	Denote $b_n$ as the desired convergence rate of $\hat{\vect{\theta}}^{(1)}$, and $\vect{g}_j(\vect{\theta})=n^{-1}\sum_{i\in\mcH_j}\nabla f(X_i,\vect{\theta})$. We construct the following sets in the parameter space.
	\begin{align*}
		\Theta_0&=\{\vect{\theta}\in\mathbb{R}^p: |\vect{\theta}-\vect{\theta}^*|_2\leq r_n\},	\stepcounter{equation}\tag{\theequation}\label{eq:theta_0_def}\\
		\Theta_1&=\{\vect{\theta}\in \mathbb{R}^p:|\vect{\theta}-\vect{\theta}^*|_2=b_n\},	\\
		\Theta_2&=\{\vect{\theta}\in \mathbb{R}^p:|\vect{\theta}-\vect{\theta}^*|_2\leq b_n\}.	
	\end{align*}
	Given an initial estimator $\hat{\vect{\theta}}^{(0)}$ lies in the set $\Theta_0$, we will show
	\begin{equation}	\label{bound_circle.ineq}
		\frac{1}{n}\sum_{i\in \mcH_0}f(X_i,\vect{\theta}_1)-\langle \vect{g}_0(\hat{\vect{\theta}}^{(0)})-\bar{\vect{g}}(\hat{\vect{\theta}}^{(0)}), \vect{\theta}_1\rangle>\frac{1}{n}\sum_{i\in \mcH_0}f(X_i,\vect{\theta}^*)-\langle \vect{g}_0(\hat{\vect{\theta}}^{(0)})-\bar{\vect{g}}(\hat{\vect{\theta}}^{(0)}), \vect{\theta}^* \rangle,
	\end{equation}
	holds uniformly for $\vect{\theta}_1\in\Theta_1$ with probability tending to 1. Notice that
	\begin{align*}
		&\frac{1}{n}\sum_{i\in \mcH_0}\left\{f(X_i,\vect{\theta}_1)-f(X_i,\vect{\theta}^*)\right\}-\left\langle \vect{g}_0(\hat{\vect{\theta}}^{(0)})-\bar{\vect{g}}(\hat{\vect{\theta}}^{(0)}), \vect{\theta}_1-\vect{\theta}^*\right\rangle	\stepcounter{equation}\tag{\theequation}\label{eq:convergence_taylor}\\
			=&\left\langle\int_0^1\big[\vect{g}_0(\vect{\theta}^*+s(\vect{\theta}_1-\vect{\theta}^*))-\vect{g}_0(\vect{\theta}^*)-\vect{\mu}(\vect{\theta}^*+s(\vect{\theta}_1-\vect{\theta}^*))+\vect{\mu}(\vect{\theta}^*)\big]\diff s,\vect{\theta}_1-\vect{\theta}^*\right\rangle\\
			&+\left\langle\int_0^1\left\{\vect{\mu}(\vect{\theta}^*+s(\vect{\theta}_1-\vect{\theta}^*))-\vect{\mu}(\vect{\theta}^*)\right\}\diff s,\vect{\theta}_1-\vect{\theta}^*\right\rangle+\left\langle\vect{g}_0(\vect{\theta}^*)- \vect{g}_0(\hat{\vect{\theta}}^{(0)})+\bar{\vect{g}}(\hat{\vect{\theta}}^{(0)}), \vect{\theta}_1-\vect{\theta}^*\right\rangle\\
			=&\frac{1}{2}\left\langle \nabla\vect{\mu}(\vect{\theta}^*)(\vect{\theta}_1-\vect{\theta}^*),\vect{\theta}_1-\vect{\theta}^*\right\rangle\\
			&+\left\langle\underbrace{\int_0^1(1-s)\left\{\nabla\vect{\mu}(\vect{\theta}^*+s(\vect{\theta}_1-\vect{\theta}^*))-\nabla\vect{\mu}(\vect{\theta}^*)\right\}\diff s}_{\vect{T}_1}\,(\vect{\theta}_1-\vect{\theta}^*),\vect{\theta}_1-\vect{\theta}^*\right\rangle\\
			&+\left\langle\underbrace{\int_0^1\Big[\vect{g}_0(\vect{\theta}^*+s(\vect{\theta}_1-\vect{\theta}^*))-\vect{g}_0(\vect{\theta}^*)-\vect{\mu}(\vect{\theta}^*+s(\vect{\theta}_1-\vect{\theta}^*))+\vect{\mu}(\vect{\theta}^*)\Big]\diff s}_{\vect{T}_2},\vect{\theta}_1-\vect{\theta}^*\right\rangle\\
			&+\left\langle\underbrace{\vect{g}_0(\vect{\theta}^*)- \vect{g}_0(\hat{\vect{\theta}}^{(0)})+\vect{\mu}(\hat{\vect{\theta}}^{(0)})-\vect{\mu}(\vect{\theta}^*)}_{\vect{T}_3}, \vect{\theta}_1-\vect{\theta}^*\right\rangle+\left\langle\underbrace{\bar{\vect{g}}(\hat{\vect{\theta}}^{(0)})-\vect{\mu}(\hat{\vect{\theta}}^{(0)})}_{\vect{T}_4}, \vect{\theta}_1-\vect{\theta}^*\right\rangle.
	\end{align*}
	To show positivity of this difference, it left to bound the norms of $\vect{T}_1,\vect{T}_2,\vect{T}_3$ and $\vect{T}_4$ for $\hat{\vect{\theta}}^{(0)}\in\Theta_0$ and $\vect{\theta}_1\in\Theta_1$ uniformly.
	

	Firstly, from assumption \ref{assump:liptchitz_hess}, we have
	\begin{align*}
		\Norm{\vect{T}_1}\leq&\int_0^1(1-s)\Norm{\nabla\vect{\mu}(\vect{\theta}^*+s(\vect{\theta}_1-\vect{\theta}^*))-\nabla\vect{\mu}(\vect{\theta}^*)}dt\\
			\leq&\int_0^1C_H(1-s)s|\vect{\theta}_1-\vect{\theta}^*|_2\diff s=\frac{C_H|\vect{\theta}_1-\vect{\theta}^*|_2}{6}=o(1).
	\end{align*}


	To control the rest three terms, we denote
	\begin{equation}	\label{eq:z_def}
		\begin{aligned}
			\vect{Z}(x,\vect{\theta}):=&\nabla f(x,\vect{\theta})-\nabla f(x,\vect{\theta}^*)-\vect{\mu}(\vect{\theta})+\vect{\mu}(\vect{\theta}^*),
		\end{aligned}
	\end{equation}
	for ease of notations. Then we will show
	\begin{equation}	\label{eq:supz_bound}
		\sup_{\vect{\theta}\in\Theta_0}\left|\frac{1}{n}\sum_{i\in\mcH_0}\vect{Z}(X_i,\vect{\theta})\right|_2=O_{\mbP}\left(r_n\sqrt{\frac{p\log n}{n}}\right).
	\end{equation}
	Construct the $(r_nn^{-M})$-net $\mfN_0$ for the set $\Theta_0$, where $M>0$ is some sufficiently large number. From Lemma 5.2 of \cite{vershynin.2010} we know $\card(\mfN_0)\leq (1+2n^M)^p$. Then we have
	\begin{align*}
		\sup_{\vect{\theta}\in\Theta_0}\left|\frac{1}{n}\sum_{i\in\mcH_0}r_n^{-1}\vect{Z}(X_i,\vect{\theta})\right|_2\leq& \max_{\tilde{\vect{\theta}}\in\mfN_0}\left|\frac{1}{n}\sum_{i\in\mcH_0}r_n^{-1}\vect{Z}(X_i,\tilde{\vect{\theta}})\right|_2\\
			&+\frac{1}{n^{M+1}}\sum_{i\in\mcH_0}\left\{\sup_{\vect{v}\in\mbS^{p-1}}\bar{M}(X_i,\vect{v})\right\}+\frac{1}{n^M}\mbE\left\{\sup_{\vect{v}\in\mbS^{p-1}}\bar{M}(X_1,\vect{v})\right\}.
	\end{align*}
	From \eqref{eq:lipgrad_assumption2} in Assumption \ref{assump:liptchitz_gradient} and Markov inequality, there is
	\begin{align*}
		&\frac{1}{n^M}\mbE\left\{\sup_{\vect{v}\in\mbS^{p-1}}\bar{M}(X_1,\vect{v})\right\}\leq \frac{p^{\gamma_0}}{n^M}\mbE\left[\sup_{\vect{v}\in\mbS^{p-1}}\exp\left\{p^{-\gamma_0}\bar{M}(X_1,\vect{v})\right\}\right]<\frac{C_Mp^{\gamma_0}}{n^M},\\
		&\frac{1}{n^{M+1}}\sum_{i\in\mcH_0}\left\{\sup_{\vect{v}\in\mbS^{p-1}}\bar{M}(X_i,\vect{v})\right\}=O_{\mbP}\left(\frac{p^{\gamma_0}\log n}{n^M}\right).
	\end{align*}
	On the other hand, by standard $\epsilon$-net argument for vector norms, we know that there exists a $1/2$-net $\mfN_{S}$ of $\mbS^{p-1}$ such that $\card(\mfN_S)\leq 5^p$. It holds that
	\begin{align*}
		&\left|\frac{1}{n}\sum_{i\in\mcH_0}\vect{Z}(X_i,\tilde{\vect{\theta}})\right|_2&&=\sup_{\vect{v}\in\mbS^{p-1}}\left\langle \frac{1}{n}\sum_{i\in\mcH_0}\vect{Z}(X_i,\tilde{\vect{\theta}}),\vect{v}\right\rangle\\
			&\mbox{}&&\leq\max_{\vect{v}\in\mfN_S}\left\langle \frac{1}{n}\sum_{i\in\mcH_0}\vect{Z}(X_i,\tilde{\vect{\theta}}),\vect{v}\right\rangle+\sup_{|\vect{v}-\tilde{\vect{v}}|_2\leq 1/2}\left\langle\frac{1}{n}\sum_{i\in\mcH_0}\vect{Z}(X_i,\tilde{\vect{\theta}}),\vect{v}-\tilde{\vect{v}}\right\rangle,\\
		\Rightarrow&\left|\frac{1}{n}\sum_{i\in\mcH_0}\vect{Z}(X_i,\tilde{\vect{\theta}})\right|_2&&\leq 2\max_{\vect{v}\in\mfN_S}\left\langle \frac{1}{n}\sum_{i\in\mcH_0}\vect{Z}(X_i,\tilde{\vect{\theta}}),\vect{v}\right\rangle.
	\end{align*}
	Thus we have
	\begin{align*}
		\mbP\left(\max_{\tilde{\vect{\theta}}\in\mfN_0}\left|\frac{1}{n}\sum_{i\in\mcH_0}r_n^{-1}\vect{Z}(X_i,\tilde{\vect{\theta}})\right|_2\geq 2x\right)\leq& (1+2n^M)^p\max_{\tilde{\vect{\theta}}\in\mfN_0}\mbP\left(\left|\frac{1}{n}\sum_{i\in\mcH_0}r_n^{-1}\vect{Z}(X_i,\tilde{\vect{\theta}})\right|_2\geq 2x\right)\\
			\leq&5^p(1+2n^M)^p\max_{\tilde{\vect{\theta}}\in\mfN_0}\sup_{\vect{v}\in\mbS^{p-1}}\mbP\left(\frac{1}{n}\sum_{i\in\mcH_0}\left\langle r_n^{-1}\vect{Z}(X_i,\tilde{\vect{\theta}}),\vect{v}\right\rangle\geq x\right).
	\end{align*}
	Moreover, by \eqref{eq:lipgrad_assumption1} in Assumption \ref{assump:liptchitz_gradient}, for every $\vect{v}\in\mbS^{p-1}$ and $\tilde{\vect{\theta}}\in\mfN_0$, we can compute that
	\begin{align*}
		&\mfE\left\{\eta/2,\left\langle r_n^{-1}\vect{Z}(X_1,\tilde{\vect{\theta}}),\vect{v}\right\rangle\right\}\\
			\leq&\mbE\left[\left\{M_{\tilde{\vect{\theta}},\vect{\theta}^*}(X_1,\vect{v})+\mbE[M_{\tilde{\vect{\theta}},\vect{\theta}^*}(X_1,\vect{v})]\right\}^2\exp\frac{\eta }{2}\left\{M_{\tilde{\vect{\theta}},\vect{\theta}^*}(X_1,\vect{v})+\mbE[M_{\tilde{\vect{\theta}},\vect{\theta}^*}(X_1,\vect{v})]\right\}\right]\\
			\leq&4\eta^{-2}\mbE\big[\exp\eta\big\{M_{\tilde{\vect{\theta}},\vect{\theta}^*}(X_1,\vect{v})+\mbE[M_{\tilde{\vect{\theta}},\vect{\theta}^*}(X_1,\vect{v})]\big\}\big]\leq4C_M^2\eta^{-2}=O\left(1\right).
	\end{align*}
	So Lemma \ref{lemma:cai_liu} yields
	\begin{align*}
		5^p(1+2n^M)^p\max_{\tilde{\vect{\theta}}\in\mfN_0}\sup_{\vect{v}\in\mbS^{p-1}}\mbP\left(\frac{1}{n}\sum_{i\in\mcH_0}\left\langle r_n^{-1}\vect{Z}(X_i,\tilde{\vect{\theta}}),\vect{v}\right\rangle\geq x\right)=O(n^{-p\gamma}),
	\end{align*}
	with $x=C_1\sqrt{\frac{p\log n}{n}}$ and $C_1$ large enough. Hence we proved the bound \eqref{eq:supz_bound}. Similarly we can give bounds for $\vect{T}_2$ and $\vect{T}_3$ as follows
	\begin{align*}
		|\vect{T}_2|_2\leq&\int_0^1\left|\frac{1}{n}\sum_{i\in \mcH_0}\vect{Z}(X_i,\vect{\theta}^*+s(\vect{\theta}_1-\vect{\theta}^*))\right|_2\diff s\\
			\leq&\sup_{\vect{\theta}\in\Theta_2}\left|\frac{1}{n}\sum_{i\in \mcH_0}\vect{Z}(X_i,\vect{\theta})\right|_2=O\left(b_n\sqrt{\frac{p\log n}{n}}\right);\\
		|\vect{T}_3|_2\leq&\sup_{\vect{\theta}\in\Theta_0}\left|\frac{1}{n}\sum_{i\in \mcH_0}\vect{Z}(X_i,\vect{\theta})\right|_2=O\left(r_n\sqrt{\frac{p\log n}{n}}\right).
	\end{align*}
	Now for the term $\vect{T}_4$, under the rates constraints in Assumption \ref{assump:mnp_dependence}, we can prove a similar result as \eqref{eq:supz_bound}:
	\begin{equation}	\label{eq:gbar_lipschitz}
		\begin{aligned}
			&\big|\bar{\vect{g}}(\hat{\vect{\theta}}^{(0)})-\vect{\mu}(\hat{\vect{\theta}}^{(0)})-\bar{\vect{g}}(\vect{\theta}^*)+\vect{\mu}(\vect{\theta}^*)\big|_2\\
				\leq&\sup_{\vect{\theta}\in\Theta_0}\big|\bar{\vect{g}}(\vect{\theta})-\vect{\mu}(\vect{\theta})-\bar{\vect{g}}(\vect{\theta}^*)+\vect{\mu}(\vect{\theta}^*)\big|_2\\
				=&O_{\mbP}\left(\frac{\alpha_n\sqrt{p}}{\sqrt{n}}+\sqrt{\frac{p\log n}{mn}}+\frac{\sqrt{p}\log n}{m\sqrt{n}}+r_n\sqrt{\frac{p^2\log n}{n}}\right).
		\end{aligned}
	\end{equation}
	However, the proof of \eqref{eq:gbar_lipschitz} involves more delicate analysis, so we delegate this part to Lemma \ref{lemma:gbar_lipschitz} below. Moreover, follow the proof of Theorem \ref{thm:normality_vrmom} together with Assumption \ref{assump:bound_variance} and \ref{assump:coordinate_exp}, we can apply exponential inequality (Lemma \ref{lemma:cai_liu}) to the i.i.d. terms in \eqref{eq:simplified_diff} and yields
	\begin{align*}
		|\bar{\vect{g}}(\vect{\theta}^*)-\vect{\mu}(\vect{\theta}^*)|_2=O_{\mbP}\left(\frac{\alpha_n\sqrt{p}}{\sqrt{n}}+\sqrt{\frac{p\log n}{mn}}+\frac{p^{1/2}\log^{3/4}n}{n^{1/2}m^{3/4}}\right).
	\end{align*}
	 Note that here we have a $\sqrt{\log n}$ in the second term because of the diverging dimension $p$. Thus we have
	\begin{align*}
		|\vect{T}_4|_2\leq& |\bar{\vect{g}}(\vect{\theta}^*)-\vect{\mu}(\vect{\theta}^*)|_2+ \sup_{\theta\in\Theta_0}\big|\bar{\vect{g}}(\vect{\theta})-\vect{\mu}(\vect{\theta})-\bar{\vect{g}}(\vect{\theta}^*)+\vect{\mu}(\vect{\theta}^*)\big|_2\\
			=&O_{\mbP}\left(\frac{\alpha_n\sqrt{p}}{\sqrt{n}}+\sqrt{\frac{p\log n}{mn}}+\frac{p^{1/2}\log^{3/4}n}{n^{1/2}m^{3/4}}+r_n\sqrt{\frac{p^2\log n}{n}}\right),
	\end{align*}
	Thus there exists a constant $\tilde{C}$ such that
	\begin{equation*}
		\Norm{\vect{T}_1}=o(1),\quad\text{and}\quad|\vect{T}_2|_2+|\vect{T}_3|_2+|\vect{T}_4|_2\leq \tilde{C}\left(\frac{\alpha_n\sqrt{p}}{\sqrt{n}}+\sqrt{\frac{p\log n}{mn}}+\frac{p^{1/2}\log^{3/4}n}{n^{1/2}m^{3/4}}+r_n\sqrt{\frac{p^2\log n}{n}}\right).
	\end{equation*}
	Now in the view of Assumption \ref{assump:hess_eigen}, we continue with \eqref{eq:convergence_taylor}, there is
	\begin{align*}
		&\frac{1}{n}\sum_{i\in \mcH_0}\{f(X_i,\vect{\theta}_1)-f(X_i,\vect{\theta}^*)\}-\langle \vect{g}_0(\hat{\vect{\theta}}^{(0)})-\bar{\vect{g}}(\hat{\vect{\theta}}^{(0)}), \vect{\theta}_1-\vect{\theta}^*\rangle\\
			\geq&\frac{\rho_0}{2}|\vect{\theta}-\vect{\theta}^*|_2^2-\Norm{\vect{T}_1}\cdot|\vect{\theta}-\vect{\theta}^*|_2^2-|\vect{T}_2|_2\cdot|\vect{\theta}-\vect{\theta}^*|_2-|\vect{T}_3|_2\cdot|\vect{\theta}-\vect{\theta}^*|_2-|\vect{T}_4|_2\cdot|\vect{\theta}-\vect{\theta}^*|_2\\
			\geq&\frac{\rho_0b_n^2}{2}-o\left(b_n^2\right)-\tilde{C}b_n\left(\frac{\alpha_n\sqrt{p}}{\sqrt{n}}+\sqrt{\frac{p\log n}{mn}}+\frac{p^{1/2}\log^{3/4}n}{n^{1/2}m^{3/4}}+r_n\sqrt{\frac{p^2\log n}{n}}\right)>0,
	\end{align*}
	provided $b_n=O(\frac{\alpha_n\sqrt{p}}{\sqrt{n}}+\sqrt{\frac{p\log n}{mn}}+\frac{p^{1/2}\log^{3/4}n}{n^{1/2}m^{3/4}}+r_n\sqrt{\frac{p^2\log n}{n}})$. So we have (\ref{bound_circle.ineq}) holds true, and thus
	\begin{equation*}
		|\hat{\vect{\theta}}^{(1)}-\vect{\theta}^*|_2=O_{\mbP}\left(\frac{\alpha_n\sqrt{p}}{\sqrt{n}}+\sqrt{\frac{p\log n}{mn}}+\frac{p^{1/2}\log^{3/4}n}{n^{1/2}m^{3/4}}+r_n\sqrt{\frac{p^2\log n}{n}}\right),
	\end{equation*}
	which proves Theorem \ref{thm:rdane_byzantine_conv}. Apply this formula inductively, we can obtain Theorem \ref{thm:iter_rdane_byzantine_conv}.
\end{proof}

\begin{lemma}	\label{lemma:gbar_lipschitz}
	Let $\bar{\vect{g}}(\vect{\theta})$ be defined as \eqref{eq:gbar_def}, with $\hat{\vect{\theta}}^{(0)}$ replaced by $\vect{\theta}$. Then under the Assumption \ref{assump:convexity}-\ref{assump:mnp_dependence}, there is
	\begin{equation*}
		\sup_{\vect{\theta}\in\Theta_0}\big|\bar{\vect{g}}(\vect{\theta})-\vect{\mu}(\vect{\theta})-\bar{\vect{g}}(\vect{\theta}^*)+\vect{\mu}(\vect{\theta}^*)\big|_2=O_{\mbP}\left(\frac{\alpha_n\sqrt{p}}{\sqrt{n}}+\sqrt{\frac{p\log n}{mn}}+\frac{\sqrt{p}\log n}{m\sqrt{n}}+r_n\sqrt{\frac{p^2\log n}{n}}\right).
	\end{equation*}
\end{lemma}

\begin{proof}
	First of all we need to split it into three parts:
	\begin{align*}
		&\sup_{\vect{\theta}\in\Theta_0}|\bar{\vect{g}}(\vect{\theta})-\vect{\mu}(\vect{\theta})-\bar{\vect{g}}(\vect{\theta}^*)+\vect{\mu}(\vect{\theta}^*)|_2\\
		\leq&\underbrace{\sup_{\vect{\theta}\in\Theta_0}\left|\hat{\vect{g}}(\vect{\theta})-\vect{\mu}(\vect{\theta})-\hat{\vect{g}}(\vect{\theta}^*)+\vect{\mu}(\vect{\theta}^*)\right|_2}_{T_{41}}\\
		&+\underbrace{\sup_{\vect{\theta}\in\Theta_0}\max_{1\leq l\leq p}\frac{3\sqrt{p}}{m\sqrt{n}}\sum_{k=1}^K\Big|\left\{\hat{\sigma_l}(\vect{\theta})-\hat{\sigma}_l(\vect{\theta}^*)\right\}\sum_{j=0}^{m}\Big[\mbI\Big\{g_{j,l}(\vect{\theta})\leq\hat{g}_l(\vect{\theta})+\frac{\hat{\sigma}_l(\vect{\theta})\Delta_k}{\sqrt{n}}\Big\}-\frac{k}{K+1}\Big]\Big|}_{T_{42}}\\
		&+\underbrace{\sup_{\vect{\theta}\in\Theta_0}\max_{1\leq l\leq p}\frac{3\sqrt{p}}{m\sqrt{n}}\Big|\hat{\sigma}_l(\vect{\theta}^*)\sum_{k=1}^K\sum_{j=0}^{m}\Big[\mbI\Big\{g_{j,l}(\vect{\theta})\leq\hat{g}_l(\vect{\theta})+\frac{\hat{\sigma}_l(\vect{\theta})\Delta_k}{\sqrt{n}}\Big\}-\mbI\Big\{g_{j,l}(\vect{\theta}^*)\leq\hat{g}_l(\vect{\theta}^*)+\frac{\hat{\sigma}_l(\vect{\theta}^*)\Delta_k}{\sqrt{n}}\Big\}\Big]\Big|}_{T_{43}},
	\end{align*}
	where the factor $3$ comes from the fact $1/\psi(0)=\sqrt{2\pi}<3$. For the first term $T_{41}$, rehash the proof of equation (\ref{eq:supz_bound}) we can easily obtain
	\begin{equation}	\label{eq:zcoordinate_bound}
		\begin{aligned}
			&\mbP\left\{\max_{1\leq l\leq p}\max_{j\notin\mcB}\sup_{\vect{\theta}\in\Theta_0}\left|\frac{1}{n}\sum_{i\in \mcH_j} Z_l(X_i,\vect{\theta})\right|\geq Cr_n\sqrt{\frac{p\log n}{n}}\right\}\\
			=&\mbP\left\{\max_{1\leq l\leq p}\max_{j\notin\mcB}\sup_{\vect{\theta}\in\Theta_0}\left|\frac{1}{n}\sum_{i\in \mcH_j}\langle \vect{Z}(X_i,\vect{\theta}),\vect{e}_l\rangle\right|\geq Cr_n\sqrt{\frac{p\log n}{n}}\right\}\leq mpn^{-\gamma p}.
		\end{aligned}
	\end{equation}
	where $Z_l(X_i,\vect{\theta})$ is the $l$'th coordinate of $\vect{Z}(X_i,\vect{\theta})$ defined in (\ref{eq:z_def}). Then from Lemma \ref{lemma:stab_vrmom} (with $q=1/2$ and $\alpha=\alpha_n$), we know
	\begin{align*}
		&\max_{1\leq l\leq p}\sup_{\vect{\theta}\in\Theta_0}\left|\hat{g}_{l}(\vect{\theta})-\mu_l(\vect{\theta})-\hat{g}_l(\vect{\theta}^*)+\mu_l(\vect{\theta}^*)\right|\\
			\leq&2\max_{1\leq l\leq p}\max_{j\notin\mcB}\sup_{\vect{\theta}\in\Theta_0}\left|\frac{1}{n}\sum_{i\in \mcH_j}Z_l(X_i,\vect{\theta})\right|+\max_{1\leq l\leq p}\left|\hat{g}_l^{\mcB,(1-2\alpha_n)/(2-2\alpha_n)}(\vect{\theta}^*)-\hat{g}_l^{\mcB,1/(2-2\alpha_n)}(\vect{\theta}^*)\right|\\
			\leq&2\max_{1\leq l\leq p}\max_{j\notin\mcB}\sup_{\vect{\theta}\in\Theta_0}\left|\frac{1}{n}\sum_{i\in \mcH_j}Z_l(X_i,\vect{\theta})\right|+2\max_{1\leq l\leq p}\max\Big\{\Big|\hat{g}_l^{\mcB,(1-2\alpha_n)/(2-2\alpha_n)}(\vect{\theta}^*)-\mu_l(\vect{\theta}^*)\Big|,\\
			&\Big|\hat{g}_l^{\mcB,1/(2-2\alpha_n)}(\vect{\theta}^*)-\mu_l(\vect{\theta}^*)\}\Big|\Big\},
	\end{align*}
	where $\hat{g}_l^{\mcB,(1-2\alpha_n)/(2-2\alpha_n)}(\vect{\theta}^*),\hat{g}_l^{\mcB,1/(2-2\alpha_n)}(\vect{\theta}^*)$ represent the $(1-2\alpha_n)/(2-2\alpha_n)$-th and $1/(2-2\alpha_n)$-th quantile of non-Byzantine machines respectively. For the additional term, we can follow the proof of Lemma \ref{lemma:concen_mom_byzantine} and show
	\begin{align*}
		\max_{1\leq l\leq p}\max\Big\{\Big|\hat{g}_l^{\mcB,(1-2\alpha_n)/(2-2\alpha_n)}(\vect{\theta}^*)-\mu_l(\vect{\theta}^*)\Big|,\,\Big|\hat{g}_l^{\mcB,1/(2-2\alpha_n)}(\vect{\theta}^*)-\mu_l(\vect{\theta}^*)\Big|\Big\}=O_{\mbP}\left(\frac{\alpha_n}{\sqrt{n}}+\sqrt{\frac{\log n}{mn}}\right),
	\end{align*}
	provided condition $\alpha_n\leq1/2-\delta$ holds for some $\delta\in(0,1/2)$. Therefore we have
	\begin{equation}	\label{eq:ghat_bound}
		\max_{1\leq l\leq p}\sup_{\vect{\theta}\in\Theta_0}\left|\hat{g}_{l}(\vect{\theta})-\mu_l(\vect{\theta})-\hat{g}_l(\vect{\theta}^*)+\mu_l(\vect{\theta}^*)\right|=O_{\mbP}\left(\frac{\alpha_n}{\sqrt{n}}+\sqrt{\frac{\log n}{mn}}+r_n\sqrt{\frac{p^2\log n}{n}}\right).
	\end{equation} 
	Taking all coordinate together we have
	\begin{equation}	\label{eq:step3_term1}
		\left|T_{41}\right|=O_{\mbP}\left(\frac{\alpha_n\sqrt{p}}{\sqrt{n}}+\sqrt{\frac{p\log n}{mn}}+r_n\sqrt{\frac{p^2\log n}{n}}\right).
	\end{equation}

	
	The rate of $T_{42}$ hinges on the uniform rate of $\hat{\sigma}(\vect{\theta})-\hat{\sigma}(\vect{\theta}^*)$ over $\Theta_0$. Indeed, in Lemma \ref{lemma:lipschitz_sigma} we proved
	\begin{equation}	\label{eq:sigma_coordinate}
		\max_{1\leq l\leq p}\sup_{\vect{\theta}\in\Theta_0}\left|\hat{\sigma}_l(\vect{\theta})-\hat{\sigma}_l(\vect{\theta}^*)\right|=O_{\mbP}\left(r_n\right),
	\end{equation}
	provided $p=O(\sqrt{n}\log^{-1} n)$ in Assumption \ref{assump:mnp_dependence}. So this yields
	\begin{equation}	\label{eq:step3_term2}
		\left|T_{42}\right|\leq \frac{3\sqrt{p}}{\sqrt{n}}\max_{1\leq l\leq p}\sup_{\vect{\theta}\in\Theta_0}\left|\hat{\sigma}_l(\vect{\theta})-\hat{\sigma}_l(\vect{\theta}^*)\right|=O_{\mbP}\left(\frac{r_n\sqrt{p}}{\sqrt{n}}\right).
	\end{equation}


	It left to deal with the term $T_{43}$. From \eqref{eq:zcoordinate_bound}, \eqref{eq:ghat_bound}, \eqref{eq:sigma_coordinate} and \eqref{eq:concen_sigma_theta}, we know there exists a constant $C$ large enough, such that the following inequalities holds uniformly:
	\begin{align*}
		\max_{1\leq l\leq p}\max_{j\notin\mcB}\sup_{\vect{\theta}\in\Theta_0}&\left|g_{l,j}(\vect{\theta})-\mu_l(\vect{\theta})-g_{l,j}(\vect{\theta}^*)+\mu_l(\vect{\theta}^*)\right|\leq Cr_n\sqrt{\frac{p\log n}{n}},\\
		\max_{1\leq l\leq p}\sup_{\vect{\theta}\in\Theta_0}&\left|\hat{g}_{l}(\vect{\theta})-\mu_l(\vect{\theta})-\hat{g}_{l}(\vect{\theta}^*)+\mu_l(\vect{\theta}^*)\right|\leq C\left(\frac{\alpha_n}{\sqrt{n}}+\sqrt{\frac{\log n}{mn}}+r_n\sqrt{\frac{p\log n}{n}}\right),\\
		\max_{1\leq l\leq p}\max_{1\leq k\leq K}\sup_{\vect{\theta}\in\Theta_0}&\left|\hat{\sigma}_{l}(\vect{\theta})-\hat{\sigma}_{l}(\vect{\theta}^*)\right|\Delta_k n^{-1/2}\leq C\frac{r_n}{\sqrt{n}},\\
		\max_{1\leq l\leq p}&\left|\hat{\sigma}_l(\vect{\theta}^*)\right|\leq C,
	\end{align*}
	with probability higher than $1-O(n^{-\gamma})$. Under this event, using Lemma \ref{stab_modifier_byzantine.lemma} with 
	\begin{equation*}
		\delta_n=C\left(\frac{\alpha_n}{\sqrt{n}}+\sqrt{\frac{\log n}{mn}}+3r_n\sqrt{\frac{p\log n}{n}}\right),
	\end{equation*} 
	we have	
	\begin{align*}
		\left|T_{43}\right|\leq& \frac{C\sqrt{p}}{\sqrt{n}}\max_{1\leq k\leq K}\max_{1\leq l\leq p}\frac{1}{m+1}\sum_{j=0}^{m}\mbI\left\{\left|g_{j,l}(\vect{\theta}^*)-\hat{g}_l(\vect{\theta}^*)-\frac{\hat{\sigma}_l(\vect{\theta}^*)\Delta_k}{\sqrt{n}}\right|\leq\delta_n\right\}\\
			\leq&\frac{C\alpha_n\sqrt{p}}{\sqrt{n}}+\frac{C\sqrt{p}}{\sqrt{n}}\max_{1\leq k\leq K}\max_{1\leq l\leq p}\frac{1}{m+1}\sum_{j\notin\mcB}\mbI\left\{\left|g_{j,l}(\vect{\theta}^*)-\hat{g}_l(\vect{\theta}^*)-\frac{\hat{\sigma}_l(\vect{\theta}^*)\Delta_k}{\sqrt{n}}\right|\leq\delta_n\right\}\\
			=&O_{\mbP}\left(\frac{\alpha_n\sqrt{p}}{\sqrt{n}}+\sqrt{\frac{p\log n}{mn}}+r_n\sqrt{\frac{p^2\log n}{n}}\right),\stepcounter{equation}\tag{\theequation}\label{eq:step3_term3}
	\end{align*}
	holds with probability larger than $1-O(pn^{-\gamma})$. Combining (\ref{eq:step3_term1}), (\ref{eq:step3_term2}) and (\ref{eq:step3_term3}), the lemma is proved.
\end{proof}


\begin{proof}[Proof of Theorem \ref{cor:VRMOM_limit}]
	When the iteration number satisfies \eqref{eq:iteration}, and the rate constraints satisfies $\alpha_n=o(1\sqrt{mp})$ and $p=o(\min\{\frac{n^{1/3}}{\log^{2/3}n}, \frac{m^{1/2}}{\log^{3/2}n}\})$, we clearly know that 
	$$
	\hat{\vect{\theta}}^{(t)}\in\Theta_t:=\left\{\vect{\theta}\in\mathbb{R}^{p}:|\vect{\theta}-\vect{\theta}^*|_2\leq C\sqrt{\frac{p\log n}{mn}}\right\},
	$$ 
	with high probability, for some $C$ sufficiently large. Moreover, we need sharper rate for $|\vect{T}_4|_2$ in \eqref{eq:convergence_taylor}. Indeed, from Lemma \ref{lemma:stab_vrmom} we have
	\begin{align*}
		&\max_{1\leq l\leq p}\sup_{\vect{\theta}\in\Theta_t}\left|\hat{g}_{l}(\vect{\theta})-\mu_l(\vect{\theta})-\hat{g}_l(\vect{\theta}^*)+\mu_l(\vect{\theta}^*)\right|\\
			\leq&2\max_{1\leq l\leq p}\max_{j\notin\mcB}\sup_{\vect{\theta}\in\Theta_t}\left|\frac{1}{n}\sum_{i\in \mcH_j}Z_l(X_i,\vect{\theta})\right|+\max_{1\leq l\leq p}\left|\hat{g}_l^{\mcB,(1-2\alpha_n)/(2-2\alpha_n)}(\vect{\theta}^*)-\hat{g}_l^{\mcB,1/(2-2\alpha_n)}(\vect{\theta}^*)\right|.
	\end{align*}
	Note that $1/(2-2\alpha_n)-(1-2\alpha_n)/(2-2\alpha_n)=o(1/\sqrt{mp})$, by Lemma \ref{lemma:quantile_gap_mom} we have sharper constraint on the quantile gap
	 \begin{align*}
		\max_{1\leq l\leq p}\left|\hat{g}_l^{\mcB,(1-2\alpha_n)/(2-2\alpha_n)}(\vect{\theta}^*)-\hat{g}_l^{\mcB,1/(2-2\alpha_n)}(\vect{\theta}^*)\right|= O_{\mbP}\left(\frac{\alpha_n}{\sqrt{n}}+\frac{1}{n}+\frac{\log n}{m\sqrt{n}}\right).
	\end{align*}
	Therefore we have
	\begin{equation*}
	\begin{aligned}
		&|\hat{\vect{g}}(\hat{\vect{\theta}}^{(t)})-\vect{\mu}(\hat{\vect{\theta}}^{(t)})-\hat{\vect{g}}(\vect{\theta}^*)+\vect{\mu}(\vect{\theta}^*)|_2\\
		=&O_{\mbP}\left(\frac{\alpha_n\sqrt{p}}{\sqrt{n}}+\frac{\sqrt{p}}{n}+\frac{\sqrt{p}\log n}{m\sqrt{n}}+\frac{p^{3/2}\log n}{\sqrt{m}n}\right).
	\end{aligned}
	\end{equation*}
	Then follow the proof of Lemma \ref{lemma:gbar_lipschitz} we have
	\begin{equation}	\label{eq:sharper_diff}
		\begin{aligned}
		&|\bar{\vect{g}}(\hat{\vect{\theta}}^{(t)})-\vect{\mu}(\hat{\vect{\theta}}^{(t)})-\bar{\vect{g}}(\vect{\theta}^*)+\vect{\mu}(\vect{\theta}^*)|_2\\
		=&O_{\mbP}\left(\frac{\alpha_n\sqrt{p}}{\sqrt{n}}+\frac{\sqrt{p}}{n}+\frac{\sqrt{p}\log n}{m\sqrt{n}}+\frac{p^{3/2}\log n}{\sqrt{m}n}\right)=o_{\mbP}\left(\frac{1}{\sqrt{mn}}\right).
		\end{aligned}
	\end{equation}
	Now we start to prove asymptotic normality. From equations \eqref{eq:true_para} and \eqref{eq:gbar_iterate}, we know
	\begin{align*}
		&\vect{\mu}(\vect{\theta}^*)=0,\quad\text{and }\vect{g}_0(\hat{\vect{\theta}}^{(t+1)})=\vect{g}_0(\hat{\vect{\theta}}^{(t)})-\bar{\vect{g}}(\hat{\vect{\theta}}^{(t)}).
	\end{align*}
	Therefore, from \eqref{eq:sharper_diff} and \eqref{eq:supz_bound} there is
	\begin{align*}
		&\left|\vect{g}_0(\hat{\vect{\theta}}^{(t+1)})-\vect{g}_0(\vect{\theta}^{*})\right|_2\\
			=&\left|\vect{g}_0(\hat{\vect{\theta}}^{(t)})-\bar{\vect{g}}(\hat{\vect{\theta}}^{(t)})-\vect{g}_0(\vect{\theta}^{*})\right|_2\\
			=&\Big|\vect{\mu}(\vect{\theta}^*)-\bar{\vect{g}}(\vect{\theta}^{*})+\bar{\vect{g}}(\vect{\theta}^{*})-\bar{\vect{g}}(\hat{\vect{\theta}}^{(t)})-\vect{\mu}(\vect{\theta}^*)+\vect{\mu}(\hat{\vect{\theta}}^{(t)})\\
			&+\vect{g}_0(\hat{\vect{\theta}}^{(t)})-\vect{g}_0(\vect{\theta}^{*})-\vect{\mu}(\hat{\vect{\theta}}^{(t)})+\vect{\mu}(\vect{\theta}^*)\Big|_2\\
			=&\left|\vect{\mu}(\vect{\theta}^*)-\bar{\vect{g}}(\vect{\theta}^{*})\right|_2+o_{\mbP}\left(\frac{1}{\sqrt{mn}}\right).\stepcounter{equation}\tag{\theequation}\label{eq:normal_proof_lhs}
	\end{align*}
	On the other hand, from Assumption \ref{assump:liptchitz_hess} and equation \eqref{eq:supz_bound}
	\begin{align*}
		&\left|\vect{g}_0(\hat{\vect{\theta}}^{(t+1)})-\vect{g}_0(\vect{\theta}^{*})\right|_2\stepcounter{equation}\tag{\theequation}\label{eq:normal_proof_rhs}\\
			=&\left|\vect{\mu}(\hat{\vect{\theta}}^{(t+1)})-\vect{\mu}(\vect{\theta}^{*})+\vect{g}_0(\hat{\vect{\theta}}^{(t+1)})-\vect{g}_0(\vect{\theta}^{*})-\vect{\mu}(\hat{\vect{\theta}}^{(t+1)})+\vect{\mu}(\vect{\theta}^{*})\right|_2\\
			=&\Big|\nabla\vect{\mu}(\vect{\theta}^{*})\cdot(\hat{\vect{\theta}}^{(t+1)}-\vect{\theta}^*)+\int_0^1\left\{\nabla\vect{\mu}(\vect{\theta}^{*}+s(\hat{\vect{\theta}}^{(t+1)}-\vect{\theta}^{*}))-\nabla\vect{\mu}(\vect{\theta}^{*})\right\}\diff s\cdot(\hat{\vect{\theta}}^{(t+1)}-\vect{\theta}^*)\\
			&+\vect{g}_0(\hat{\vect{\theta}}^{(t+1)})-\vect{g}_0(\vect{\theta}^{*})-\vect{\mu}(\hat{\vect{\theta}}^{(t+1)})+\vect{\mu}(\vect{\theta}^{*})\Big|_2\\
			=&\left|\nabla\vect{\mu}(\vect{\theta}^{*})\cdot(\hat{\vect{\theta}}^{(t+1)}-\vect{\theta}^*)\right|_2+O_{\mbP}\left(\frac{p\log n}{\sqrt{m}n}\right)=\left|\nabla\vect{\mu}(\vect{\theta}^{*})\cdot(\hat{\vect{\theta}}^{(t+1)}-\vect{\theta}^*)\right|_2+o_{\mbP}\left(\frac{1}{\sqrt{mn}}\right).
	\end{align*}
	We can combine \eqref{eq:normal_proof_lhs} and \eqref{eq:normal_proof_rhs}, rearrange the terms, then there is
	\begin{equation}	\label{eq:normal_proof_main1}
		\left|\nabla\vect{\mu}(\vect{\theta}^{*})\cdot(\hat{\vect{\theta}}^{(t+1)}-\vect{\theta}^*)\right|_2=\left|\vect{\mu}(\vect{\theta}^*)-\bar{\vect{g}}(\vect{\theta}^{*})\right|_2+o_{\mbP}\left(\frac{1}{\sqrt{mn}}\right).
	\end{equation}
	Denote 
	\begin{align*}
		G_{j,l}(x)=\mbP\left\{\frac{\sqrt{n}g_{j,l}(\vect{\theta}^*)}{\sigma_l(\vect{\theta}^*)}\leq x\right\},&\quad I_{j,l}(x)=\mbI\left\{\frac{\sqrt{n}g_{j,l}(\vect{\theta}^*)}{\sigma_l(\vect{\theta}^*)}\leq x\right\}.\stepcounter{equation}\tag{\theequation}\label{eq:normal_proof_notation}
	\end{align*}
	Then from \eqref{eq:simplified_diff}, for every entry we have
	\begin{align*}
		&\mu_l(\vect{\theta}^*)-\bar{g}_l(\vect{\theta}^{*})\stepcounter{equation}\tag{\theequation}\label{eq:normal_proof_main2}\\
			=&\frac{1}{m+1}\sum_{j\notin\mcB}\frac{\sigma_l(\vect{\theta}^*)}{\sqrt{n}\sum_{k=1}^K\psi(\Delta_k)}\sum_{k=1}^K\{I_{j,l}(\Delta_k)-G_{j,l}(\Delta_k)\}+O_{\mbP}\left(\frac{\log n}{m\sqrt{n}}+\frac{1}{n}+\frac{\log^{3/4} n}{n^{1/2}m^{3/4}}\right).
	\end{align*}
	From equations \eqref{eq:normal_proof_main1}, \eqref{eq:normal_proof_main2} and the rates constrains, for any vector $\tilde{\vect{v}}\in\mbR^p$, there is
	\begin{align*}
		&\left\langle\nabla\vect{\mu}(\vect{\theta}^{*})\cdot(\hat{\vect{\theta}}^{(t+1)}-\vect{\theta}^*),\tilde{\vect{v}}\right\rangle\\
			=& \frac{1}{m+1}\sum_{j\notin\mcB}\frac{1}{\sqrt{n}\sum_{k=1}^K\psi(\Delta_k)}\sum_{k=1}^K\sum_{l=1}^p \left\{\sigma_l(\vect{\theta}^*) I_{j,l}(\Delta_k)-\sigma_l(\vect{\theta}^*) G_{j,l}(\Delta_k)\right\}\tilde{v}_{l} +o_{\mbP}\left( \frac{1}{\sqrt{mn}}\right).
	\end{align*}
	Now we apply central limit theorem and yield
	\begin{align*}
		&\frac{\sqrt{(m+1)n}}{\tilde{\sigma}_{\tilde{\vect{v}}}}\left\langle\nabla\vect{\mu}(\vect{\theta}^{*})\cdot(\hat{\vect{\theta}}^{(t+1)}-\vect{\theta}^*),\tilde{\vect{v}}\right\rangle\xrightarrow{d}\mcN(0,1),\\
		&\text{where }\quad \tilde{\sigma}_{\tilde{\vect{v}}}^2=\tilde{\vect{v}}^{\rm T}\tilde{\vect{\mathcal{C}}}\tilde{\vect{v}}.
	\end{align*}
	Here $\tilde{\vect{\mathcal{C}}}\in\mbR^{p\times p}$ has its $(l_1,l_2)$-entry defined as
	\begin{align*}
		\tilde{\mathcal{C}}_{l_1,l_2}=&\frac{\sigma_{l_1}(\vect{\theta}^*)\sigma_{l_2}(\vect{\theta}^*)}{\{\sum_{k=1}^K\psi(\Delta_k)\}^2}\mbE\left[\sum_{k=1}^K\{I_{0,l_1}(\Delta_k)-G_{0,l_1}(\Delta_k)\}\sum_{k=1}^K\{I_{0,l_2}(\Delta_k)-G_{0,l_2}(\Delta_k)\}\right]\\
			=&\frac{\sqrt{\sigma_{l_1,l_1}\sigma_{l_2,l_2}}}{\{\sum_{k=1}^K\psi(\Delta_k)\}^2}\sum_{k_1,k_2}\left\{\mbP\left(\frac{\sqrt{n}g_{0,{l_1}}(\vect{\theta}^*)}{\sqrt{\sigma_{l_1,l_1}}}\leq\Delta_{k_1}, \frac{\sqrt{n}g_{0,{l_2}}(\vect{\theta}^*)}{\sqrt{\sigma_{l_2,l_2}}}\leq\Delta_{k_2}\right)-G_{0,l_1}(\Delta_{k_1})G_{0,l_2}(\Delta_{k_2})\right\},
	\end{align*}
	since $\sigma_{l}(\vect{\theta}^*)=\sqrt{\var\{\nabla f_l(X,\vect{\theta}^*)\}}=\sqrt{\sigma_{l,l}}$. Moreover, we can apply multivariate Berry-Esseen theorem (See Theorem 1.3 in \cite{gotze.1991}) and give
	\begin{align*}
		\tilde{\mathcal{C}}_{l_1,l_2}&=\mathcal{C}_{l_1,l_2}+O(n^{-1/2}),
	\end{align*}
	where $\mathcal{C}_{l_1,l_2}$ is defined in \eqref{eq:normal_cov_entry}. Now we replace $\tilde{\vect{v}}$ with $\left\{\nabla\vect{\mu}(\vect{\theta}^{*})\right\}^{-1}\vect{v}$. From Assumption \ref{assump:hess_eigen} we know the norm of $\vect{v}\subseteq\mbR^p$ is rescaled by a factor of constant order. Thus the theorem is proved.
\end{proof}


\section{Examples Verification}\label{sec:examples_verification}

In this appendix, we will verify that a large class of generalized linear models and M-estimators satisfy our proposed Assumptions \ref{assump:convexity}--\ref{assump:coordinate_exp}. 

\subsection{Generalized linear models}

For a generalized linear models (GLM) with  the canonical link function $\mcL:\mbR\rightarrow\mbR$, each \emph{i.i.d.} observation $(\vect{X},Y)\in\mbR^{p+1}$ admits the following conditional probability function
\begin{align}	\label{eq:GLM_model}
	\mathbb{P}(Y\mid \vect{X})=\tilde{c}\exp\left\{\frac{Y\langle\vect{\theta}^*,\vect{X}\rangle-\mcL(\langle\vect{\theta}^*,\vect{X}\rangle)}{c(\sigma)}\right\},
\end{align} 
where $\tilde{c}$ and $c(\sigma)$ are some constants, $\vect{\theta}^*$ is the true parameter. The loss function based on the maximum likelihood estimator is defined by, 
\begin{equation}	\label{eq:GLM_loss}
	f(Y,\vect{X},\vect{\theta})=-Y\langle\vect{\theta},\vect{X}\rangle+\mcL(\langle\vect{\theta},\vect{X}\rangle).
\end{equation}
Then we have the following proposition.

\begin{proposition}	\label{prop:glm}
	Let $(\vect{X},Y)$ be observation of a generalized linear model \eqref{eq:GLM_model} with a convex link function $\mcL$. Suppose the following condition holds:
	\begin{enumerate}[label=(C\arabic*)]
		\item There exists $\rho_0>0$ such that
			\begin{equation*}
				\inf_{\vect{v}\in\mbS^{p-1}}\mbE\left\{\mcL''(\langle\vect{\theta}^*,\vect{X}\rangle)|\langle\vect{v},\vect{X}\rangle|^2\right\}\geq\rho_0;
			\end{equation*}
		\item There exists $M>0$ such that
			\begin{equation*}
				|\mcL''(x)|\leq M,\quad |\mcL''(x_1)-\mcL''(x_2)|\leq M|x_1-x_2|;
			\end{equation*}
		\item There exists $\eta,C_M>0$ such that
			\begin{equation*}
				\sup_{\vect{v}\in\mbS^{p-1}}\mbE\left[\exp\left\{\eta|\langle\vect{v},\vect{X}\rangle|^2\right\}\right]\leq C_M,\quad \mbE\left[\exp\eta|\mcL'(\langle\vect{\theta}^*,\vect{X}\rangle)-Y|^2\right]\leq C_M.
			\end{equation*}
	\end{enumerate}
	Then the loss function defined in \eqref{eq:GLM_loss} satisfies Assumption \ref{assump:convexity}-\ref{assump:coordinate_exp}.
\end{proposition}

Condition (C1) and (C3) imply that the covariate $\vect{X}$ is a non-degenerate subgaussian vector. And another part of condition (C3) shows that $\mcL'(\langle\vect{\theta}^*,\vect{X}\rangle)-Y=\mbE[Y|\vect{X}]-Y$ has a subgaussian tail. For convenience of verification,  we assume the $\mcL''$ to be Lipschitz continuous in Condition (C2). 

\begin{proof}[Proof of Proposition \ref{prop:glm}]
	Firstly we can compute the gradient and Hessian as follows,
\begin{align*}
	&\nabla f(Y,\vect{X},\vect{\theta})=-Y\vect{X}+\mathcal{L}'(\langle\vect{\theta},\vect{X}\rangle)\vect{X},\quad &&\vect{\mu}(\vect{\theta})=\mbE\big\{\mathcal{L}'(\langle\vect{\theta},\vect{X}\rangle)\vect{X}-\mathcal{L}'(\langle\vect{\theta}^*,\vect{X}\rangle)\vect{X}\big\},\\
	&\nabla \vect{\mu}(\vect{\theta})=\mbE\left\{\mathcal{L}''(\langle\vect{\theta},\vect{X}\rangle)\vect{XX}^{\tp}\right\}.\quad&&
\end{align*}
	Then we can verify these assumptions one by one.
	\begin{itemize}
		\item Assumption \ref{assump:hess_eigen}: Compute that
			\begin{align*}
				\Lambda_{\max}\{\nabla\vect{\mu}(\vect{\theta}^*)\}=&\sup_{\vect{v}\in\mbS^{p-1}}\mbE\left\{\mathcal{L}''(\langle\vect{\theta}^*,\vect{X}\rangle)|\langle\vect{v},\vect{X}\rangle|^2\right\}\\
					\leq&M\sup_{\vect{v}\in\mbS^{p-1}}\mbE|\langle\vect{v},\vect{X}\rangle|^2\leq\frac{M}{\eta}C_M,\\
			\Lambda_{\min}\{\nabla\vect{\mu}(\vect{\theta}^*)\}=&\inf_{\vect{v}\in\mbS^{p-1}}\mbE\left\{\mathcal{L}''(\langle\vect{\theta}^*,\vect{X}\rangle)|\langle\vect{v},\vect{X}\rangle|^2\right\}\\
					\geq&\rho_0.
			\end{align*}
		\item Assumption \ref{assump:liptchitz_hess}: By elementary inequalities $3|xy^2|\leq |x|^3+2|y|^3$ and $2|x|^3\leq e^{x^2}$, we have
			\begin{align*}
				&\Norm{\nabla\vect{\mu}(\vect{\theta}_1)-\nabla\vect{\mu}(\vect{\theta}_2)}\\
					\leq&\sup_{\vect{v}\in\mbS^{p-1}}\mbE\left\{\big|\mathcal{L}''(\langle\vect{\theta}_1,\vect{X}\rangle)-\mathcal{L}''(\langle\vect{\theta}_2,\vect{X}\rangle)\big|\cdot|\langle\vect{v},\vect{X}\rangle|^2\right\}\\
					\leq&M\sup_{\vect{v}\in\mbS^{p-1}} \mbE\left\{|\langle \vect{\theta}_1-\vect{\theta}_2, \vect{X}\rangle|\cdot|\langle\vect{v},\vect{X}\rangle|^2\right\}\\
					\leq&M\sup_{\vect{v}\in\mbS^{p-1}}\mbE\left(|\langle\vect{v},\vect{X}\rangle|^3\right)|\vect{\theta}_1-\vect{\theta}_2|_2\leq \frac{MC_M}{2\eta^{3/2}}|\vect{\theta}_1-\vect{\theta}_2|_2.
			\end{align*}
		\item Assumption \ref{assump:liptchitz_gradient}:
			\begin{align*}
				M_{\vect{\theta}_1,\vect{\theta}_2}(Y,\vect{X},\vect{v})=&\frac{1}{|\vect{\theta}_1-\vect{\theta}_2|_2}\left|\langle\vect{v},\vect{X}\rangle\left\{\mathcal{L}'(\langle\vect{\theta}_1,\vect{X}\rangle)-\mathcal{L}'(\langle\vect{\theta}_2,\vect{X}\rangle)\right\}\right|\\
					\leq&M|\langle\vect{v},\vect{X}\rangle|\cdot\left|\left\langle\frac{\vect{\theta}_1-\vect{\theta}_2}{|\vect{\theta}_1-\vect{\theta}_2|_2},\,\vect{X}\right\rangle\right|.			
			\end{align*}
			Thus
			\begin{align*}
				\sup_{\vect{v}\in\mbS^{p-1}}\sup_{\vect{\theta}_1,\vect{\theta}_2}\mbE\left[\exp\left\{\frac{\eta}{M} M_{\vect{\theta}_1,\vect{\theta}_2}(Y,\vect{X},\vect{v})\right\}\right]\leq \sup_{\vect{v}\in\mbS^{p-1}}\mbE\left\{\exp\left(\eta |\langle\vect{v},\vect{X}\rangle|^2\right)\right\}\leq C_M.
			\end{align*}
			Similarly
			\begin{align*}
				\sup_{\vect{v}\in\mbS^{p-1}}\bar{M}(Y,\vect{X},,\vect{v})\leq&\sup_{\vect{v}\in\mbS^{p-1}}M|\vect{X}|_2\cdot\left|\langle\vect{v},\vect{X}\rangle\right|\leq M|\vect{X}|_2^2=M\sum_{l=1}^p|\langle\vect{X},\vect{e}_l\rangle|^2,
			\end{align*}
			where $\vect{e}_l$ is the $l$-th base vector. Then using generalized H\"older's inequality we can prove
			\begin{align*}
				\mbE\left[\sup_{\vect{v}\in\mbS^{p-1}}\exp\left\{\frac{\eta}{Mp}\bar{M}(Y,\vect{X},\vect{v})\right\}\right]=&\mbE\left[\exp\left\{\frac{\eta}{p}\sum_{l=1}^p|\langle\vect{X},\vect{e}_l\rangle|^2\right\}\right]\\
				\leq&\left[\prod_{l=1}^p\mbE\left\{\exp\left(\eta|\langle\vect{X},\vect{e}_l\rangle|^2\right)\right\}\right]^{1/p}\leq C_M.
			\end{align*}
		\item Assumption \ref{assump:bound_variance}: The variance at $\vect{\theta}^*$ is
			\begin{align*}
				\sigma_l^2(\vect{\theta}^*)=&\mbE\left[\left\{-Y+\mathcal{L}'(\langle\vect{\theta}^*,\vect{X}\rangle)\right\}^2X_l^2\right]\\
					=&\mbE\left[c(\sigma)\mathcal{L}''(\langle\vect{\theta}^*,\vect{X}\rangle)X_l^2\right],
			\end{align*}
			then we can bound them as follows
			\begin{align*}
				\sigma_l^2(\vect{\theta}^*)\leq&c(\sigma)\Lambda_{\max}\{\nabla\vect{\mu}(\vect{\theta}^*)\}\leq \eta^{-1}Mc(\sigma)C_M,\\
				\sigma_l^2(\vect{\theta}^*)\geq&c(\sigma)\Lambda_{\min}\{\nabla\vect{\mu}(\vect{\theta}^*)\}\geq \rho_0c(\sigma).
			\end{align*}
		\item Assumption \ref{assump:coordinate_exp}: Using Cauchy inequality we have
			\begin{align*}
				&\mbE[\exp\eta|\nabla f_l(Y,\vect{X},\vect{\theta}^*)-\mu_l(\vect{\theta}^*)|]\\
				\leq&\mbE\Big[\exp\eta\left|\mathcal{L}'(\langle\vect{\theta}^*,\vect{X}\rangle)X_l-YX_l\right|\Big]\\
				\leq&\sqrt{\mbE\left[\exp\eta|\mathcal{L}'(\langle\vect{\theta}^*,\vect{X}\rangle)-Y|^2\right]\mbE\left[\exp\eta |\langle\vect{X}_l,\vect{e}_l\rangle|^2\right]}\leq C_M.
			\end{align*}
	\end{itemize}
\end{proof}

As an example, we can show that the logistic regression model satisfies these conditions. 

\begin{example}
	(Logistic regression) In logistic regression model, the response variable $Y$ takes value in $\{0,1\}$, and the link function is $\mcL(x)=\log(1+e^x)$. Then we can compute that
	\begin{equation*}
		\mcL'(x)=\frac{1}{1+e^{-x}},\quad\mcL''(x)=\frac{1}{(1+e^{x})(1+e^{-x})},\quad |\mcL'''(x)|\leq 2.
	\end{equation*}
	It is not hard to verify that conditions (C1)--(C3) in Proposition \ref{prop:glm} hold, provided non-degenerate subgaussian covariate $\vect{X}$.
\end{example}


\subsection{$M$-Estimator}

    Next we consider the M-estimator. Assume that each  \emph{i.i.d.} observation $(\vect{X},Y)\in\mbR^{p+1}$ is generated from the linear model,
    \begin{equation}	\label{eq:linear_model}
	Y=\langle \vect{\theta}^*,\vect{X}\rangle +\epsilon,
    \end{equation}
    and the loss function is
    \begin{equation}	\label{eq:Mest_loss}
	f(Y,\vect{X},\vect{\theta})=\mcL(Y-\langle \vect{\theta},\vect{X}\rangle).
    \end{equation}
    Then we have the following proposition.

\begin{proposition}	\label{prop:m-estimator}
	Let $(\vect{X},Y)$ be observation of linear model \eqref{eq:linear_model} and $\mcL$ is a convex regression function. Suppose the noise $\epsilon$ is independent of the covariate $\vect{X}$, and $\mbE\{\mcL'(\epsilon)\}=0$. Moreover the following conditions hold true:
	\begin{enumerate}[label=(C\arabic*')]
		\item There exists $\rho_0>0$ such that
			\begin{equation*}
				\min\left\{\mbE\{\mcL''(\epsilon)\}, \mbE\{\mcL'(\epsilon)^2\},  \inf_{\vect{v}\in\mbS^{p-1}}\mbE|\langle\vect{v},\vect{X}\rangle|^2\right\}\geq\rho_0;
			\end{equation*}
		\item There exists a constant $M>0$ such that
			\begin{equation*}
				|\mcL''(x)|\leq M,\qquad |\mcL''(x_1)-\mcL''(x_2)|\leq M|x_1-x_2|;
			\end{equation*}
		\item There exists $\eta,C_M>0$ such that
			\begin{equation*}
				\sup_{\vect{v}\in\mbS^{p-1}}\mbE\left[\exp\left\{\eta|\langle\vect{v},\vect{X}\rangle|^2\right\}\right]\leq C_M,\quad \mbE\left[\exp\left\{\eta|\mcL'(\epsilon)|^2\right\}\right]\leq C_M.
			\end{equation*}
	\end{enumerate}
	Then the loss function defined in \eqref{eq:Mest_loss} satisfies Assumptions \ref{assump:convexity}-\ref{assump:coordinate_exp}.
\end{proposition}

Again condition (C1') implies the non-degeneracy of covariate $\vect{X}$. The first half of condition (C3') requires covariate $\vect{X}$ to be sub-gaussian. While the noise, depending on the explicit formulation of $\mcL$, can be heavy-tailed. As will be seen in the following, the noise has to be sub-gaussian in linear regression, but is allowed to be heavy-tailed in Huber regression. It is worthwhile noting that for the ease of presentation, in condition (C2') we simply assume $\mcL''$ to be Lipschitz continuous. However, in Example \ref{exp:huber_reg}, we will prove that the Huber regression model satisfies all assumptions in Section \ref{sec:tech_assump}.

\begin{proof}[Proof of Proposition \ref{prop:m-estimator}] We can directly compute the gradient and Hessian as follows:
	\begin{align*}
		&\nabla f(Y,\vect{X},,\vect{\theta})=\mcL '(\langle \vect{X},\vect{\theta}-\vect{\theta}^*\rangle+\epsilon)\vect{X},\\
		&\nabla \vect{\mu}(\vect{\theta})=\mbE\{\mcL ''(\langle \vect{X},\vect{\theta}-\vect{\theta}^*\rangle+\epsilon)\vect{X}\vect{X}^{\rm T}\}.
	\end{align*}	
	Now we verify those assumptions.
	\begin{itemize}
	\item Assumption \ref{assump:hess_eigen}: 
		\begin{align*}
			\Lambda_{\max}\{\nabla \vect{\mu}(\vect{\theta}^*)\}=&\sup_{\vect{v}\in\mbS^{p-1}}\mbE\left\{\mcL''(\epsilon)|\vect{v}^{\rm T}\vect{X}|^2\right\}\\
				=&\mbE\left\{\mcL''(\epsilon)\right\}\sup_{\vect{v}\in\mbS^{p-1}}\mbE|\vect{v}^{\rm T}\vect{X}|^2\leq\eta^{-1}MC_M,\\
			\Lambda_{\min}\{\nabla \vect{\mu}(\vect{\theta}^*)\}=&\inf_{\vect{v}\in\mbS^{p-1}}\mbE\left\{\mcL''(\epsilon)|\vect{v}^{\rm T}\vect{X}|^2\right\}\\
				=&\mbE\left\{\mcL''(\epsilon)\right\}\inf_{\vect{v}\in\mbS^{p-1}}\mbE|\vect{v}^{\rm T}\vect{X}|^2\geq\rho^2_0.
		\end{align*}
	\item Verification for Assumption \ref{assump:liptchitz_hess} and \ref{assump:liptchitz_gradient} are almost the same as the proof in Proposition \ref{prop:glm}, thus omitted for brevity. 
	\item Assumption \ref{assump:bound_variance}: The variance of the $l$'th coordinate is
		\begin{align*}
			\sigma_l^2(\vect{\theta}^*)=&\mbE\left[\mcL'(\epsilon)^2X_l^2\right]=\mbE\left[\mcL'(\epsilon)^2 \right]\mbE(X_l^2),
		\end{align*}
		thus
		\begin{align*}
			\sigma_l^2(\vect{\theta}^*)\leq&\mbE\left[\mcL'(\epsilon)^2\right]\max_{1\leq l\leq p}\{\mbE(X_l^2)\}\leq\eta^{-2}C_M^2,\\
			\sigma_l^2(\vect{\theta}^*)\geq&\mbE\left[\mcL'(\epsilon)^2\right]\min_{1\leq l\leq p}\{\mbE(X_l^2)\}\geq \rho_0^2.
		\end{align*}
	\item Assumption \ref{assump:coordinate_exp}: Using Cauchy inequality we have
		\begin{align*}
			&\mbE\big[\exp\eta|\nabla f_l(Y,\vect{X},\vect{\theta}^*)-\mu_l(\vect{\theta}^*)|\big]\\
				\leq&\mbE\Big[\exp\eta\left|\mcL'(\epsilon)X_l\right|\Big]\\
				\leq&\sqrt{\mbE\left[\exp\eta|\mcL'(\epsilon)|^2\right]\mbE\left[\exp\eta|\langle\vect{X},\vect{e}_l\rangle|^2\right]}\leq C_M.
		\end{align*}
	\end{itemize}
\end{proof}

\noindent\textbf{Example \ref{exp:lin_reg} Continued.} In linear regression model, the regression function $\mcL$ is defined by $\mcL(x)=x^2/2$. Then we can compute that $\mcL'(x)=x,\mcL''(x)=1$. It is relatively straightforward to verify  that the conditions in Proposition \ref{prop:m-estimator} hold, provided $\vect{X}$ is a non-degenerate sub-gaussian random vector and the noise $\epsilon$ follows a zero-mean sub-gaussian distribution.
	
\vspace{2mm}

\noindent\textbf{Example \ref{exp:huber_reg} Continued.} In Huber regression model, the regression function $\mcL$ is defined by 
	\begin{equation*}
		\mcL(x)=
		\begin{cases}
			x^2/2\quad&\text{for }|x|\leq\delta,\\
			\delta(|x|-\delta/2)\quad&\text{otherwise.}
		\end{cases}
	\end{equation*} 
	Then we can compute that
	\begin{equation*}
		\mcL'(x)=
		\begin{cases}
			x\quad&\text{for }|x|\leq\delta,\\
			\delta\,\mathrm{sign}(x)\quad&\text{otherwise},
		\end{cases}\qquad\mcL''(x)=\mbI(|x|\leq\delta).
	\end{equation*}
	In this case, Proposition \ref{prop:m-estimator} is not directly applicable since $\mcL''(x)$ is not Lipschitz continuous. However, if noise $\epsilon$ has a symmetric distribution and uniformly bounded probability density function, Assumption \ref{assump:liptchitz_hess} can be verified as follows. 
	
\vspace{2mm}

\begin{proof}[Verification for Huber Regression]
	We only need to show the Lipschitz Hessian assumption \ref{assump:liptchitz_hess} holds true. It is easy to compute the Hessian matrix of Huber loss as follows:
	\begin{equation*}
		\nabla \vect{\mu}(\vect{\theta})=\mbE\left\{\mbI(|\langle \vect{X},\vect{\theta}-\vect{\theta}^*\rangle+\epsilon|\leq\delta)\vect{X}\vect{X}^{\rm T}\right\}.
	\end{equation*}
	Assume the noise $\epsilon$ has probability density function $\mathfrak{p}(x)$ uniformly bounded by a constant $M>0$, then by independence of $\epsilon$ and $\vect{X}$, there is
	\begin{equation*}
		\nabla \vect{\mu}(\vect{\theta})=\mbE\left\{\mbP\Big(|\langle \vect{X},\vect{\theta}-\vect{\theta}^*\rangle+\epsilon|\leq\delta\,\Big|\, \vect{X}\Big)\vect{X}\vect{X}^{\rm T}\right\}.
	\end{equation*}
	Then for arbitrary $\vect{\theta}_1,\vect{\theta}_2\in\mbR^p$, we have
	\begin{align*}
		&\Norm{\nabla \vect{\mu}(\vect{\theta}_1)-\nabla \vect{\mu}(\vect{\theta}_2)}\\
		=&\sup_{\vect{v}\in\mbS^{p-1}}\mbE\Big[\left\{\mbP\Big(|\langle \vect{X},\vect{\theta}_1-\vect{\theta}^*\rangle+\epsilon|\leq\delta\,\Big|\, \vect{X}\Big)-\mbP\Big(|\langle \vect{X},\vect{\theta}_2-\vect{\theta}^*\rangle+\epsilon|\leq\delta\,\Big|\, \vect{X}\Big)\right\}|\vect{v}^{\rm T}\vect{X}|^2  \Big]\\
		\leq&2M\sup_{\vect{v}\in\mbS^{p-1}}\mbE\left\{|\langle\vect{\theta}_1-\vect{\theta}_2,\vect{X}\rangle|\cdot|\langle\vect{v},\vect{X}\rangle|^2\right\}\leq\eta^{-3/2}MC_M|\vect{\theta}_1-\vect{\theta}_2|_2.
	\end{align*}
\end{proof}

\bibliographystyle{chicago}
\bibliography{VRMOM_arXiv}


\end{document}